\documentclass{article}

\usepackage{graphicx}
\usepackage{subfig}
\usepackage{booktabs}
\usepackage{adjustbox}
\usepackage{amsmath, amsthm, amssymb}
\usepackage{bbm}
\usepackage{wrapfig}
\usepackage{pgfplots,tikz}
\pgfplotsset{compat=newest}
\usepgfplotslibrary{polar}

\usepackage{caption}
\newtheorem{theorem}{Theorem}[section]
\newtheorem{definition}{Definition}[section]

\newtheorem{lemma}[theorem]{Lemma}
\newtheorem{remark}{Remark}[section]
\newtheorem{proposition}[theorem]{Proposition}

\newtheorem{corollary}[theorem]{Corollary}

\def\eqref#1{(\ref{#1})}
\def\R{\mathbb R}
\def\Rp{\mathbb R_{+}}
\def\Rpz{\mathbb R_{\geq0}}

\def\H{\mathcal H}

\def\1{\mathbbm{1}}
\def\Trn{\text{Trn}}

\def\Perm{\mathcal{P}}
\def\ptheta{\pmb{\theta}}

\def\pupsilon{\pmb{\upsilon}}

\def\prho{\pmb{\rho}}
\def\pgamma{\pmb{\gamma}}

\def\pphi{\pmb{\phi}}

\def\dim{\text{dim}}
\def\summ{\textit{sum}}

\def\dout{d_{\text{out}}}
\def\din{d_{\text{in}}}
\def\bm{\pmb{m}}
\def\br{\pmb{r}}

\def\subs{\Gamma_{s}(\ptheta^r)}
\def\subsg{\Gamma_{s}(\ptheta_*)}
\def\subsc{\gbar{\Gamma}_{s}(\ptheta_*^r)}

\def\expmanfc{\Theta_{r \to m}(\ptheta_*^r)}
\def\expmanc{\gbar{\Theta}_{r \to m}(\ptheta_*^r)}

\def\expman{\Theta_{r \to m}(\ptheta^r)}

\def\expmang{\Theta_{r^* \to m}(\ptheta_*)}

\def\expmanmultig{\Theta_{\br^* \to \bm}(\ptheta_*)}

\def\ftwo{f^{(2)}(x | \ptheta)}
\def\fL{f^{(L)}(x | \ptheta)}

\newcommand\gbar[1]{\,\overline{\!{#1}\!}}

\usepackage{hyperref}

\usepackage[accepted]{icml2021}

\icmltitlerunning{Geometry of the Loss Landscape in Overparameterized Neural Networks}

\begin{document}

\twocolumn[
\icmltitle{Geometry of the Loss Landscape in Overparameterized Neural Networks: Symmetries and Invariances}

\icmlsetsymbol{eq}{*}

\begin{icmlauthorlist}
    \icmlauthor{Berfin \c{S}im\c{s}ek}{cft,lcn}
    \icmlauthor{François Ged}{cft}
    \icmlauthor{Arthur Jacot}{cft}
    \icmlauthor{Francesco Spadaro}{cft} \\
    \icmlauthor{Clément Hongler}{cft,eq}
    \icmlauthor{Wulfram Gerstner}{lcn,eq}
    \icmlauthor{Johanni Brea}{lcn,eq}
\end{icmlauthorlist}

\icmlaffiliation{cft}{Chair of Statistical Field Theory, École Polytechnique Fédérale de Lausanne, Switzerland}
\icmlaffiliation{lcn}{Laboratory of Computational Neuroscience, École Polytechnique Fédérale de Lausanne, Switzerland}

\icmlcorrespondingauthor{Berfin \c{S}im\c{s}ek}{berfin.simsek@epfl.ch}

\icmlkeywords{Neural Network Landscapes, overparameterization}

\vskip 0.3in
]
\printAffiliationsAndNotice{\icmlEqualContribution}

\begin{abstract}
We study how permutation symmetries in overparameterized multi-layer neural
networks generate `symmetry-induced' critical points.
Assuming a network with $ L $ layers of minimal widths $ r_1^*, \ldots, r_{L-1}^* $ reaches a zero-loss minimum at $ r_1^*! \cdots r_{L-1}^*! $ isolated points that are permutations of one another,
we show that adding one extra neuron to each layer is sufficient to connect all these previously discrete minima into a single manifold.
For a two-layer overparameterized network of width $ r^*+ h =: m $ we explicitly describe the manifold of global minima: it consists of $ T(r^*, m) $ affine subspaces of dimension at least $ h $ that are connected to one another.
For a network of width $m$, we identify the number $G(r,m)$ of affine subspaces containing only symmetry-induced critical points that are related to the critical points of a smaller network of width $r<r^*$.
Via a combinatorial analysis, we derive closed-form formulas for $ T $ and $ G $ and show that the number of symmetry-induced critical subspaces dominates the number of affine subspaces forming the global minima manifold in the mildly overparameterized regime (small $ h $) and vice versa in the vastly overparameterized regime ($h \gg r^*$).
Our results provide new insights into the minimization of the non-convex loss function of overparameterized neural networks.
\end{abstract}

\section{Introduction}

Neural network landscapes were traditionally thought of as highly non-convex landscapes, where non-global critical points may harm gradient-descent by slowing it down (due to saddles) or making it stop in local minima.
Earlier works have argued in favor of a proliferation of saddles in high-dimensional neural network landscapes through an analogy with random error functions \cite{dauphin2014identifying}.
On the other hand, practical neural network landscapes are found to exhibit surprising properties, such as the connectivity of global minima \cite{draxler2018essentially, garipov2018loss} and the convergence to a global minimum in the so-called overparameterized regime \cite{jacot2018neural}, thereby ruling out proliferating saddles as a problem in this regime.
Yet, in mildly overparameterized networks, gradient descent may find a global minimum only for a small fraction of random initializations \cite{sagun2014explorations, chizat2018global, frankle2018lottery}.

In this work, we study the width-dependent scaling of the number of symmetry-induced critical points and the connectivity of global minima by exploiting the permutation symmetry and further invariances of the network parameterization.
The permutation symmetry introduces an invariance to a permutation in parameterization that is characteristic for many machine learning models beyond neural networks, such as mixture models, multiple kernel learning, or matrix factorization.

Further invariances in a neural network of width $m$ induce equal loss manifolds such that all points in the manifold are equivalent to a single point in a narrower network of width $ r < m $.
The mapping approach from a point in parameter space of the narrower network to a parameter manifold of the full network is particularly useful for the study of critical points as critical points of the narrow network turn into symmetry-induced critical subspaces of the full one. In particular, a global minimum of the narrow network turns into a collection of global minima subspaces that are connected to one another.

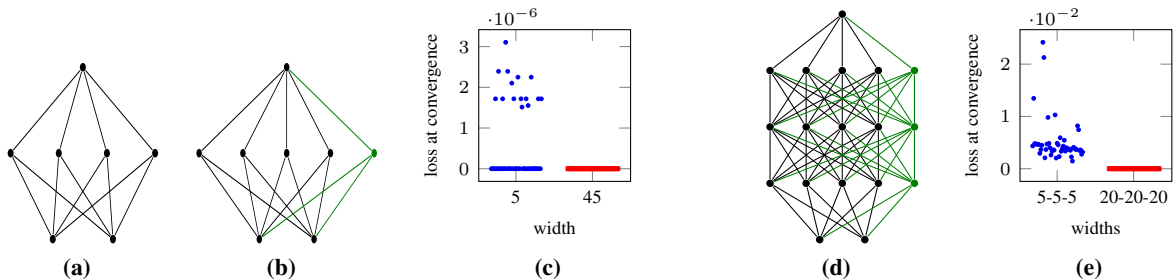
\begin{figure*}[h!]
    \tikzset{font = \scriptsize, mark size = .7}
\pgfplotsset{width = 3.6cm, height = 3.6cm}
\centering
\begin{tabular}{llllll}
    \begin{adjustbox}{width=2cm, height=2.4cm}
\begin{tikzpicture}[every node/.style = {fill = black, inner sep = 1, circle}]
    \foreach \i in {1,2} {
        \node (in\i) at (\i, 0) {};
    };
    \foreach \i in {1,...,4} {
        \node (h\i) at (.8*\i - .5, 1) {};
    };
    \node (out) at (1.5, 2) {};
    \foreach \i in {1,...,4} {
        \foreach \j in {1,2} {
            \draw (in\j) -- (h\i);
        }
        \draw (h\i) -- (out);
    }; 
\end{tikzpicture}
\end{adjustbox} &
    \begin{adjustbox}{width=2.4cm, height=2.4cm}
\begin{tikzpicture}[every node/.style = {fill = black, inner sep = 1, circle}]
    \foreach \i in {1,2} {
        \node (in\i) at (\i, 0) {};
    };
    \foreach \i in {1,...,4} {
        \node (h\i) at (.8*\i - .9, 1) {};
    };
    \node (out) at (1.5, 2) {};
    \foreach \i in {1,...,4} {
        \foreach \j in {1,2} {
            \draw (in\j) -- (h\i);
        }
        \draw (h\i) -- (out);
    };
    \node[green!50!black] (h5) at (3.1, 1) {};
    \foreach \j in {1,2} {
        \draw[green!50!black] (in\j) -- (h5);
    }
    \draw[green!50!black] (h5) -- (out);
\end{tikzpicture}
\end{adjustbox} &
    \begin{tikzpicture}
\begin{axis}[only marks, ylabel={loss at convergence}, xtick={1,2}, xticklabels={5,45}, xlabel={width}]
    \addplot
        coordinates {
            (0.68,4.798395452638219e-32)
            (0.6933333333333334,1.3389165795411656e-31)
            (0.7066666666666667,5.986262297576164e-32)
            (0.72,5.074097880846038e-32)
            (0.7333333333333334,1.7176022709040847e-6)
            (0.7466666666666667,1.7410315040868256e-31)
            (0.76,4.7162713250869535e-32)
            (0.7733333333333333,2.3919528215368495e-6)
            (0.7866666666666667,1.45740996357942e-31)
            (0.8,1.4245603125589137e-31)
            (0.8133333333333334,5.446589459382135e-32)
            (0.8266666666666667,1.7176022709040955e-6)
            (0.84,5.569775650709033e-32)
            (0.8533333333333334,7.666873907821702e-32)
            (0.8666666666666667,3.1056484870609425e-6)
            (0.88,5.346867304498455e-32)
            (0.8933333333333333,2.3919528215368745e-6)
            (0.9066666666666667,6.147577548123292e-32)
            (0.92,6.5347341494364e-32)
            (0.9333333333333333,4.282186650887408e-32)
            (0.9466666666666667,2.1035430269633406e-6)
            (0.96,4.513894010764192e-32)
            (0.9733333333333334,1.7176022709040968e-6)
            (0.9866666666666667,3.774776862802804e-31)
            (1.0,4.8423905209692535e-32)
            (1.0133333333333334,1.375872436939235e-31)
            (1.0266666666666666,2.248953749976516e-6)
            (1.04,5.587373678041446e-32)
            (1.0533333333333335,5.549244618821217e-32)
            (1.0666666666666667,1.7176022709040904e-6)
            (1.08,1.5141149772488473e-6)
            (1.0933333333333333,4.68694127953293e-32)
            (1.1066666666666667,4.757333388862586e-32)
            (1.12,4.622415179314079e-32)
            (1.1333333333333333,1.7176022709041002e-6)
            (1.1466666666666667,4.974375725962359e-32)
            (1.16,1.5528903142592819e-6)
            (1.1733333333333333,4.6546782294235045e-32)
            (1.1866666666666668,5.631368746372482e-32)
            (1.2,2.248953749976509e-6)
            (1.2133333333333334,2.9602814977675758e-31)
            (1.2266666666666668,6.907225727972497e-32)
            (1.24,6.009726334019383e-32)
            (1.2533333333333334,4.933313662186726e-32)
            (1.2666666666666666,1.1790678312717383e-31)
            (1.28,7.423434529723308e-32)
            (1.2933333333333334,1.717602270904095e-6)
            (1.3066666666666666,6.2824957576718e-32)
            (1.32,6.2824957576718e-32)
            (1.3333333333333333,1.7176022709040836e-6)
        }
        ;
    \addplot
        coordinates {
            (1.68,4.845434935640551e-17)
            (1.6933333333333334,4.710290723953481e-17)
            (1.7066666666666668,7.45825061016093e-17)
            (1.72,2.0218052232303973e-16)
            (1.7333333333333334,2.93897011819373e-17)
            (1.7466666666666666,4.741471488217759e-17)
            (1.76,7.78589256500585e-17)
            (1.7733333333333334,3.7788208711726603e-17)
            (1.7866666666666666,2.5588672078449037e-16)
            (1.8,1.0462997078740756e-16)
            (1.8133333333333332,3.5371530679458856e-17)
            (1.8266666666666667,5.068953811544443e-17)
            (1.84,2.1205794206777482e-17)
            (1.8533333333333333,4.2868662196560685e-17)
            (1.8666666666666667,8.125067177780508e-17)
            (1.8800000000000001,1.2265585858734658e-16)
            (1.8933333333333333,4.9414352616791514e-17)
            (1.9066666666666667,1.114654345834633e-16)
            (1.92,4.863227437166358e-17)
            (1.9333333333333333,8.954194064919412e-17)
            (1.9466666666666668,4.79400903478199e-17)
            (1.96,4.883737521593099e-17)
            (1.9733333333333334,2.4942714066280585e-17)
            (1.9866666666666666,1.9357457648042708e-16)
            (2.0,8.582650915195954e-17)
            (2.013333333333333,4.366716964373899e-17)
            (2.026666666666667,2.4934617943708424e-16)
            (2.04,1.5782164727422665e-16)
            (2.0533333333333332,6.555899121690997e-17)
            (2.066666666666667,4.7814092281654133e-17)
            (2.08,3.1503237052334357e-17)
            (2.0933333333333333,6.591283461665985e-17)
            (2.1066666666666665,1.4009323829618706e-16)
            (2.12,2.7553804023680573e-16)
            (2.1333333333333333,4.1563817420847774e-17)
            (2.1466666666666665,4.5609786243056056e-17)
            (2.16,3.910576723108334e-17)
            (2.1733333333333333,2.034913958892805e-16)
            (2.1866666666666665,6.763033015391691e-17)
            (2.2,8.635643835464755e-17)
            (2.2133333333333334,3.732287243131954e-16)
            (2.2266666666666666,1.1811799574615018e-16)
            (2.24,8.452461563486795e-17)
            (2.2533333333333334,4.0254807054541755e-16)
            (2.2666666666666666,1.8891705567895134e-17)
            (2.28,5.805480843041512e-17)
            (2.2933333333333334,3.6249414273692065e-17)
            (2.3066666666666666,1.4783730551601792e-16)
            (2.32,1.3500693947331865e-15)
            (2.3333333333333335,7.964556218902253e-17)
        }
        ;
\end{axis}
\end{tikzpicture} \hspace{1.2cm} &
    \begin{tikzpicture}[every node/.style = {fill = black, inner sep = 1, circle}, yscale = 0.75, xscale = 0.6, newnode/.style = {green!50!black}]
    \foreach \i in {1,2} {
        \node (in\i) at (\i, 0) {};
    };
    \foreach \i in {1,...,4} {
        \node (h1\i) at (.8*\i - .9, 1) {};
    };
    \foreach \i in {1,...,4} {
        \node (h2\i) at (.8*\i - .9, 2) {};
    };
    \foreach \i in {1,...,4} {
        \node (h3\i) at (.8*\i - .9, 3) {};
    };
    \node (out) at (1.5, 4) {};
    \foreach \i in {1,...,4} {
        \foreach \j in {1,2} {
            \draw (in\j) -- (h1\i);
        }
        \foreach \j in {1,...,4} {
            \draw (h1\i) -- (h2\j);
            \draw (h2\i) -- (h3\j);
        }
        \draw (h3\i) -- (out);
    };
    \foreach \i in {1,...,3} {
        \node[newnode] (h\i5) at (3.1, \i) {};
    }
    \foreach \i in {1,2} {
        \draw[newnode] (in\i) -- (h15);
    }
    \foreach \i in {1,...,5} {
        \draw[newnode] (h2\i) -- (h15);
        \draw[newnode] (h1\i) -- (h25);
        \draw[newnode] (h3\i) -- (h25);
        \draw[newnode] (h2\i) -- (h35);
    }
    \draw[newnode] (h35) -- (out);
\end{tikzpicture} &
    \begin{tikzpicture}
\begin{axis}[only marks, ylabel={loss at convergence}, xtick={1,2},
xticklabels={5-5-5,20-20-20}, xlabel={widths}]
    \addplot
        coordinates {
            (0.68,0.0043551999002999)
            (0.6933333333333334,0.013461110863562144)
            (0.7066666666666667,0.004805520101473104)
            (0.72,0.00471389632417121)
            (0.7333333333333334,0.004721985618800117)
            (0.7466666666666667,0.004675814103854579)
            (0.76,0.004650303458914569)
            (0.7733333333333333,0.0029390316944898354)
            (0.7866666666666667,0.003669285151194906)
            (0.8,0.004490168228607979)
            (0.8133333333333334,0.024174307170947422)
            (0.8266666666666667,0.021270489078553573)
            (0.84,0.0020866174497368366)
            (0.8533333333333334,0.003676770305382683)
            (0.8666666666666667,0.004721192518722597)
            (0.88,0.009816125445068751)
            (0.8933333333333333,0.004850082691029067)
            (0.9066666666666667,0.004013806173798219)
            (0.92,0.0026235219666558276)
            (0.9333333333333333,0.003382268927953068)
            (0.9466666666666667,0.0036075108758624927)
            (0.96,0.003560533713884744)
            (0.9733333333333334,0.010283433414225013)
            (0.9866666666666667,0.0020478012417873034)
            (1.0,0.004875180587011746)
            (1.0133333333333334,0.004561948140865512)
            (1.0266666666666666,0.0023213029878292927)
            (1.04,0.005928965601184014)
            (1.0533333333333335,0.003336488368541117)
            (1.0666666666666667,0.0038245893160895896)
            (1.08,0.004324948898285068)
            (1.0933333333333333,0.005443467853661867)
            (1.1066666666666667,0.003368384481496864)
            (1.12,0.0040190082516728615)
            (1.1333333333333333,0.003730498192875604)
            (1.1466666666666667,0.0037759685093642194)
            (1.16,0.0037545603303565612)
            (1.1733333333333333,0.0035294177715176974)
            (1.1866666666666668,0.002267300061154192)
            (1.2,0.001461098983321214)
            (1.2133333333333334,0.004084999190951158)
            (1.2266666666666668,0.003979962982465959)
            (1.24,0.0037283478422445236)
            (1.2533333333333334,0.003683100229774075)
            (1.2666666666666666,0.008200941538627377)
            (1.28,0.007460606475048363)
            (1.2933333333333334,0.0034893419578807426)
            (1.3066666666666666,0.0036015834532282277)
            (1.32,0.002774542773425838)
            (1.3333333333333333,0.003213278372108965)
        }
        ;
    \addplot
        coordinates {
            (1.68,5.723556089140091e-6)
            (1.6933333333333334,6.270197600253691e-6)
            (1.7066666666666668,3.4970509392871213e-6)
            (1.72,5.228532844385378e-6)
            (1.7333333333333334,6.371978245982154e-6)
            (1.7466666666666666,3.8892151473821075e-6)
            (1.76,4.422527148694087e-6)
            (1.7733333333333334,1.2442312850915064e-5)
            (1.7866666666666666,4.373559669358041e-6)
            (1.8,5.2450918258996e-6)
            (1.8133333333333332,6.3244847605397344e-6)
            (1.8266666666666667,7.0076992676329105e-6)
            (1.84,4.392401606340131e-6)
            (1.8533333333333333,6.175787173376639e-6)
            (1.8666666666666667,4.170103912786634e-6)
            (1.8800000000000001,4.358823477884257e-6)
            (1.8933333333333333,7.355758215504749e-6)
            (1.9066666666666667,5.859440902078967e-6)
            (1.92,4.5117658353531666e-6)
            (1.9333333333333333,7.595506613439916e-6)
            (1.9466666666666668,6.142853636358983e-6)
            (1.96,6.453908807818233e-6)
            (1.9733333333333334,4.040145276459618e-6)
            (1.9866666666666666,3.9822952915648074e-6)
            (2.0,3.2256012667898644e-6)
            (2.013333333333333,8.939321983481519e-6)
            (2.026666666666667,5.120392923474475e-6)
            (2.04,4.640989001212733e-6)
            (2.0533333333333332,4.417849534959585e-6)
            (2.066666666666667,5.258011134254064e-6)
            (2.08,5.715134330837762e-6)
            (2.0933333333333333,4.859301611059056e-6)
            (2.1066666666666665,5.194979947832384e-6)
            (2.12,3.844839441316876e-6)
            (2.1333333333333333,4.9809133903831965e-6)
            (2.1466666666666665,4.19688778384005e-6)
            (2.16,4.569473370054264e-6)
            (2.1733333333333333,5.135402167651158e-6)
            (2.1866666666666665,4.72604423769936e-6)
            (2.2,5.389620315683631e-6)
            (2.2133333333333334,6.294063001076576e-6)
            (2.2266666666666666,3.32394753140405e-6)
            (2.24,5.784270256841017e-6)
            (2.2533333333333334,3.6074654705934254e-6)
            (2.2666666666666666,5.144212378836289e-6)
            (2.28,5.64813189609432e-6)
            (2.2933333333333334,6.29087395852139e-6)
            (2.3066666666666666,5.517178977438293e-6)
            (2.32,3.1876023868080357e-6)
            (2.3333333333333335,5.415097810344741e-6)
        }
        ;
\end{axis}
\end{tikzpicture} \\
    \hspace{0.65cm} {\small\bf{(a)}}  & \hspace{0.85cm} {\small\bf{(b)}}  & \hspace{1.5cm} {\small\bf{(c)}} \hspace{1.85cm} & \hspace{0.65cm} {\small\bf{(d)}} & \hspace{1.5cm} {\small\bf{(e)}}
\end{tabular}

    \vspace{-0.35cm}
    \caption{\textit{Graph of {\bf (a)} a minimal network of width 4 (teacher) and {\bf (b)} a mildly overparameterized student network of width 5.
     {\bf (c)} With 50 random initializations,  mildly overparameterized networks (blue) find a global minimum for only a fraction of initializations, whereas vastly overparameterized networks (red, width 45) consistently find a global minimum.
    {\bf (d)} Graph of student network with three hidden layer learning from a teacher with widths $ 4-4-4 $. {\bf (e)} Vastly  overparameterized networks (red) consistently find a global minimum whereas mildly overparameterized networks (blue) typically do not. } \label{fig:converge-or-not}
    } 
\end{figure*}

\subsection{Main Contributions}
\begin{enumerate}
    \item Suppose an $ L $-layer Artificial Neural Network (ANN) with hidden layer widths $ r_1^*, \ldots, r_{L-1}^* $ reaches a unique (up to permutation) zero-loss global minimum (we call such a network \emph{minimal} if it cannot achieve zero loss if any neuron is removed). The permutation symmetries give rise to $ r_1^*! \cdots r_{L-1}^*! $ equivalent discrete global minima.
    We show that adding one neuron to each layer is sufficient to connect these global minima into a single zero-loss manifold.
    \item For a two-layer overparameterized network of width $ m = r^* + h $,
   we describe the geometry of the global minima manifold precisely:
   it consists of a union of a number $ T(r^*, m ) $ of affine subspaces of dimension $ \geq h $ and it is connected.
   Furthermore, we show that the global minima manifold contains \textit{all} the zero-loss points for smooth activation functions satisfying a technical condition and in the presence of infinitely many data points with full support of the input space.
    \item
    The symmetries of the network generate symmetry-induced critical points, such as saddle points, which may prevent the convergence to a global minimum (see Figure~\ref{fig:converge-or-not}). We find a surprising scaling relation between the number of subspaces formed by the symmetry-induced critical points and the number of subspaces making up the global minima:
\begin{itemize}
\item When the number of additional neurons satisfies $h\ll r^* $ (i.e. at the beginning of the overparameterized regime), the number of subspaces that make up the global minima manifold is much {\em smaller} than the number of subspaces that make up the symmetry-induced critical points.
In this sense, there is a proliferation of saddles and the global minima manifold is `tiny'.
\item Conversely, when $h\gg r^*$ (i.e. we are far into or within the overparameterized regime), the number of subspaces  that make up the global minima manifold is much {\em greater} than the number of subspaces that make up the symmetry-induced (non-global) critical points. In this sense the global minima manifold is `huge'.
\end{itemize}
\item One may worry that, by adding $ h $ neurons, a saddle of a network of width $ r $ could transform into a local minimum. However, we show that this is not the case and a saddle point in the smaller network transforms into symmetry-induced saddle points.
\end{enumerate}

\subsection{Related Work}

A number of recent works have explored the typical path taken by a gradient-based optimizer.
For very wide ANNs, the gradient flow converges to a global minimum in spite of the non-convexity of the loss \cite{jacot2018neural, du2018gradient, chizat2018global, arora2019fine, du2019gradient, lee2019wide, lee2020finite}. 
First-order gradient algorithms provably escape strict saddles \cite{jin2017escape, lee2019first},
although they can face an exponential slowdown around these saddles \cite{du2017gradient}.
For pruned ANNs, the training with typical (random) initialization does not reach any global minimum, in spite of their presence in the landscape \cite{frankle2018lottery}.

Another body of work focuses on the geometric investigation of neural network landscapes.
\citet{dauphin2014identifying} suggested a proliferation of saddles in ANN landscapes through an analogy with high-dimensional Gaussian Processes.
Other models have been proposed to understand the general structure of ANN landscapes inspired by statistical physics \cite{geiger2019jamming}, and via high-dimensional wedges \cite{fort2019large}.
These model-based empirical works focus mainly on the Hessian spectrum at the critical points.

Another line of work suggests that global minima found by stochastic gradient descent are connected (i.e. there is a path linking arbitrary two minima along which the loss increases only negligibly) via simply parameterized low-loss curves \cite{draxler2018essentially, garipov2018loss}
or line segments \cite{sagun2017empirical, frankle2020linear, fort2020deep}.
Theoretical work limited to ReLU-type activation functions, showed that in overparameterized networks, all global minima lie in a connected manifold   \cite{freeman2016topology, nguyen2019connected}, however without giving a geometrical description of this manifold.
\citet{cooper2020critical} studied the geometry of a subset of the manifolds of critical points.
\citet{kuditipudi2019explaining} showed that the global minima for ReLU networks, for which {\em half} of the neurons can be dropped without incurring a significant increase in loss, are connected via piecewise linear paths of minimal cost.

In this paper, we show that adding or removing a {\em single} neuron radically changes the connectedness without any change in loss. We are the first to prove the connectivity of the global minima manifold for continuously differentiable activation functions.
The focus on symmetries in our work is similar to that of \cite{fukumizu2000local, brea2019weight, fukumizu2019semi} regarding the critical points coming from neuron replications. In an orthogonal direction, \citet{kunin2020neural, gluch2021noether} present a catalog of symmetries appearing in deep networks, which however does not include the permutation symmetry.
To the best of our knowledge, this work is the first to study the scaling of the number of critical points in ANN landscapes as a function of the overparameterization amount.
A key challenge to overcome is the numerous
equivalent
arrangements of neurons inside the network.

\textbf{Notation.} For $ m \geq 1$, set $ [m] = \{ 1, \ldots, m \} $ and let $ S_m $ denote the symmetric group on $ m $ symbols, i.e. the set of permutations of $[m]$.
For a  permutation $\pi \in S_m$ and $ D \geq 1$, the map $ \Perm_{\pi} : \R^{Dm} \to \R^{Dm} $ permutes the units $ \vartheta_i \in \R^D $ of a point $ \ptheta = (\vartheta_1, \ldots, \vartheta_m) $ according to $\pi$, i.e. $ \Perm_{\pi} \ptheta = (\vartheta_{\pi(1)}, \ldots, \vartheta_{\pi(m)}) $; we sometimes use $ \ptheta_\pi := \Perm_\pi \ptheta $.

\section{Symmetric Losses}\label{sec:sym-loss}

\begin{figure}
     \begin{center}
    \includegraphics[width=0.3\textwidth]{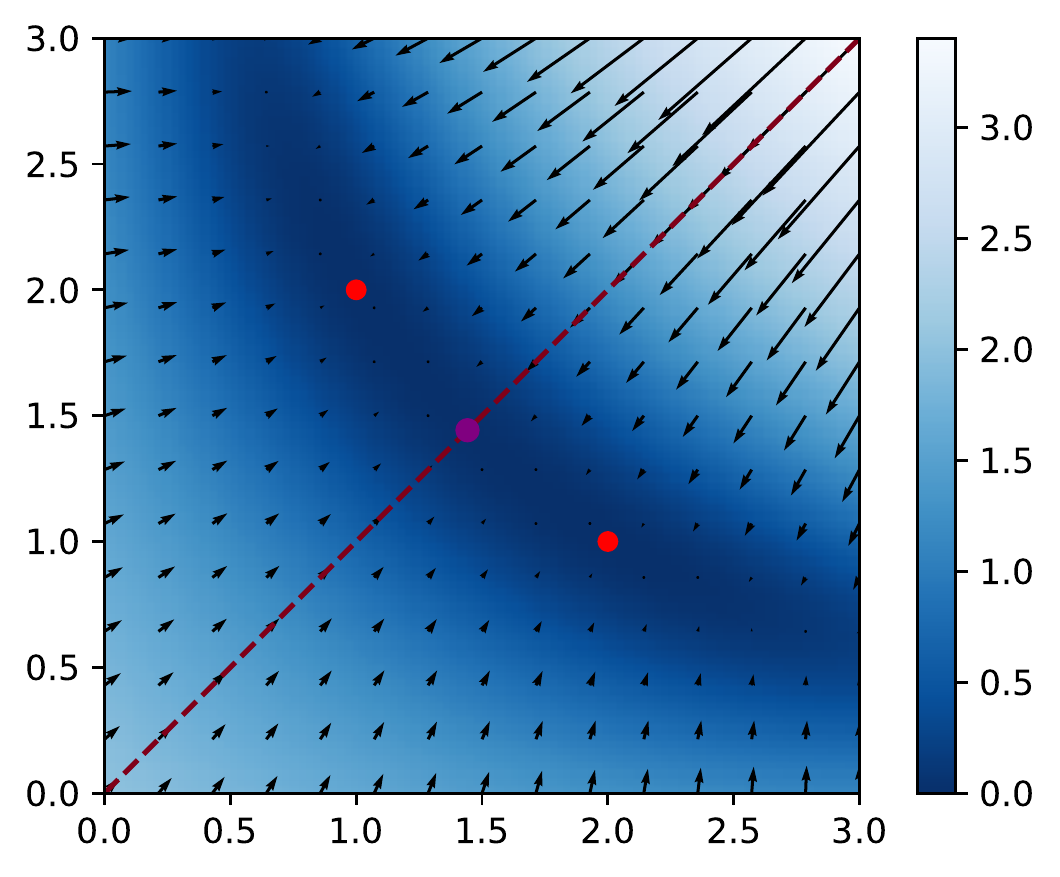}
  \end{center}
  \vspace{-0.3cm}
  \caption{\label{fig:no-gradient-outside} {\small \textit{No gradient pointing outside of a symmetry subspace.} The gradient flow of a permutation-symmetric loss $ L(w_1, w_2)= \log(\frac{1}{2} ((w_1 + w_2 - 3)^2 + (w_1 w_2 - 2)^2 ) + 1) $. Red: permutation-symmetric global minima, purple: saddle, dashed line: the symmetry subspace.}}
\end{figure}

Numerous machine learning models involve permutation-symmetric parameterizations: mixture models,
matrix factorization, and neural networks.
In this section, we abstract away the particular parameterization of these models and focus on the implications of permutation symmetry on the gradient flow. In particular, the discussion here is general and applies to ANNs which is the main focus of this paper.

\begin{definition}\label{def:sym-loss} A loss function $ L^m : \R^{Dm} \to \R $ is a \textbf{\emph{symmetric loss}}\footnote{When the units are $ 1 $-dimensional, symmetric losses are symmetric functions \cite{kung2009combinatorics, sagan2013symmetric}.} on $ m $ units if it is a $ C^1 $ function and if for any $\pi \in S_m$ and  any $ \ptheta = (\vartheta_1, \vartheta_2, \ldots, \vartheta_m) $ with $ \vartheta_i \in \R^D $, we have \vspace{-0.3cm}
  \begin{align*}
      L^m( \ptheta ) = L^m( \Perm_\pi \ptheta  ).
  \end{align*}
\end{definition}

The term \emph{unit} may refer to a Gaussian vector in the context of Gaussian mixture models, to a factor in the context of matrix factorization,
or to a neuron in the context of neural networks.
The symmetry subspaces are defined by the constraint that at least two units are identical:
\begin{definition} Let $ i_1, \ldots, i_k \in [m] $ be distinct indices. The \textbf{\emph{symmetry subspace}} $ \H_{i_1, \ldots, i_k} $ is defined as
\begin{align*}
    \H_{i_1, \ldots, i_k} := \{ (\vartheta_1,  \ldots, \vartheta_m) \in \R^{Dm}: \vartheta_{i_1} = \cdots = \vartheta_{i_k} \}.
\end{align*}
\end{definition}

As each constraint $ \vartheta_i = \vartheta_j $ suppresses $ D $ degrees of freedom, we have $ \dim(\H_{i_1, \ldots, i_k}) = D (m - k + 1) $. The largest symmetry subspaces are $ \H_{i, j} $'s: any other symmetry subspace is the intersection of such subspaces.

Let $ \prho: \Rpz \to \R^{Dm} $ denote the gradient flow under a symmetric loss
  \begin{align}
        \dot{\prho}(t) = - \nabla L^m (\prho(t))
  \end{align}
for $ t \geq 0 $ and for a given initialization $ \prho(0) $.
In Figure \ref{fig:no-gradient-outside}, we observe that the gradient on the symmetry subspace is tangent to it. In general, the gradient components of a symmetry subspace pointing to neighbor regions cancel out due to permutation symmetry.

\begin{lemma}\label{invariant-manifolds} Let $ L^m : \R^{Dm} \to \R $ be a symmetric loss on $ m $ units thus a $C^1$ function and let $ \prho: \Rpz \to \R^{Dm} $ be its gradient flow.
If $ \prho (0) \in \H_{i_1, \ldots, i_k } $, the gradient flow stays inside the symmetry subspace, i.e. $ \prho (t) \in \H_{i_1, \ldots, i_k } $ for all $ t > 0 $.
If $ \prho (0) \notin \mathcal{H}_{i,j} $ for all $i \neq j \in [m]$, that is outside of all symmetry subspaces, the gradient flow does not visit any symmetry subspace in finite time.
\end{lemma}

\begin{remark}
Lemma~\ref{invariant-manifolds} does not exclude the following scenario:
if there is a critical point on the symmetry subspace that is attractive in some directions orthogonal to the symmetry subspace, the gradient flow can reach it in infinite time (i.e. at convergence).
\end{remark}

\section{Foundations: Invariances in 2-Layer ANNs}\label{sec:two-layers}

In this section, we discuss the implications of the permutation symmetry for the ANN landscapes and identify further invariances in network function parameterization. This approach will allow us to describe the precise geometry of the global minima manifold (Subsection~\ref{sec:minima-manifold}) and the symmetry-induced critical points (Subsection~\ref{sec:critical-manifold}) in overparameterized ANNs.

Let $ f^{(2)} : \R^{\din} \to \R^{\dout} $ be a two-layer ANN of width $ m $
\begin{align*}
    \ftwo = \sum_{i=1}^m a_i \sigma( w_i \cdot x)
\end{align*}
where $ \ptheta = (w_1, \ldots, w_m, a_1, \ldots, a_m ) $ is an \emph{$ m $-neuron point} in the parameter space $ \R^{Dm}$ with $ w_i \in \R^{\din} $ and $ a_i \in \R^{\dout} $ so that $ D= \din + \dout $ and $ \sigma : \R \to \R $ is a $ C^1 $ activation function with $ \sigma(x) \neq 0 $ for all $ x \in \R $.\footnote{We exclude homogenous activation functions, such as ReLU and linear function (for linear networks), where the scaling invariance should also be considered.}
Sometimes, we will write $ \ptheta^m := \ptheta $ to emphasize the number of neurons.

The training dataset of size $ N $ is denoted by $ \Trn = \{ (x_k, y_k) \}_{k=1}^N $ where $ x_k \in \R^{\din}, y_k \in \R^{\dout} $. The training loss $ L^m : \R^{Dm} \to \R $ is
\begin{align}\label{eq:loss}
      L^m(\ptheta) = \frac{1}{N} \sum_{(x, y) \in \Trn} c(\ftwo, y)
\end{align}
where $ c : \R^{\dout} \times \R^{\dout} \to [0, +\infty)$ is a single-sample loss that is $ C^1 $ in its first component and $ c(\hat{y}, y) = 0 $ if and only if $ \hat{y} =  y $, such as the least-squares loss or the logistic loss.

\begin{figure}[t!]
  \begin{center}
    \includegraphics[width=0.5\textwidth]{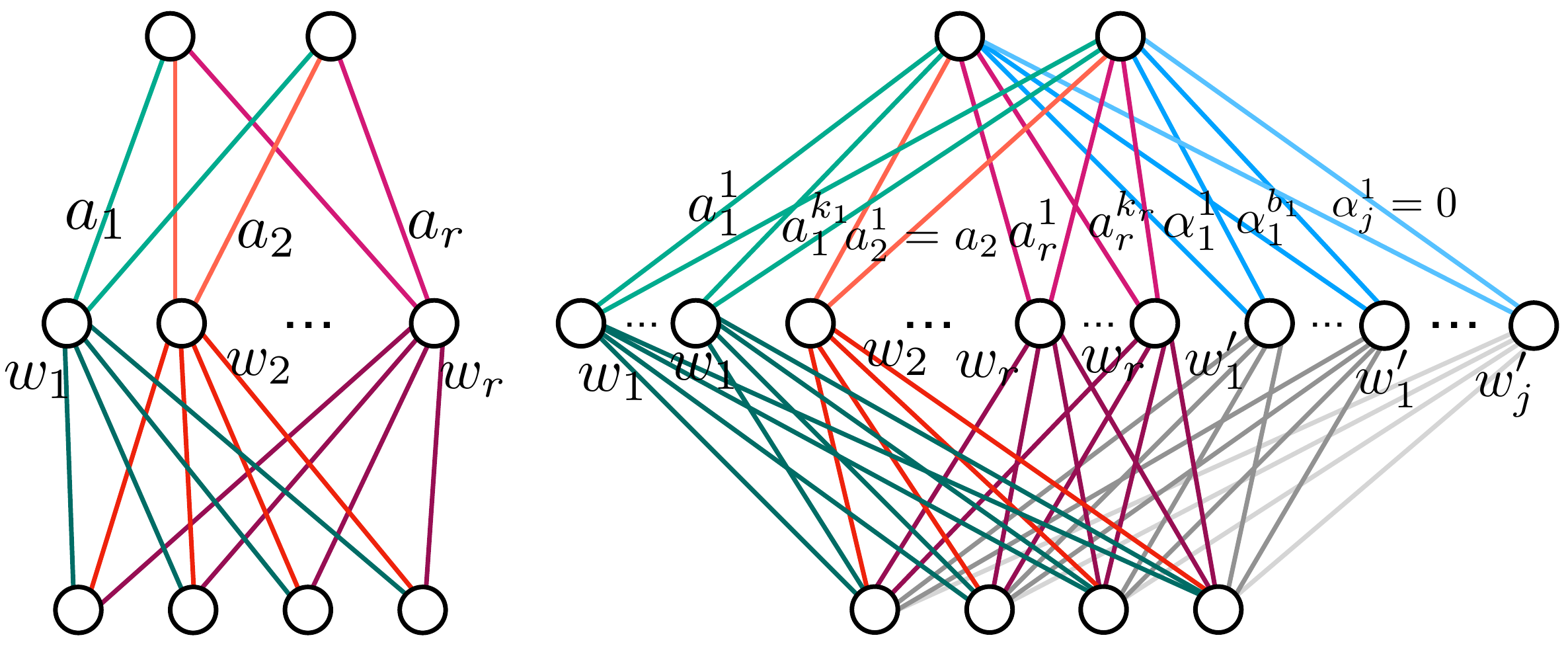}
  \end{center}
  \vspace{-0.2cm}
  \caption{\label{fig:expansion-schem} {\bf Left}: Parameters $\ptheta^r $ of  an irreducible point in a network of $ r $ neurons with $ w_i \neq w_j $ for all $ i \neq j $ and $ a_i \neq 0 $ for all $ i $. \textbf{Right:} example of a reducible point in $ \subs $ in an expanded network of $m>r$ neurons. The incoming weight vector of the first neuron is replicated $k_1$ times, the second one only once, etc.}
\end{figure}

Since $ \ftwo $ is invariant under the permutation of neurons $ \vartheta_i := [w_i, a_i] \in \R^D $, the concatenation of the incoming and outgoing weight vectors,
and both $ \sigma $ and $ c $ are $ C^1 $, $ L^m $ is a symmetric loss (Def.~\ref{def:sym-loss}).
Therefore the symmetry subspaces $ \vartheta_i = \vartheta_j $ are invariant under the gradient flow (Lemma~\ref{invariant-manifolds}).

ANN functions exhibit further invariances:

\begin{definition}\label{def:irreducible} We call an $ m $-neuron point $ \ptheta^m $ \textbf{\emph{irreducible}} if it has $ m $ distinct incoming weight vectors $ w_i $, and no zero outgoing weight vector, i.e. $ a_i \neq 0 $ for all $i\in[m]$.
Otherwise we say that $ \ptheta^m $ is \textbf{\emph{reducible}}.
\end{definition}

Any reducible point $ \ptheta^m $ is equivalent to a point $ \ptheta^{m-1} $ with $ (m-1) $-neurons in that they produce the same function $ f^{(2)}(x | \ptheta^m) = f^{(2)}(x | \ptheta^{m-1}) $ where $ \ptheta^{m-1} $ is
\begin{enumerate}
\item $ (w_2, w_3, \ldots, w_m, a_1 + a_2, a_3, \ldots, a_m) $ if $ w_1 = w_2 $,
\item $ (w_2, w_3, \ldots, w_m, a_2, a_3, \ldots, a_m) $ if $ a_1 = 0 $.
\end{enumerate}

Note that because of permutation symmetry, the above reductions hold whenever two incoming weight vectors are equal, i.e. $ w_i = w_j $, or any one of the outgoing vectors is zero $ a_i = 0 $.
Moreover, if $ \ptheta^{m-1}$ is also reducible, we can continue dropping neurons as above until we find an irreducible point $ \ptheta^{r} $.
Equivalently (going in the opposite direction), an irreducible $ r $-neuron point $$ \ptheta^r = (w_1, \ldots, w_r, a_1, \ldots, a_r) $$ yields an affine subspace of equal loss points in a network with width $ m \geq r $ (see Figure~\ref{fig:expansion-schem}):

\begin{definition}\label{def:affine-subs-Gamma} For $r\geq 1,\, j\geq0$ with $r+j\leq m$,  let  $ s = (k_1, ..., k_r, b_1, ..., b_j ) $ be an $ (r + j) $-tuple of integers such that
$ \summ(s) := k_1 + ... + k_r + b_1 + ... + b_j = m $
with $ k_i \geq 1 $ and $ b_i \geq 0 $. The \textbf{\emph{affine subspace}} $ \subs $ of an irreducible point $\ptheta^r$ is
\begin{align}\label{eq:affine-subs0}
    \{ &(\underbrace{w_1, ..., w_1}_{k_1}, ..., \underbrace{w_r, ..., w_r}_{k_r}, \underbrace{w_1', ..., w_1'}_{b_1},
    ..., \underbrace{w_j', ..., w_j'}_{b_j}, \notag \\
    &a_1^1, ..., a_1^{k_1}, ...,
   a_r^1, ..., a_r^{k_r}, \alpha_1^1, ..., \alpha_1^{b_1},
   ..., \alpha_j^1, ..., \alpha_j^{b_j}): \notag \\
   &\text{where} \ \sum_{i=1}^{k_t} a_t^{i} = a_t \ \text{for} \ t \in [r] \ \text{and} \
   \sum_{i=1}^{b_t} \alpha_t^{i} = 0 \ \text{for} \ t \in [j] \}.
\end{align}
\end{definition}

Note that all $ \ptheta^m \in \subs $ implement the same function:
\begin{align*}
    f^{(2)}(x | \ptheta^m) &= \sum_{t=1}^r \sum_{i=1}^{k_t} a_t^i \sigma(w_t \cdot x) + \sum_{t=1}^j \sum_{i=1}^{b_t} \alpha_t^i \sigma(w'_t \cdot x) \\
    &=  f^{(2)}(x | \ptheta^r).
\end{align*}

\begin{figure*}[t!]
    \centering
    \hspace{0.5cm}
    \subfloat{
        \includegraphics[width=0.22\textwidth]{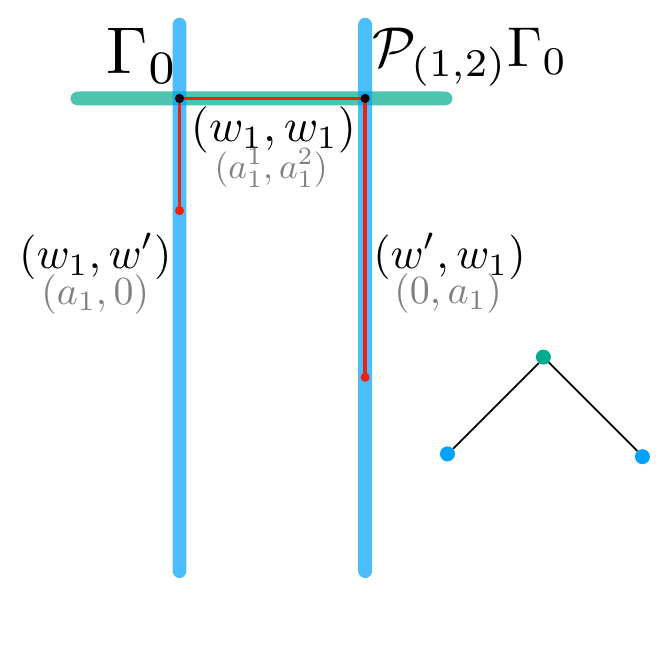}
        }
        \hspace{0.4cm}
    \subfloat{
        \includegraphics[width=0.38\textwidth]{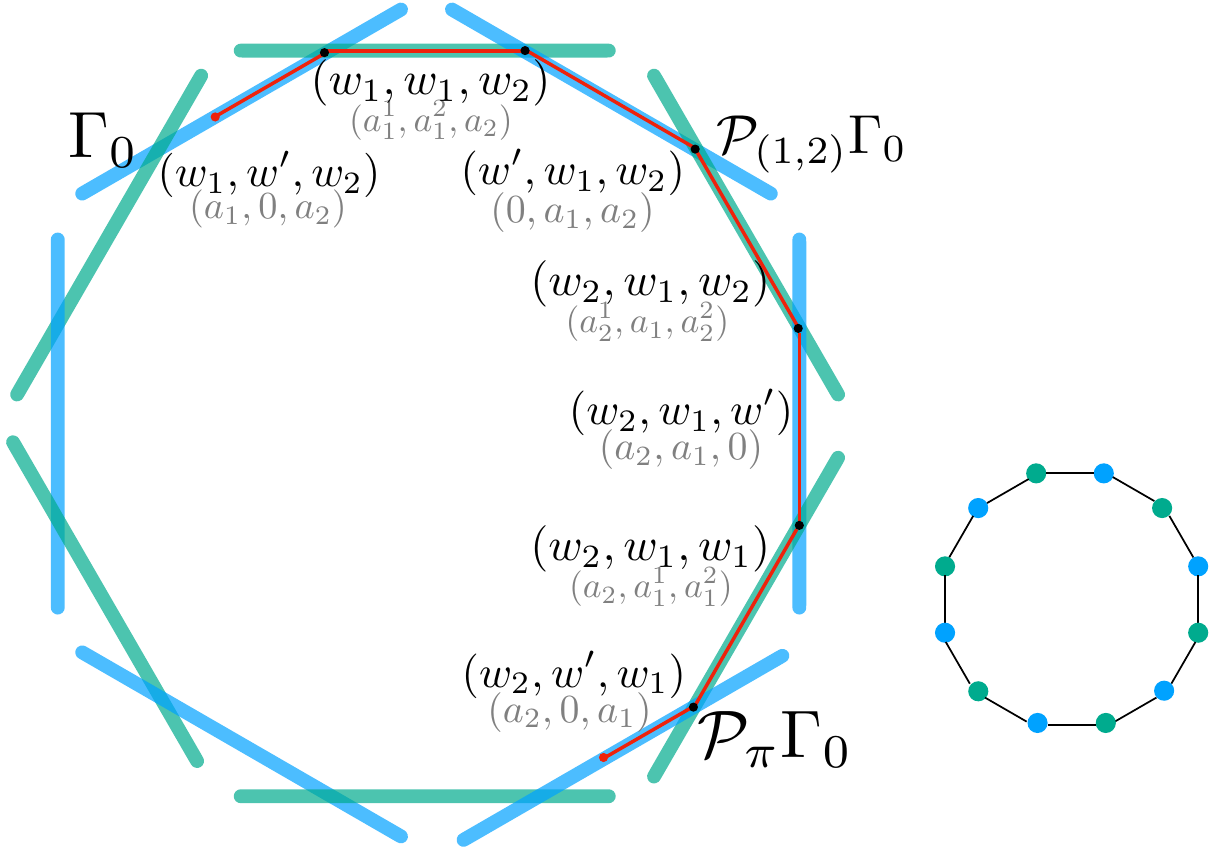}
        }
        \hspace{0.5cm}
    \subfloat{
        \includegraphics[width=0.26\textwidth]{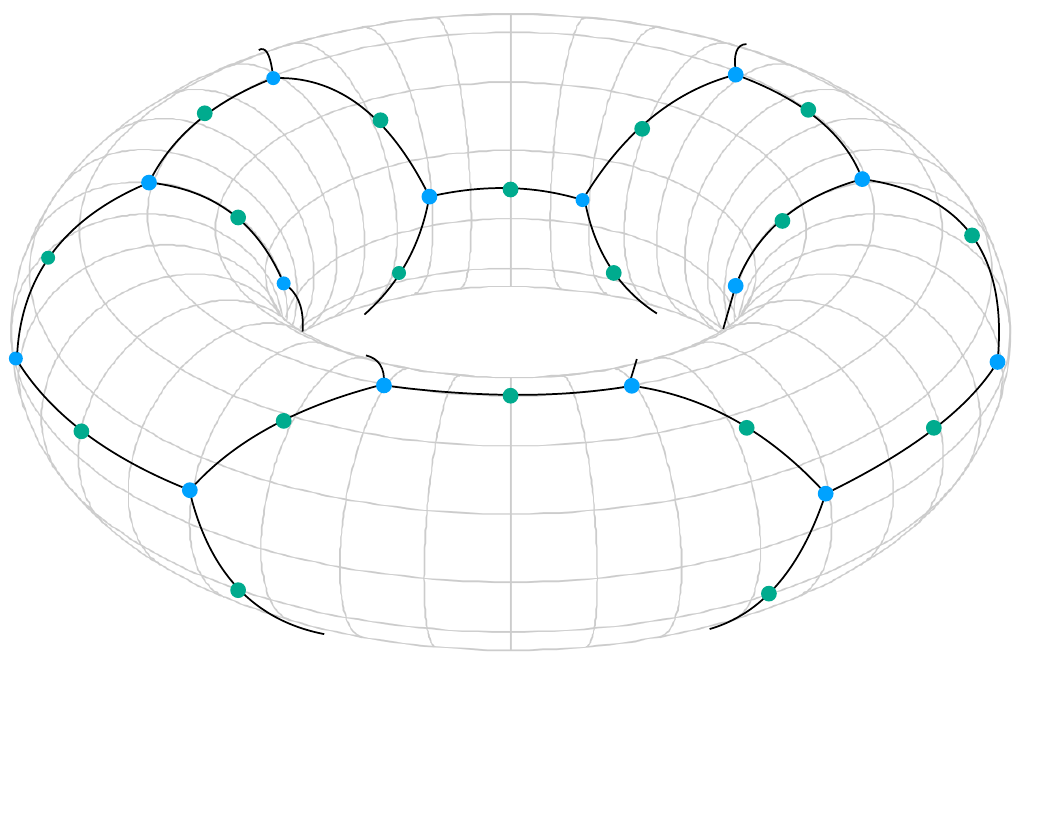}
          }
        \linebreak
      \vspace{-0.2cm}
      {\small \hspace{0.3cm} \textbf{(a)} $ \Theta_{1 \to 2}(\ptheta^{1}) $  \hspace{3.7cm} \textbf{(b)} $ \Theta_{2 \to 3}(\ptheta^{2}) $ \hspace{3.8cm} \textbf{(c)} $ \Theta_{3 \to 4}(\ptheta^{3}) $ }
    \caption{ \textit{The geometry of the expansion manifold $ \Theta_{r \to m} $ with $m=r+1$ and the connectivity graph of the affine subspaces.}  The arrangement of the subspaces is demonstrated geometrically only in (a)-(b), but their connectivity graph is shown in all three cases.
    Blue subspaces have one vanishing output weight, green subspaces have two identical incoming weight vectors.
    \textbf{(a)} Case of a network with two hidden neurons with parameters $(w_1, w', a_1, 0)$ that is reducible to a network with a single hidden neuron. The base subspace $ \Gamma_0 $ is connected to a neighbor subspace $ \Perm_{(1,2)} \Gamma_0 $ via three line segments: we first shift $w'$ towards $w_1$ while keeping the other parameters fixed and then move $a_1^1$ from $a_1$ to $0$ while keeping $a_1^1+ a_1^2=a_1$.
    The connectivity graph (bottom right) shows each subspace as an appropriately colored dot.
    \textbf{(b)} Case of a network with three hidden neurons with parameters $(w_1, w', w_2, a_1, 0, a_2)$ that is reducible to a network with two hidden neurons. $ \Gamma_0 $ is connected to any other subspace $ \Perm_\pi \Gamma_0 $ through transitions from one neighbor to the next.
    Note that there are $ T(2, 3) = 12 $ subspaces.
    \textbf{(c)} The connectivity graph of subspaces for the expansion $ 3 \to 4 $, there are $ T(3, 4) = 60 $ subspaces ($ 24 $ blue and $ 36 $ green), where each blue subspace is connected to three green subspaces and each green subspace is connected to two blue subspaces. \label{fig:connected-manifold}}
\end{figure*}

Neurons with incoming weight vectors $ w' $ and outgoing weight vectors adding up to zero are called in the following {\bf `zero-type' neurons}.
Moreover, the network function remains invariant under any permutation of neurons in Definition~\ref{def:affine-subs-Gamma}. Each permutation defines another affine subspace
$$
 \Perm_\pi \subs := \{ \Perm_\pi \ptheta^m : \ptheta^m \in \subs \ \text{and} \ \pi \in S_m \}
$$
where $ \Perm_\pi $ permutes the neurons $ \vartheta_i = [w_i, a_i] $ of $ \ptheta^m $. We call the union of these affine subspaces the expansion manifold of $ \ptheta^r $:
\begin{definition}\label{def:exp-man} For $ r \leq m $, the \textbf{\emph{expansion manifold}}  $ \expman \subset \R^{Dm} $
of an irreducible $ r $-neuron point $ \ptheta^r $ is defined by
\begin{align*}
   \expman := \bigcup_{\substack{s = (k_1, \ldots, k_r, b_1, \ldots, b_j) \\ \pi \in S_m}} \Perm_\pi \subs,
\end{align*}
where $ s $ is a tuple with $k_i \geq 1 $, $ b_i \geq 0$ such that $ \summ(s) = m $.
\end{definition}

Since $ \expman $ is an equal-loss manifold, the gradient flow can cross it at most for once.
Therefore $\expman$ is not an invariant manifold like the symmetry subspaces.
Next, we describe the precise geometry of the expansion manifolds
\begin{theorem}\label{thm:geometry-of-exp-manifold} For $ m \geq r $, the expansion manifold $ \expman $ of an irreducible 
point $\ptheta^r $ consists of exactly\footnote{$ \binom{n_1 + \cdots + n_r}{n_1, ..., n_r} $ denotes the coefficient $\frac{ (n_1 + \cdots + n_r)! }{n_1! ... n_r!}$.}
 \begin{align*}
 T(r, m) := \sum_{j = 0}^{m-r}  \sum_{\substack{\summ(s) = m \\ k_i \geq 1, b_i \geq 1}} \binom{m}{k_1, ..., k_r, b_1, ..., b_j} \frac{1}{c_1! ... c_{m-r}!}
 \end{align*}
distinct affine subspaces (none is including another one) of dimension at least $\min(\din,\dout)(m-r)$, where $ c_i $ is the number of occurences of $ i $ among $ (b_1, ..., b_j) $.

For $ m > r $, $ \expman $ is connected: any pair of distinct points $ \ptheta, \ptheta' \in \expman $ is connected via a union of line segments $ \pgamma : [0,1] \to \expman $ such that $ \pgamma(0) = \ptheta $ and $ \pgamma(1) = \ptheta' $.
\end{theorem}

\emph{Proof (Sketch).} The number of affine subspaces $ T $ is equal to the distinct permutations of the incoming weight vectors
$ (w_1, \ldots, w_r, w_1', \ldots, w_j') $
for all possible tuples $ s $ where $ w_i $'s are distinct and $ w_i' $'s are dummy variables representing \emph{zero-type} neurons (the neurons that do not contribute to the network function since their outgoing weight vectors sum to zero).
The normalization factor $ 1 / c_1! c_2! \cdots c_{m-r}! $  cancels  the repetitions coming from the zero-type neurons $ (w_1', \ldots, w_j') $.
For example for the standard case $ m = r $, there is no room for zero-type neurons.
As a result we have
 $$ T(r, r) = \sum_{\substack{k_1 + ... + k_r = r \\ k_i \geq 1}} \binom{r}{k_1, ..., k_r} =  \binom{r}{1, ..., 1} = r! $$
distinct subspaces of dimension $ \min(\din, \dout)(m - r) = 0 $.

For the general case $ m > r $,
the proof for connectivity follows from the following observations.
We start from a base subspace $ \Gamma_0 = \subs $, where there is a zero-type neuron with outgoing weight vector exactly zero\footnote{If all zero-type neurons are part of a group with more than one neuron, we can choose the first neuron in a group and set its outgoing weight vector to zero while respecting the condition in Eq.~\ref{eq:affine-subs0}.} at position $ i^* $.
The neighbor subspaces $ \Perm_{(i^*, i)} \Gamma_0 $, where $ (i^*, i) \in S_m $ is a transposition that permutes two neurons only, are connected to the base subspace via three line segments  (Figure~\ref{fig:connected-manifold}-a).
Since any permutation is a composition of transpositions, permuted subspaces $ \Perm_\pi \Gamma_0 $ can be reached via a union of line segments by going from one neighbor to the next  (Figure~\ref{fig:connected-manifold}-b). $\blacksquare$

\section{Overparameterized ANN Landscapes}

In this section, we study the geometry of the global minima manifold and the critical subspaces, i.e. affine subspaces containing only critical points, in two-layer overparameterized neural networks.
In particular, we show how the affine subspaces that form the global minima manifold are connected to one another (Subsection~\ref{sec:minima-manifold}).
We then find a hierarchy of saddles induced by permutation symmetries, which we call symmetry-induced critical points (Subsection~\ref{sec:critical-manifold}). Finally, we compare the number of affine subspaces that form the global minima manifold with the number of those that contain symmetry-induced critical points (Subsection~\ref{sec:scaling}).
Generalizations to multi-layer networks are discussed in Section~\ref{sec:multiple-layers}.

We assume that there is a minimal width $ r^* $ such that $ \ptheta_* $ achieves zero loss, i.e. $ L^{r^*}(\ptheta_*) = 0 $, that the point $ \ptheta_* $ is unique up to permutation, and that any network with width $ r^* - 1 $ has loss $ >0 $ at every point.
We call the wider networks with width $ m > r^* $ {\bf overparameterized} and the narrower networks with width $ r < r^* $ {\bf underparameterized}.
Note that $ \ptheta_* $ is irreducible by minimality of $ r^* $.

\subsection{The global minima manifold}\label{sec:minima-manifold}

Applying Theorem~\ref{thm:geometry-of-exp-manifold} to the expansion manifold of a global minimum $ \ptheta_* $ of the minimal-width network, we obtain a connected manifold of global minima in an overparameterized network of width $ m $:

\begin{corollary}\label{cor:exp-manifold} In an overparameterized network with width $  m > r^* $, the expansion manifold of global minima $ \expmang $ is connected.
\end{corollary}

We have found a connected manifold $ \expmang $ of global minima. Furthermore, since $ \expmang $ is an expansion manifold, its geometry is precisely as described in Theorem~\ref{thm:geometry-of-exp-manifold}, and illustrated in Figure~\ref{fig:connected-manifold}. The next question is whether $ \expmang $ contains all the zero-loss points.

In the remaining part of this subsection, we give a positive answer to this question in a specific setting. We consider a modified loss function:
\begin{align*}
    L_\mu^m(\ptheta) = \int_{\R^{\din}} c(f^{(2)}(x | \ptheta), f^*(x)) \mu (dx),
\end{align*}
where $ \mu $ is an input data distribution with support $ \R^{\din} $, and $ f^{*} : \R^{\din} \to \R^{\dout}$ is a true data-generating function. The assumption on the activation $\sigma$ in Theorem~\ref{thm:all-global-minima} below is only required for this theorem but not in Subsections \ref{sec:critical-manifold} or \ref{sec:scaling}. We find that there is no global minimum point outside of the expansion manifold $ \expmang $ for the modified loss $ L_\mu^m $ and for a certain class of activation functions (see Figure~\ref{fig:new-activ} for an example):

\begin{theorem}\label{thm:all-global-minima} Suppose that the activation function $ \sigma $ is $C^\infty$, that $\sigma(0)\neq0 $, and that $\sigma^{(n)}(0)\neq0$ for infinitely many even and odd values of $n$ (where $\sigma^{(n)}$ denotes the $n$-th derivative of $ \sigma $). For $ m > r^* $,
let $ \ptheta $ be an $ m $-neuron point, and $ \ptheta_* $ be a unique $ r^* $-neuron global minimum up to permutation, i.e. $ L_\mu^{r^*}(\ptheta_*) = 0 $.
If $ L_\mu^m(\ptheta) = 0 $, then $ \ptheta \in \expmang $. (See Appendix-B.3 for the proof.)
\end{theorem}

\begin{remark}
The function
$\sigma_{\alpha,\gamma}(x)
   = \sigma_{\text{soft}}(x) + \alpha \sigma_{\text{sig}}(\gamma x) $
   with $\alpha,\gamma>0$ (Figure~\ref{fig:new-activ})
   satisfies the conditions of Theorem~\ref{thm:all-global-minima}, but the standard softplus $\sigma_{\text{soft}}(x) = \ln[1+\exp(x)] $ or sigmoidal $\sigma_{\text{sig}}(x) = 1/ [1+\exp(-x)] $ functions do not.
 For these, the analysis must include additional invariances.
\end{remark}
\begin{remark}If a global minimum is found by gradient descent in overparameterized networks, then the final set of parameters can be classified into groups of replicated weight vectors according to Definition~\ref{def:affine-subs-Gamma} (Figure~\ref{fig:new-activ}). The classification can be exploited for pruning the network.
\end{remark}

\begin{figure}[t!]
    \newcommand{\setboundingbox}{
        \pgfresetboundingbox
    \path
      (current axis.south west) -- ++(-0.4in,-0.4in)
      rectangle (current axis.north east) -- ++(0.1in,0.1in);
  }

\pgfplotsset{width=4.25cm, height = 4.25cm}
\centering
\begin{tabular}{cc}
    \begin{tikzpicture}
        \begin{axis}[
            no marks,
            xlabel = $x$,
            ylabel = $\sigma_{1,4}(x)$,
            legend pos = north west,
            samples = 200,
            ]
            \addplot[very thick] {1/(1 + exp(-4*x)) + ln(exp(x) + 1)};
        \end{axis}
        \setboundingbox
    \end{tikzpicture}
    &
    \begin{tikzpicture}
        \begin{axis}[
            ybar stacked,
            xtick = {0, 1},
            xticklabels = {copies, 0-type},
            ylabel = number of neurons,
            xmin = -.5, xmax = 1.5,
            ]
            \addplot[draw = none, fill = gray] coordinates {(0, 363) (1, 0)};
            \addplot[draw = none, fill = gray]  coordinates {(0, 0) (1, 34)};
            \addplot[draw = none, fill = gray!20] coordinates {(0, 0) (1, 54)};
            \addplot[draw = none, fill = gray] coordinates {(0, 0) (1, 36)};
            \addplot[draw = none, fill = gray!20] coordinates {(0, 0) (1, 16)};
            \addplot[draw = none, fill = gray] coordinates {(0, 0) (1, 5)};
            \node at (.6, 0) {1};
            \draw (.7, 0) -- (.86, 17);
            \node at (.6, 50) {2};
            \draw (.7, 50) -- (.86, 61);
            \node at (.6, 100) {3};
            \draw (.7, 100) -- (.86, 106);
            \node at (.6, 150) {4};
            \draw (.7, 150) -- (.86, 132);
            \node at (.6, 200) {5};
            \draw (.7, 200) -- (.86, 142.5);
        \end{axis}
        \setboundingbox
    \end{tikzpicture} 
\end{tabular}
    \vspace{-0.35cm}
    \caption{\label{fig:new-activ}
  {\bf Left:} The function
  $\sigma_{\alpha,\gamma}(x)
  = \sigma_{\text{soft}}(x) + \alpha \sigma_{\text{sig}}(\gamma x) $
  satisfies the technical condition of Theorem~\ref{thm:all-global-minima}.
  With this activation function, data is generated by a teacher network of width 4. All 50 student networks with width 10 find a global minimum by reaching loss values below $ 10^{-16} $.
  {\bf Right:}
  The 500 = 50$\times$10 hidden neurons of all the 50 student networks are classified as copies of teacher neurons or zero-type neurons with vanishing sum of output weights. The zero-type neurons are further classified according to group size: there are 34 neurons with vanishing output weight (group size 1), 54 neurons that have a partner neuron with the same input weights and the sum of output weights equal to 0 (group size 2) etc. All zero-type neurons and replications of weight vectors can be pruned.
  }
\end{figure}
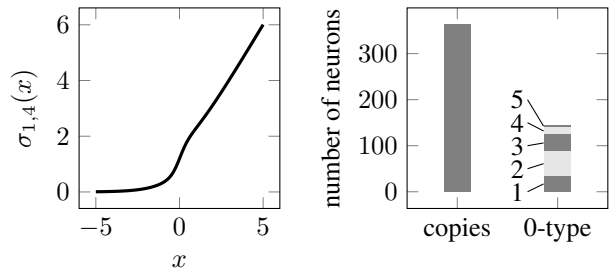

\begin{remark} \citet{kuditipudi2019explaining} construct an example of a finite-size dataset (in contrast with our infinite dataset framework) for two-layer overparameterized ReLU networks where they find \emph{discrete} global minima points.
\end{remark}

\subsection{Symmetry-induced critical points}\label{sec:critical-manifold}

In this subsection, we consider an overparameterized network with a fixed width $ m > r^*$ and study critical points in an expansion manifold $ \expmanfc $ where we assume that $ \ptheta^r_* $ is an irreducible critical point of an underparameterized network with width $ r < r^* $.
Observe that $ \ptheta^r_* $ is not a zero-loss point since $ r^* $ is the minimal width to achieve zero loss.
We consider only those points without zero-type neurons in $ \expmanfc $, we show that these have zero gradient, and therefore are critical points of $ L^m $.

\begin{definition} For $ r \leq m $, let $ s = (k_1, \ldots, k_r ) $ be an $ r $-tuple with $ k_i \geq 1 $ and $ \summ(s) = m $. The \textbf{\emph{symmetry-induced critical points}} are those in the set
\begin{align*}
   \expmanc &= \bigcup_{\substack{s = (k_1, \ldots, k_r) \\ \pi \in S_m}} \Perm_\pi \subsc
\end{align*}
where the critical (affine) subspace $ \subsc \subset \R^{Dm} $ of an irreducible critical point $ \ptheta^r_* = (w_1^*, ..., w_r^*, a_1^*, ..., a_r^*)$ is
\begin{align}\label{eq:affine-subs}
    &\{ (\underbrace{w_1^*, ..., w_1^*}_{k_1}, ..., \underbrace{w_r^*, ..., w_r^*}_{k_r}, \beta^{1}_{1} a_1^*, ..., \beta^{k_1}_{1} a_1^*, \notag \\
    & ..., \beta^{1}_{r} a_r^*,
    ..., \beta^{k_r}_{r} a_r^*):
   \sum_{i=1}^{k_t} \beta_t^{i} = 1 \ \text{for} \ t \in [r] \}.
\end{align}
\end{definition}

All points in $ \expmanc $ are critical points hence the name symmetry-induced `critical points':

\begin{proposition}\label{thm:crit-basins} For an irreducible critical point $ \ptheta^r_* $ of $ L^r $, $ \expmanc $ is a union of
  $$
   G(r, m) := \sum_{\substack{k_1 + \cdots + k_r = m \\ k_i \geq 1}} \binom{m}{k_1, ..., k_r}
  $$
distinct non-intersecting affine subspaces of dimension $ m\!-\!r $. All points in $ \expmanc $ are critical points of $ L^m $.
\end{proposition}

Proposition~\ref{thm:crit-basins} shows that a critical point of a smaller network $ \ptheta^r_* $ expands into $ G(r, m) $ critical subspaces in the overparameterized network with width $ m $.
If $ \ptheta^r_* $ is a strict saddle, $ \expmanc $ contains only strict saddles, since the escape direction is preserved for affine transformations $ \gbar{\Gamma}_{s} $.

\begin{proposition}\label{thm:spectra} For $C^2$ functions $ c $ and $ \sigma $, for all $ \ptheta_*^m \in \expmanc $, the spectrum of the Hessian $ \nabla^2 L^m (\ptheta^m_*) $ has $ (m - r) $ zero eigenvalues.
Moreover, if $ \ptheta_*^r $ is a strict saddle, then all points in $ \expmanc $ are also strict saddles, i.e., their Hessian has at least one negative eigenvalue.
\end{proposition}

If $ \ptheta^r_* $ is a local minimum, \citet{fukumizu2019semi} show that the subspaces for which only one neuron is replicated ($ k_i > 1$, $ k_j = 1 $ for all $ j \neq i $) may contain both local minima and strict saddles depending on the spectrum of a matrix of derivatives [see their Theorem 11]. We expect a similar result to hold true for all subspaces in $ \expmanc $, including arbitrary replications.

\begin{remark} We explore a hierarchy between symmetry-induced critical points in $ \cup_{r < r^*} \expmanc $ in a network of width $ m $: \textbf{\emph{first-level}} saddles refer to symmetry-induced critical points that are equivalent to a minimum of a network of width  $ r^* - 1 $; more generally, \textbf{\emph{$ k $-th level}} saddles refer to those equivalent to a minimum of a network of width $ r^* - k $.
Adding neurons enables the network to reach a lower loss minimum thus higher-level symmetry-induced saddles usually attain higher losses.
We notice a similarity with Gaussian Process \cite{bray2007statistics} and spherical spin glass \cite{auffinger2013random} landscapes, where the higher-\emph{order}\footnote{The order of a saddle point is the number of negative eigenvalues of its Hessian.} saddles attain higher losses.
\end{remark}

Finally, we note that the dimensionality of the global minima subspaces $ \Perm_\pi \subsg $  and the critical subspaces $ \Perm_\pi \subsc $ differ, in particular in the way they depend on $ r $. What is common is that they are all `tiny' compared to the ambient dimensionality of the parameter space.
In the following subsection, we will thus focus on the comparison of the number of critical subspaces and that of global minima subspaces.

\subsection{Width-dependent comparison of the critical subspaces and the global minima subspaces}\label{sec:scaling}

\begin{figure}[t!]
    \centering
    \subfloat{
        \includegraphics[width=0.4\textwidth]{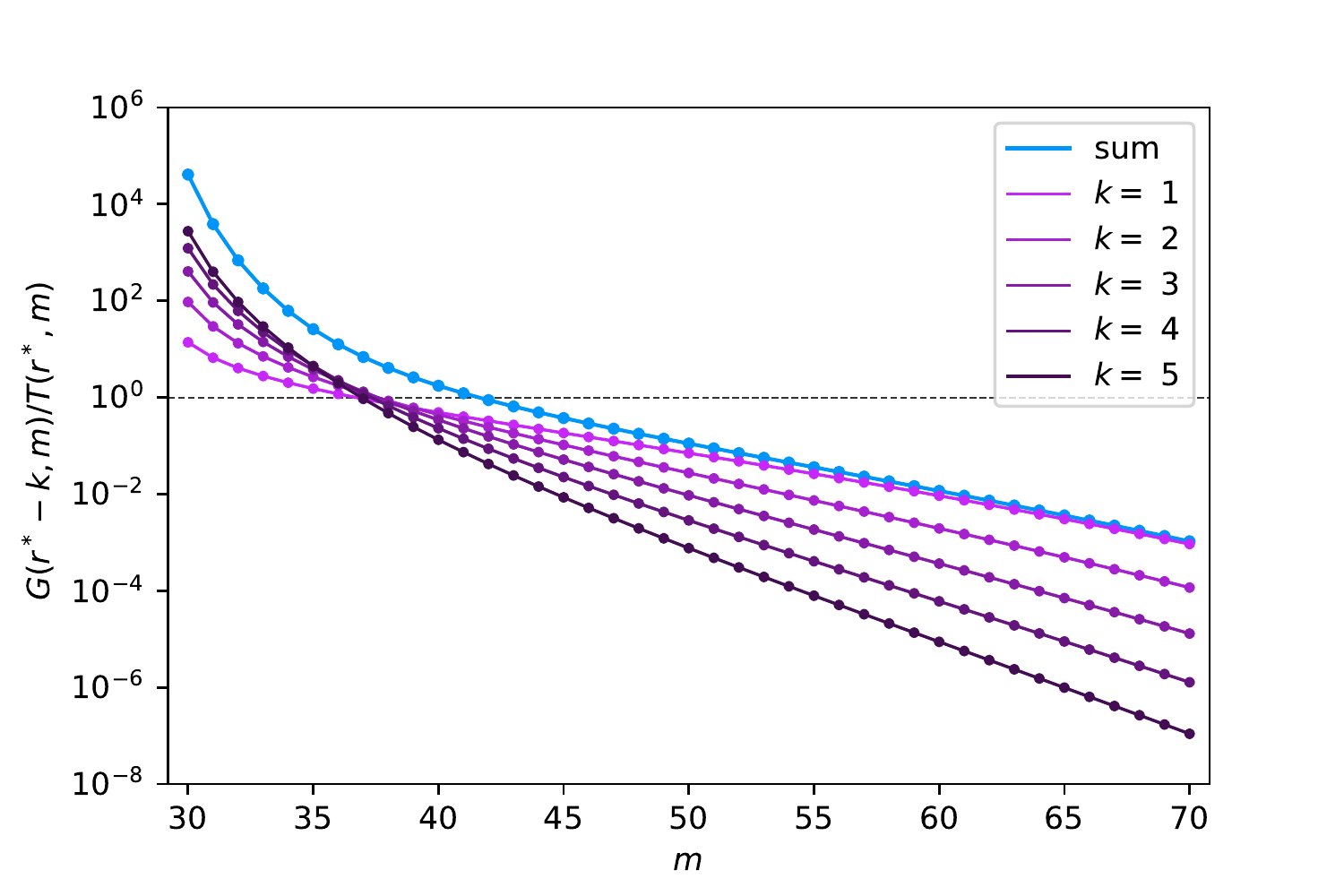}
          }
      \vspace{-0.2cm}
    \caption{\label{fig:scaling-of-critics}
    \textit{The ratio $ \mathrm{R}_k(r^*, m) $ of the multiplier for $k$-th level saddle $ G(r^* - k, m) $ to the number of global minima subspaces $ T(r^*, m) $ as the width $ m $ of the overparameterized network increases, plotted for a fixed width $r^*=30$ of the minimal network.}
    The ratio of all critical subspaces to the global minima subspaces $ \sum_{k=1}^{r^* - 1} a_k G(r^* - k, m) / T(r^*, m) $ is shown in blue assuming $ a_k = 1 $ for all $ k $.
    Note that for $ m \gg r^*$ the blue curve approaches the curve for $k=1$ indicating that only subspaces corresponding to first-level saddles are potentially relevant, yet the global minima subspaces clearly dominate.}
\end{figure}

In the loss landscape of an overparameterized network of width $ m $, we have the connected global minima manifold $ \expmang $ as well as many subspaces of symmetry-induced critical points in $ \expmanc $, where $ \ptheta^r_* $ is an irreducible critical point in a smaller network with some width $ r = r^* - k < r^* $.
In this subsection,  we count these subspaces and find
\vspace{-0.25cm}
\begin{itemize}
    \item $ T(r^*, m) $ global minima subspaces  (Corollary~\ref{cor:exp-manifold})
    \item $ G(r^* - k, m) a_k $ critical subspaces
    for all $ k = 1, \ldots, r^* - 1 $ (Proposition~\ref{thm:crit-basins})
\end{itemize}
\vspace{-0.25cm}
where $ a_k $ is the number of distinct\footnote{We say two irreducible critical points $ \ptheta_*^a $ and $ \ptheta_*^b $ are distinct if $ \ptheta_*^a \neq \Perm_\pi \ptheta_*^b $ for all applicable permutations $ \pi $.} irreducible critical points in a network with width $ r^* - k $ and where $ G(r^* - k, m) $ is the multiplier.

To compare the number of non-global critical subspaces with the number of global minima subspaces, we give closed-form formulas for $ G $ and $ T $. This is proven in Appendix-B.5 using Newton's series for finite differences \cite{milne2000calculus} and a counting argument:

\begin{proposition}\label{prop:formulas} For $ r \leq m $, we have
  \begin{align*}
        &G(r, m) = \sum_{i=1}^r \binom{r}{i}(-1)^{r - i}i^m \\
        &T(r, m) = G(r, m) + \sum_{u=1}^{m-r} \binom{m}{u} G(r, m-u) g(u)
  \end{align*}
where $ g(u) = \sum_{j=1}^{u} \frac{1}{j!} G(j, u) $.
\end{proposition}

Using Proposition~\ref{prop:formulas}, we find the following asymptotic behaviors for $ G $ and $ T $:
\begin{lemma}\label{lem:limiting-behavior0} For any $ k \geq 0 $ fixed, we have,
  \begin{align*}
        G(m - k, m) \sim T(m - k, m) \sim \frac{m^k}{ 2^k k!} m!, \ \text{as} \ m \to \infty.
  \end{align*}
For any fixed $ r \geq 0 $, we have $ G(r, m) \sim r^m $ as $ m \to \infty $.
\end{lemma}

We are now ready to compare the number of global minima subspaces $ T $ with the number of critical subspaces $ G $ under the assumption that the minimal width $ r^* $ is large.
This is realistic since for a real-world dataset the network should be sufficiently wide to achieve zero loss.
Applying Lemma~\ref{lem:limiting-behavior0}, we find that the symmetry-induced critical points dominate the global minima in mildly overparameterized, and vice versa in vastly overparameterized networks (see Figure~\ref{fig:scaling-of-critics}). A mathematical analysis yields:
\newline
\vspace{-0.5cm}

\textbf{Mildly Overparameterized.} Let $ m = r^* + h $ for fixed $ h $. We have in the limit $ r^* \to \infty $ and for fixed $ k $ a ratio:
\begin{align}\label{eq:mild}
    \mathrm{R}_k(r^*, m) := \frac{G(r^* - k, m)}{T(r^*, m)} \sim \frac{1}{2^k (h+k) \cdots (h + 1)} (r^*)^k.
\end{align}
Thus, for a small amount of overparameterization, the multiplier of the $k$-th level saddles $ G(r^* - k, m) $ scales as $ (r^*)^k T(r^*, m) $, indicating a proliferation of saddles at a rate much larger than that of the global minima.
Related to this proliferation, we empirically encounter training failures (i.e. training halts before reaching a global minimum) for typical initializations in this regime (see Figure~\ref{fig:converge-or-not}). Moreover, we empirically find traces of approaching a saddle in gradient trajectories in narrow two-layer ANNs trained on MNIST (see Appendix).
\newline
\vspace{-0.5cm}

\textbf{Vastly Overparameterized.} For $ m $ very large, i.e. $ m \gg r^* $, we have
\begin{align}\label{eq:excessive}
    \sum_{k=1}^{r^* - 1} \mathrm{R}_k(r^*, m) a_k = \frac{\sum_{k=1}^{r^* - 1}G(r^*-k, m)  a_k}{T(r^*, m)} \leq \left(\frac{r^* - 1}{r^*}  \right)^m
\end{align}
if $ a_k $'s satisfy $ a_k \leq \binom{r^* - 1}{k-1} $.
Because the RHS of Eq.~\eqref{eq:excessive} decreases down to $ 0 $ as $ m \to \infty $ (at a geometric rate), the global minima dominate \textit{all} symmetry-induced critical points.
We note that there could be other critical points in addition to those generated by the symmetries.
The calculations above are presented in the Appendix.

\section{Multi-Layer ANNs}\label{sec:multiple-layers}

In this section, we introduce the expansion manifold for multi-layer networks that enables obtaining connectivity and counting results on the global minima manifold for multi-layer networks (i.e., generalizing Theorem~\ref{thm:geometry-of-exp-manifold} and Corollary~\ref{cor:exp-manifold}).
Finally, we compare the number of affine subspaces of the global minima and symmetry-induced critical points.
An ANN with $ L $ layers $ f^{(L)} : \R^{\din} \to \R^{\dout} $ with widths $ \br = (r_1, r_2, \ldots, r_{L-1}) $ is
\begin{align}
      \fL = W^{(L)} \sigma(W^{(L-1)} \cdots \sigma(W^{(1)} x) ))
\end{align}
where $ W^{(\ell)} \in \R^{ r_{\ell} \times r_{\ell-1} } $ for $ \ell = 1, \ldots, L $ with $ r_0 = \din $ and $ r_L = \dout $, the non-linearity $ \sigma $ is applied element-wise, and $ \ptheta = (W^{(L)}, \ldots, W^{(1)}) \in \R^{ d(\br)} $ is the vector of parameters of dimension $ d(\br) = \sum_{\ell=1}^{L} r_{\ell-1} r_\ell $.
Observing that any pair of weight matrices $(W^{(\ell)},W^{(\ell+1)})$ for $\ell=1,\dots,L-1$ forms a two-layer network
within the multi-layer network, we say that a multi-layer network is irreducible if all 
pairs $ (W^{(\ell)}, W^{(\ell+1)}) $ are irreducible.
\newline
\vspace{-0.5cm}

\textbf{The global minima manifold.} We define the expansion manifold of an irreducible network with widths $\br$ into larger widths $\bm$ by taking the sequential expansion manifolds of all pairs $(W^{(\ell)},W^{(\ell+1)})$. More precisely, we define the multi-layer expansion manifold as follows
 \begin{align}\label{eq:exp-man-multi}
      \Theta_{\br \to \bm}(\ptheta^{\br}) := \{ \pphi_{1} \in \R^{d(\bm)}: \pphi_{L-1} \in \Theta^{(L-1)}_{\br \to \bm} (\ptheta^{\br}), \notag \\
      \ldots,\pphi_{1} \in \Theta^{(1)}_{\br \to \bm} (\pphi_{2}) \}
  \end{align}
where $ \Theta^{(\ell)}_{\br \to \bm}(\pphi) $ substitutes the pair $(W^{(\ell)},W^{(\ell+1)})$ with those of a point in the usual expansion manifold (Def.~\ref{def:exp-man}).
Since each expansion leaves the output of the network unchanged, all points in this expansion have the same loss. Note that the order in which we take these expansions affects the final manifold; expanding from the last layer to the first one gives the largest final manifold. The same final manifold can be obtained via a `forward pass' if one considers  expansion up to an equivalence of the incoming weight vectors.

Assume that a minimal $ L $-layer network achieves a unique (up to permutation) global minimum point $ \ptheta_* $ with widths $ \br^* = (r^*_1, r^*_2, \ldots, r^*_{L-1}) $. In an overparameterized network of widths $\bm = (m_1, \ldots, m_{L-1})$ with $m_\ell > r^*_\ell $ for all $ \ell \in [L-1] $ (i.e. at least one extra neuron at every hidden layer), we find a connected manifold of global minimum, which is simply the multi-layer expansion manifold $ \Theta_{\br^*\to \bm}(\ptheta_*) $ of the minimum point $\ptheta_*$. The zero-loss expansion manifold $ \expmanmultig $ consists of the following number of distinct affine subspaces
\vspace{-0.3cm}
\begin{align*}
    \prod_{\ell=1}^{L-1} T(r^*_\ell , m_\ell ).
\end{align*}
\textbf{Symmetry-induced critical points.} Similarly, we can consider the symmetry-induced critical points for multi-layer networks by applying sequential expansions $\overline{\Theta}^{(\ell)}_{\br \to \bm}$ to all hidden layers. We note that applying this expansion to a pair $(W_*^{(\ell)},W_*^{(\ell+1)})$ of a critical point $ \ptheta_*^{\br} $ generates a manifold of critical points as in the two-layer case, hence these expansions preserve criticality. The number of affine subspaces in the set of symmetry-induced critical points is
\begin{align*}
    \prod_{\ell=1}^{L-1} G(r_\ell , m_\ell).
\end{align*}

{\bf Application}. Similar to Fig~\ref{fig:converge-or-not}-(d,e), we consider the case where a minimal $ L $-layer network with $ r^* $ neurons at each hidden layer reaches a global minimum point $ \ptheta_* $. Let us consider an overparameterization with $ m = r^* + h $ neurons at each hidden layer. The ratio of the number of critical subspaces of $k$-th level saddles to the global minima subspaces is
\begin{align*}
    \mathrm{R}_k(r^*, m)^{L-1} = \left(\frac{G(r^*-k, m)}{T(r^*, m)}\right)^{L-1},
\end{align*}
which is exponential in depth. Therefore in the mildly overparameterized regime, i.e. when $ h $ is small, we see that the ratio of the number of saddles to that of global minima  grows exponentially with depth. In other words, we observe that the dominance of the number of saddles is even more pronounced in the multi-layer case. For the vastly overparameterized regime, i.e. when $ h $ is large, we observe the opposite effect: the dominance of the number of global minima is stronger in the multi-layer case. Finally, we observe a width-depth trade-off in reaching a dominance of the global minima: one can either increase the width of a two-layer network so that the ratios $ \mathrm{R}_k(r^*, m) $ go down to $ 0 $; or increase the depth in a network where each layer is just large enough to guarantee  $ \mathrm{R}_k(r^*, m) < 1 $  which eventually decreases the total ratio down to $ 0 $.

\section{Conclusion \& Discussion}
In this paper, we explicitly characterize the geometry formed by the critical points in overparameterized neural networks.
For the global minima, we showed that under mild conditions they live in a manifold consisting of a number of connected affine subspaces.
We characterize a certain type of critical points, the so-called symmetry-induced critical points and we showed that they form an explicit number of affine subspaces.
From the theoretical point of view, it remains an open question whether there are other critical points in the overparameterized networks in addition to the symmetry-induced ones.
We also leave it to future work to study whether all symmetry-induced critical points are strict saddles or not.

Our main result quantifies the scaling of the numbers of global minima subspaces and the subspaces containing symmetry-induced critical points as the width grows.
In mildly overparameterized networks, the number of critical subspaces is much greater than that of the global minima subspaces, so that in practice, the gradient trajectories may get influenced by these saddles or even get transiently stuck in their neighborhood for a fraction of typical initializations. However, in vastly overparameterized networks, the number of global minima subspaces dominates that of the critical subspaces so that symmetry-induced saddles play only a marginal role.
From a practical point of view, our theoretical results pave the way to applications in optimization of non-convex neural networks loss landscapes via a combination of overparameterization and pruning.

\section*{Acknowledgements}

The authors thank the authors of \citet{lengyel2020genni} for a discussion about neural network invariances at the very beginning of this project.
The authors thank Valentin Schmutz for a discussion, Bernd Illing and Levent Sagun for their detailed feedback on the manuscript.
This work is partly supported by Swiss National Science Foundation (no. $200020\_184615$) and ERC SG CONSTAMIS. C. Hongler acknowledges support from the Blavatnik Family Foundation, the Latsis Foundation, and the NCCR Swissmap.

\bibliography{references}
\bibliographystyle{icml2021}

\icmltitlerunning{Appendix: Geometry of the Loss Landscape in Overparameterized Neural Networks}

\onecolumn
\appendix

\icmltitle{Appendix for \\ Geometry of the Loss Landscape in Overparameterized Neural Networks: Symmetries and Invariances}

We organize the Appendix as follows:
\begin{itemize}
  \item In Section~\ref{sec:exp-details}, we discuss the experimental details presented in the main text (and in the Appendix).
   \begin{itemize}
      \item In Subsection~\ref{sec:MNIST-exp}, we present numerical experiments on two-layer ANNs with various widths trained to implement the MNIST task. We investigate whether the gradient trajectories approach a saddle or not.
      \item In Subsection~\ref{sec:number-numerics}, we present a detailed numerical analysis of the number of critical subspaces $ G $, the number of global minima subspaces $ T $, and their ratio $ G / T $.
      \item In Subsection~\ref{sec:sym-loss-exs}, we present a catalog of toy symmetric loss landscapes.
  \end{itemize}
  \item In Section~\ref{sec:proof}, we present proofs of the theorems and propositions stated in the main text.
  \begin{itemize}
      \item In Subsection~\ref{sec:sym-losses2}, we present further properties of symmetric loss landscapes. In particular, we prove Lemma 2.1.
      \item In Subsection~\ref{sec:exp-man2}, we present the expansion manifold in two-layer ANNs. In particular, we prove the Theorem 3.1 in the main.
      \item In Subsection~\ref{sec:no-new-minima2}, we present a case where there is no new global minimum outside of the expansion manifold for some smooth activation functions. In particular we prove Theorem 4.2 in the main and discuss the implications for standard activation functions such as sigmoid and tanh.
      \item In Subsection~\ref{sec:sym-induced-crit2}, we present the symmetry-induced critical points. In particular, we prove the Proposition 4.3 and Proposition 4.4 in the main.
      \item In Subsection~\ref{sec:comb-analysis2}, we present the combinatorial analysis which is used to derive the closed-form formulas and the limiting behavior of the numbers $ T $ and $ G $. In particular, we prove the Proposition 4.5, Lemma 4.6, and we present the calculations in the Subsection 4.3 in the main.
      \item In Subsection~\ref{sec:multi-layer2}, we present some generalizations of the two-layer ANN results for the multi-layer ANNs.
  \end{itemize}
\end{itemize}

\section{Further Experimental Results} \label{sec:exp-details}

The code is available at \url{https://github.com/jbrea/SymmetrySaddles.jl}.
We first present an extension of the Figure 1 in the main below.

\begin{figure}[h!]
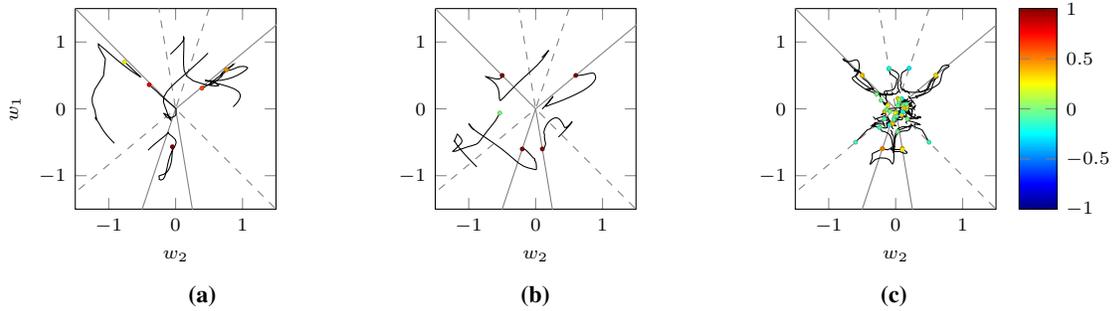

    \tikzset{font = \scriptsize, mark size = .7}
\pgfplotsset{width = 4.25cm, height = 4.25cm}
\centering
\begin{tabular}{llll}
    \input{figures/example1.tikz} &
    \hspace{1cm}\input{figures/example2.tikz} &
    \hspace{1cm}\input{figures/example3.tikz} \\
    \hspace{2.4cm} {\small\bf{(a)}} & \hspace{2.65cm} {\small\bf{(b)}} & \hspace{2.65cm} {\small\bf{(c)}}
\end{tabular}

    \vspace{-0.35cm}
    \caption{\textit{A student-teacher regression setting with 2D input and a two-layer teacher network with $r^*=4$ sigmoid neurons with incoming weight vectors shown as solid black lines and output weights set to $1$.} black lines: trajectories of the incoming weight vectors with dots marking their position at  convergence, color: output weights at convergence.
    For mild overparameterization, the algorithm may get stuck at a local minimum (a), or may find a global minimum (b); whereas for vast overparameterization it always converges to a global minimum (c).
    \textbf{(a\&b)} Mildly overparameterized networks with width $5$ do not reliably find the global minimum. \textbf{(c)} Vastly overparameterized networks with width $45$ find a global minimum by setting some of the output weights to zero or matching
    the incoming weight vectors with that of the teacher's up to a $\pm$ factor. \label{fig:converge-or-not-appendix}
    } 
\end{figure}

\textbf{Experimental details for the Figure 1 and 5 in the main.}
The input of the training data consists of 1681 two-dimensional points on a
regular grid
$\{(x_1, x_2)| 4x_1 = -20,\ldots, 20, 4x_2 = -20, \ldots, 20\}$  and
target values $y = \sum_{i=1}^4 \sigma(\sum_{j=1}^2w_{ij}x_j)$ with $w_{11} =
0.6, w_{12} = 0.5, w_{21} = -0.5, w_{22} = 0.5, w_{31} = -0.2, w_{32} = -0.6, w_{41} = 0.1, w_{42} = -0.6$.
Student networks were initialised with the Glorot uniform initialisation \cite{Glorot10}, trained with Adam \cite{Kingma14}, and gradients always computed on the full dataset, until reaching a loss below
$10^{-7}$. To reach efficiently the local minimum closest to the point found
with Adam, we continued optimizing the parameters of the student networks with
the sequential quadratic programming algorithm SLSQP of the NLopt package
\cite{Johnson} for a maximal duration of 1000 seconds. The final loss values of
all students that converged to a good solution was below $10^{-15}$ for every
random seed considered.
To obtain a non-trivial teacher network with 3 hidden layers (Fig.~1d-e), we fitted a network with widths 4-4-4 to the function $f(x_1, x_2) = \sin(2x_1) + x_1 + \cos(3x_2) - 0.4(x_2-1)^2$ evaluated on the same two-dimensional grid as above.
The teacher network does not reach zero loss on this data set.
To obtain target values for the student networks we evaluated this teacher network on the two-dimensional grid; hence there exist zero loss configurations for the student networks.

\vspace{-0.75cm}
\subsection{MNIST Experiments for Two-Layer ANNs with Various Widths}\label{sec:MNIST-exp}

\textbf{Experimental details for the Figure~\ref{fig:MNIST-exps} in the appendix.}
The training set consisted of the standard MNIST test set, i.e. 10'000 grayscale
images of 28x28 pixels with corresponding labels. The networks had a single
hidden layer of width $N$ with the softplus non-linearity $g(x) = log(exp(x) + 1)$.
The networks were initialised with the Glorot uniform initialisation
\cite{Glorot10} and trained on the cross-entropy loss with Adam and gradients
always computed on the full dataset. We measured the squared norm of the
gradients and the squared norm of the parameter updates.

\begin{figure}[h!]
    \centering
     \subfloat{
        \includegraphics[width=0.3\textwidth]{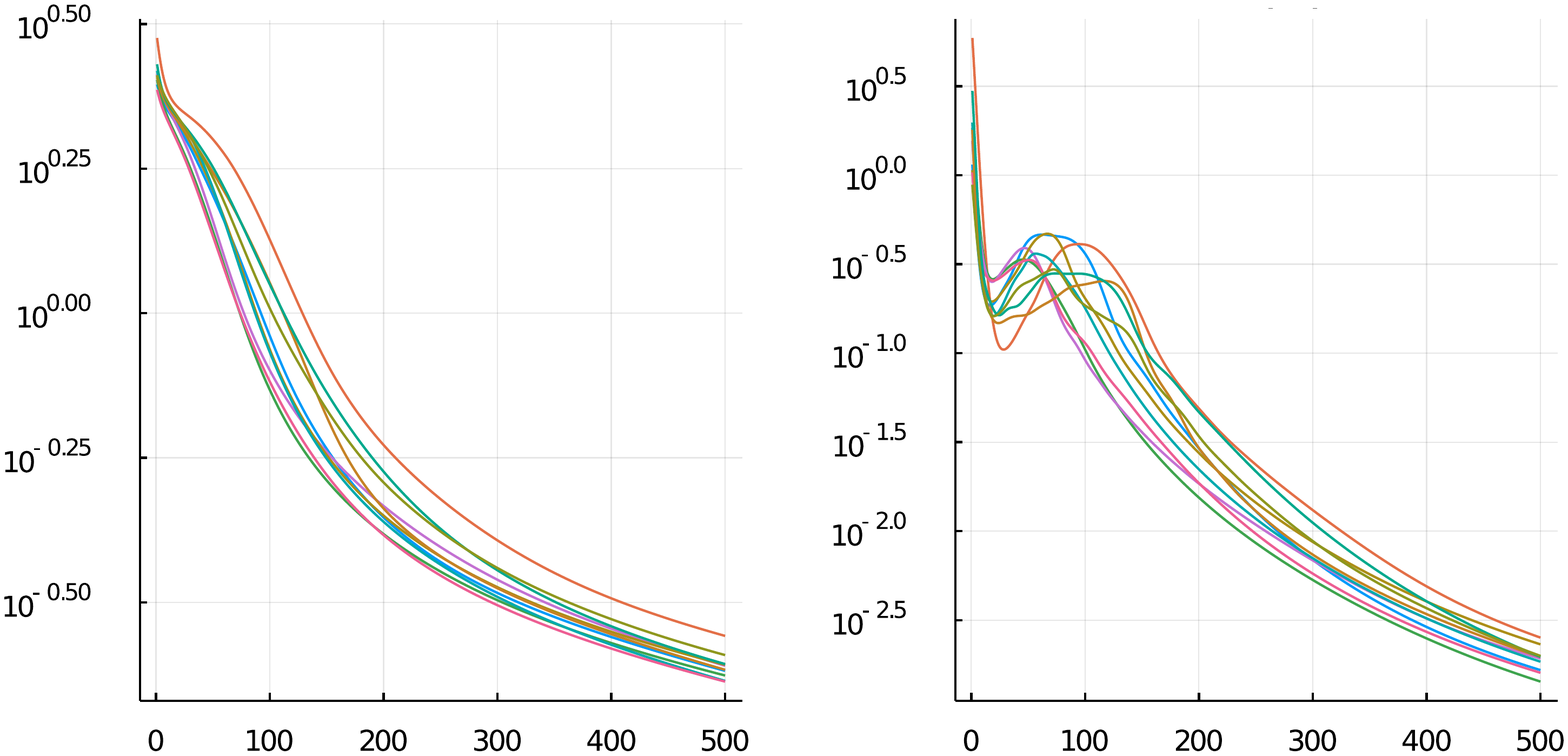}
        } \hspace{1cm}
    \subfloat{
        \includegraphics[width=0.3\textwidth]{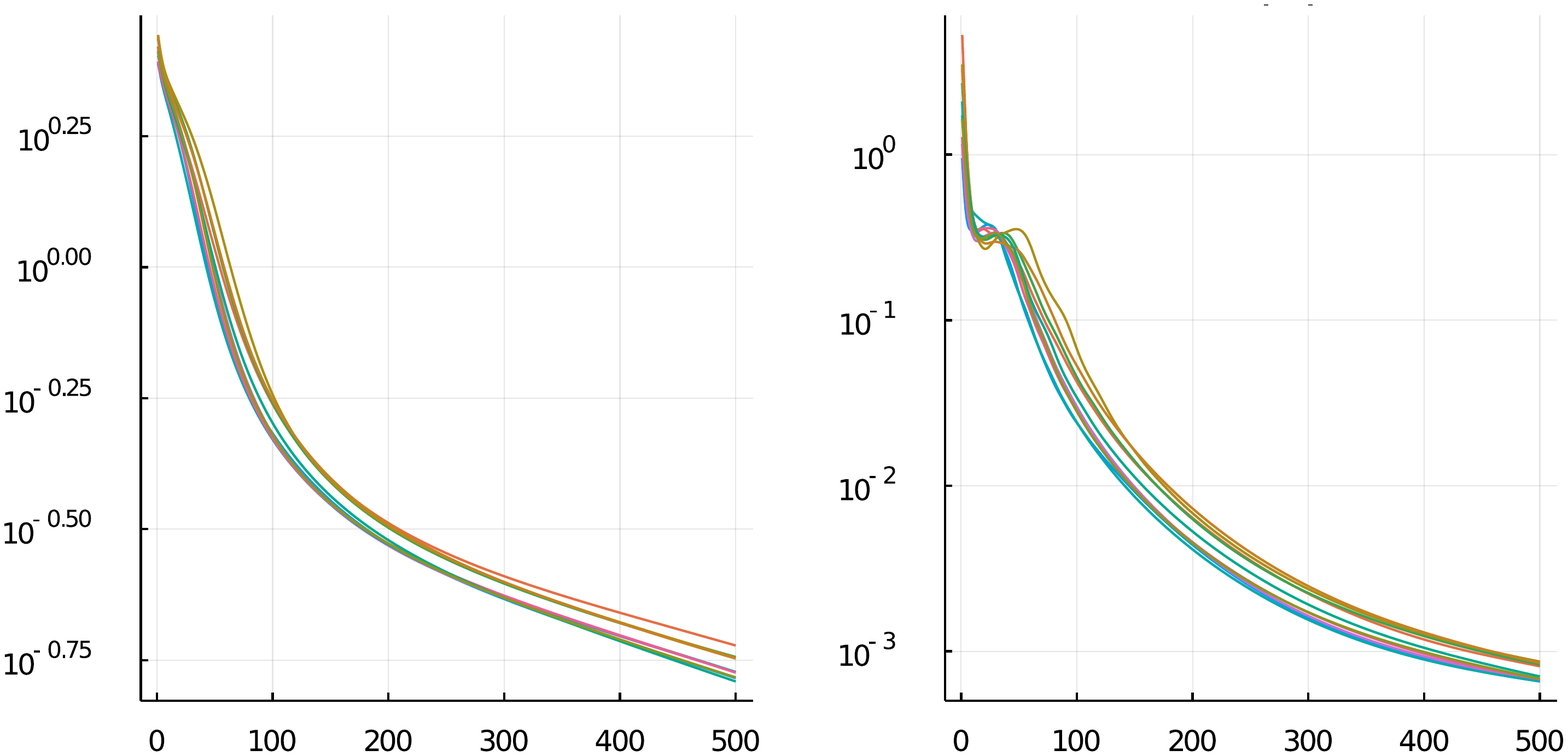}
        } \\
        {\small \hspace{3cm} \textbf{(a)} $ m = 10 $ \hspace{5.5cm}  \textbf{(b)} $ m = 20 $ \hspace{2cm} } \\
    \subfloat{
        \includegraphics[width=0.3\textwidth]{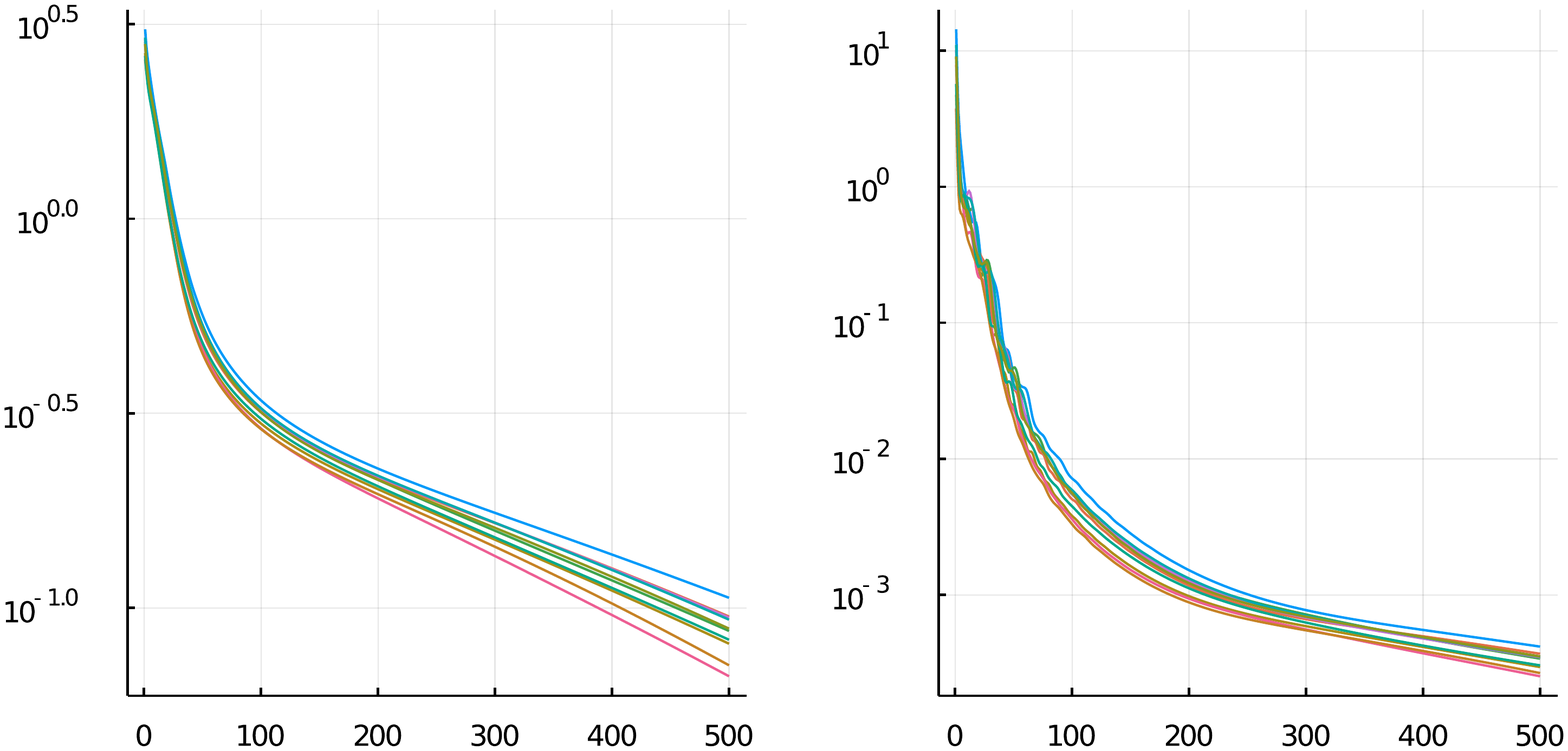}
        } \hspace{1cm}
    \subfloat{
        \includegraphics[width=0.3\textwidth]{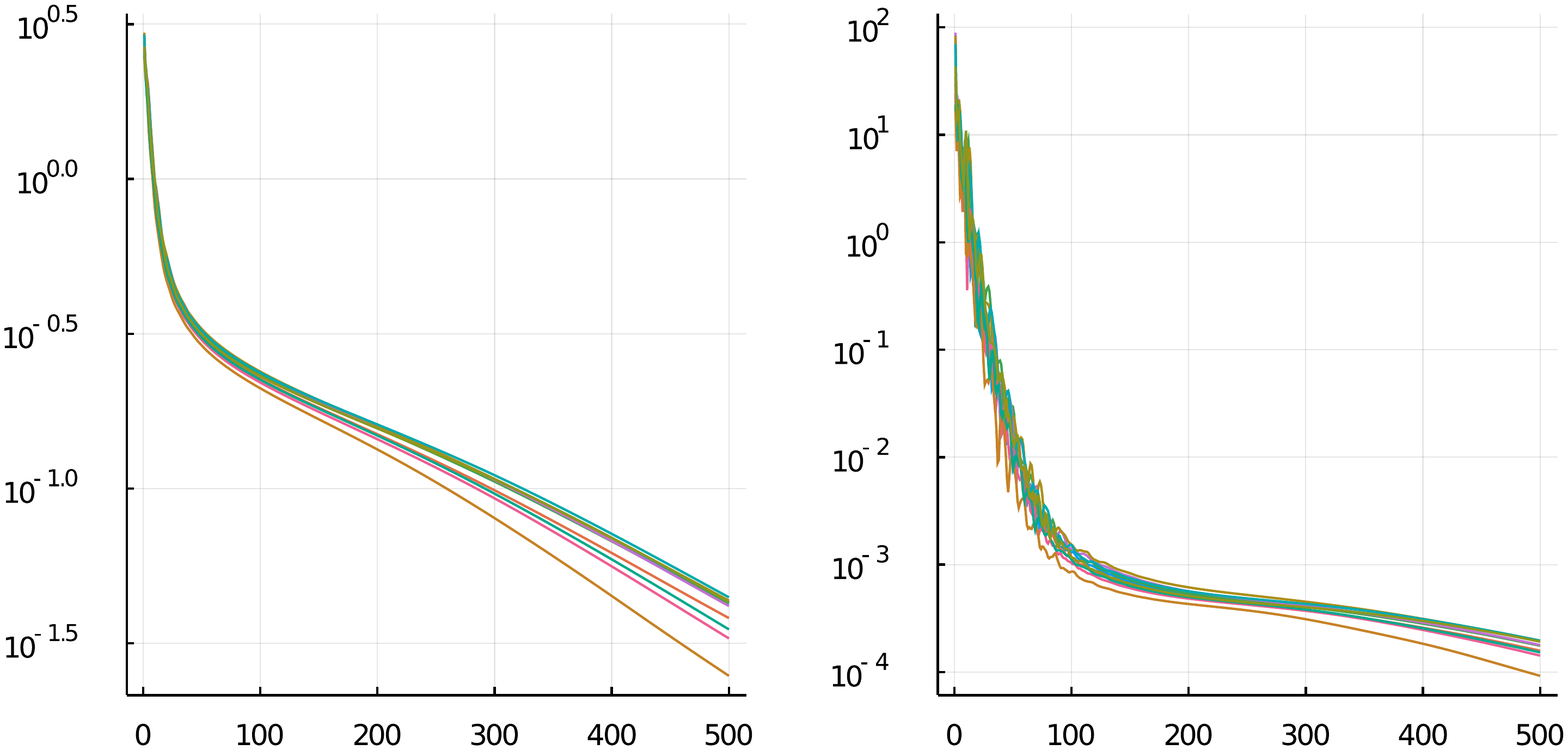}
        } \\
        {\small \hspace{3cm} \textbf{(c)} $ m = 100 $ \hspace{5.5cm}  \textbf{(d)} $ m = 1000 $ \hspace{2cm} }
    \caption{\label{fig:MNIST-exps} \textit{Network width $ m $ impacts whether gradient trajectories approach a saddle or not.} For all a-b-c-d, the loss curves are demonstrated on the left and the norm of the gradient is demonstrated on the right. We observe that the norm of the gradient decreases and then increases in narrow networks \textbf{(a-b)}, indicating an approach to a saddle and then escaping it. We do not observe a sharp non-monotonicity in the norm of the gradient for wider networks \textbf{(c-d)}. Instead we observe short decrease and increase periods in the norm of the gradient (see the zigzag) \textbf{(d)}, which indicates that the gradient trajectories move from one saddle to the next in this regime, yet without getting very close these saddles.
    }
\end{figure}

For the MNIST experiments, we observe that the gradient trajectories visit a saddle in a narrow network and the duration of the visit to the saddles becomes shorter as we increase the width (i.e. in (a), we see a longer plateau in the loss curve compare to (b)). In the excessive overparameterization regime, we observe another behavior change, i.e. we observe a zigzag behavior on the norm of the gradient, possibly indicating many short visits to the saddles.

\subsection{A Detailed Analysis of the Number of Critical Subspaces and the Number of Global Minima Subspaces}\label{sec:number-numerics}

In this section, first we present a detailed numerical analysis of the numbers $ T $ and $ G $ (see Figure~\ref{fig:numbers2}) and then we present additional figures for various minimal widths (see Figure~\ref{fig:numbers3}), expanding the Figure 6 in the main.

\begin{figure}[h!]
  \begin{center}
    \includegraphics[width=0.3\textwidth]{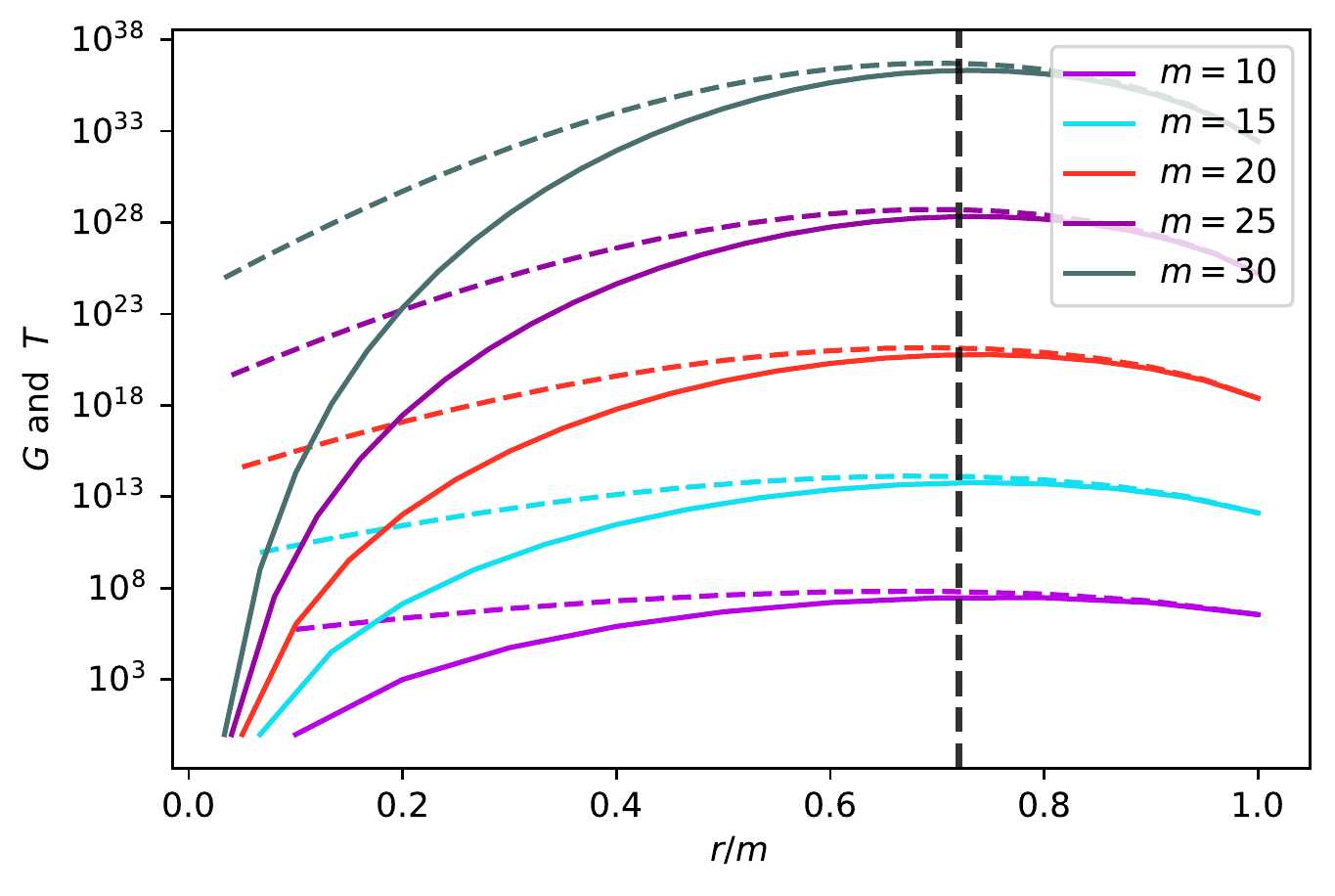} \hspace{0.8cm}
  \end{center}
  \vspace{-0.5cm}
  \caption{\label{fig:numbers2}
  \textit{Comparison of the number of critical subspaces $ G $ (solid) with the number of global minima subspaces $ T $ (dashed) for $ m = 10, 15, 20, 25, 30 $ where $ x $-axis denotes $ r / m $.}
  Each solid line indicates $ G(r, m) $, and the dashed lines indicate $ T(r, m) $ for $ r = 1, \ldots, m $. We observe that the maximum $ G(r, m) $ is achieved at $ r \approx 0.72 m $.}
\end{figure}

In Figure~\ref{fig:numbers2}, we observe an interesting linear relationship between $ r $ and $ m $, i.e. for fixed $ m $, $ G(r, m) $ is maximized for $ r \approx 0.7 m $. A refined analysis of these numbers can be useful for studying how much overparameterization is needed to converge to a global minimum efficiently.

\begin{figure}[h!]
    \centering
    \subfloat{
      \includegraphics[width=0.3\textwidth]{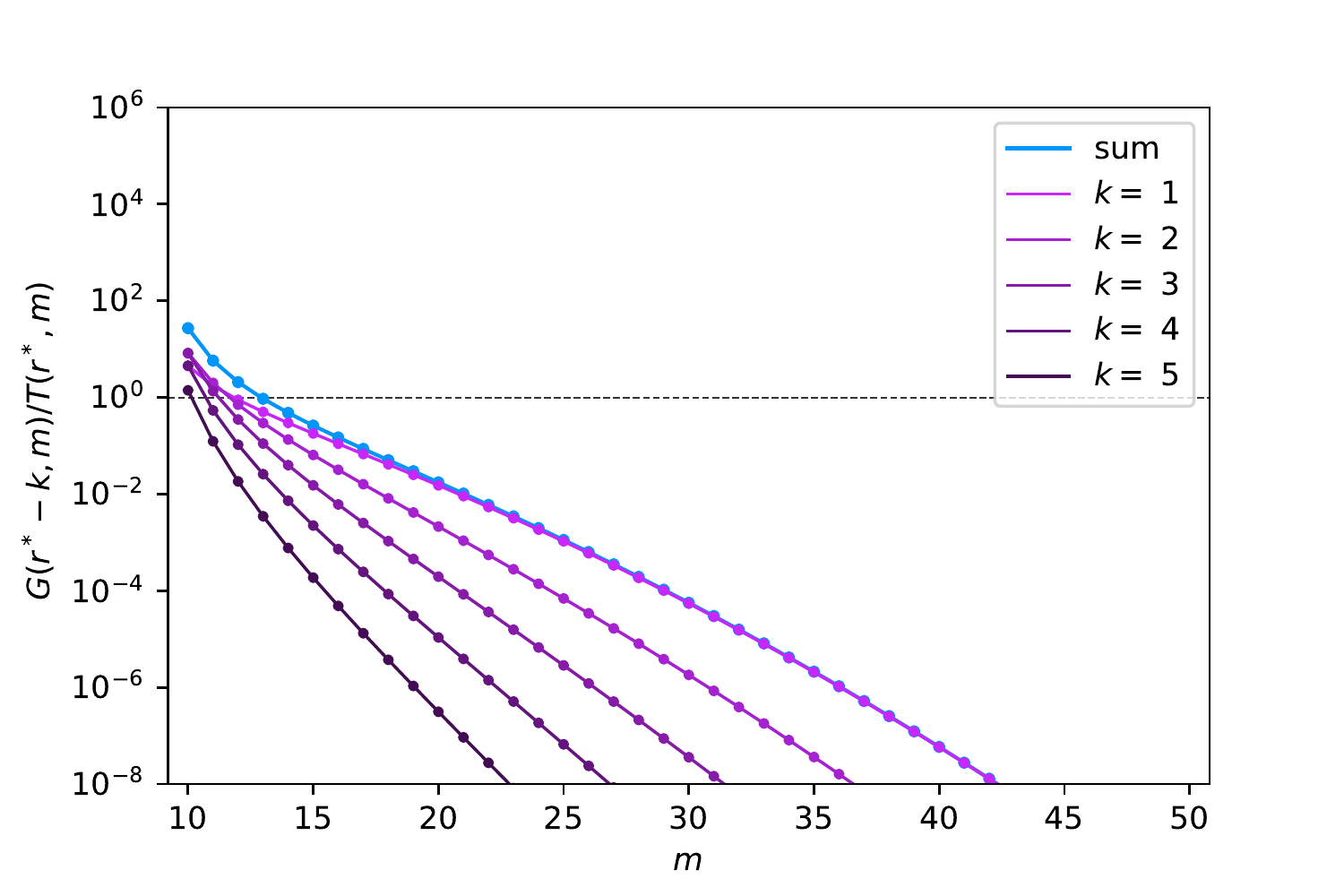}
    } \hspace{0.1cm}
    \subfloat{
       \includegraphics[width=0.3\textwidth]{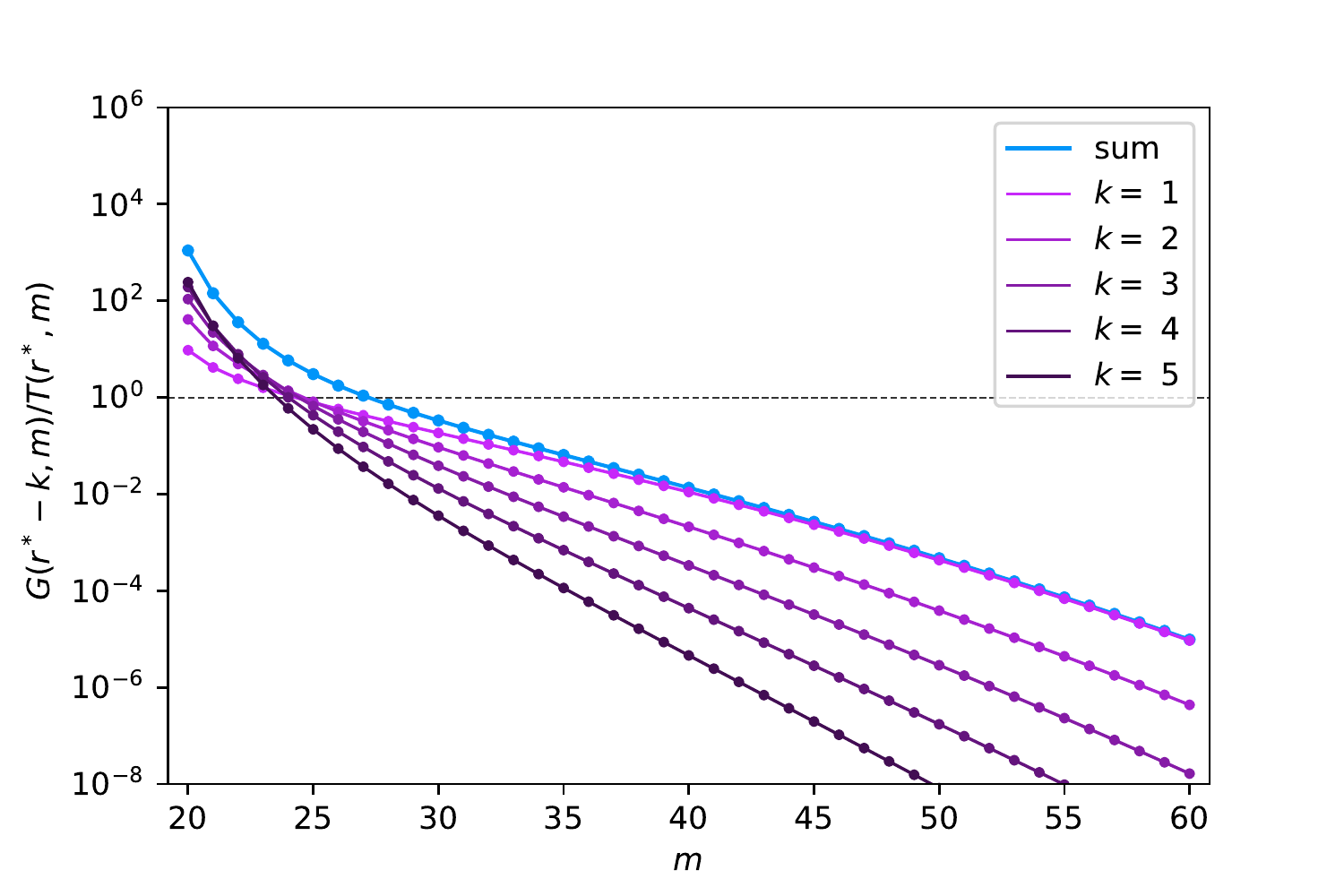}
    } \hspace{0.1cm} 
    \subfloat{
       \includegraphics[width=0.3\textwidth]{code/TG-teacher30.pdf}
    } \\
    {\small \hspace{2.1cm} \textbf{(a)} $ r^* = 10 $ \hspace{3.6cm} \textbf{(b)} $ r^* = 20 $ \hspace{3.5cm} \textbf{(c)}  $ r^* = 30 $ \hfill }
    \caption{\label{fig:numbers3}
    \textit{The ratio of the number of critical subspaces $ G(r^* - k, m) $ to global minima subspaces $ T(r^*, m) $ as the width $ m $ of the overparameterized network increases.} Plotted for various minimal widths \textbf{(a)} $ r^* = 10 $, \textbf{(b)} $ r^* = 20 $, and \textbf{(c)} $ r^* = 30 $ (as shown in the main).
    The ratio of all critical subspaces to the global minima subspaces $ \sum_{k=1}^{r^* - 1} G(r^* - k, m) / T(r^*, m) $ is shown in blue. 
    }
\end{figure}

In Figure~\ref{fig:numbers3}, we observe that the rate of decay to zero is faster is smaller minimal widths (see for example $ r^* = 10 $). This is consistent with out mathematical analysis, since $ \frac{r^* - 1}{r^*}$ increases as $ r^* $ increases, yielding a slower decay to zero (see the blue curves). We note that the exact implementation of the numbers becomes unstable for $ r^* > 35 $ in our numerical experiments. Therefore for wider minimal widths, an approximation of the numbers $ G $ and $ T $ is needed.

\subsection{Symmetric Loss Landscape Examples}\label{sec:sym-loss-exs}

We present some example symmetric losses $ \R^2 \to \R $ in Figure~\ref{fig:example-losses}, expanding Figure 2 in the main. We observe that in between two partner global minima (red points), there may be more than one saddles emerging on the symmetry subspaces.
\vspace{1cm}

\begin{figure}[h!]
    \centering
    \hspace{-1.5cm} \subfloat{
        \includegraphics[width=0.2\textwidth]{code/2dcase2-log-vtype1.pdf}
        }
    \subfloat{
        \includegraphics[width=0.2\textwidth]{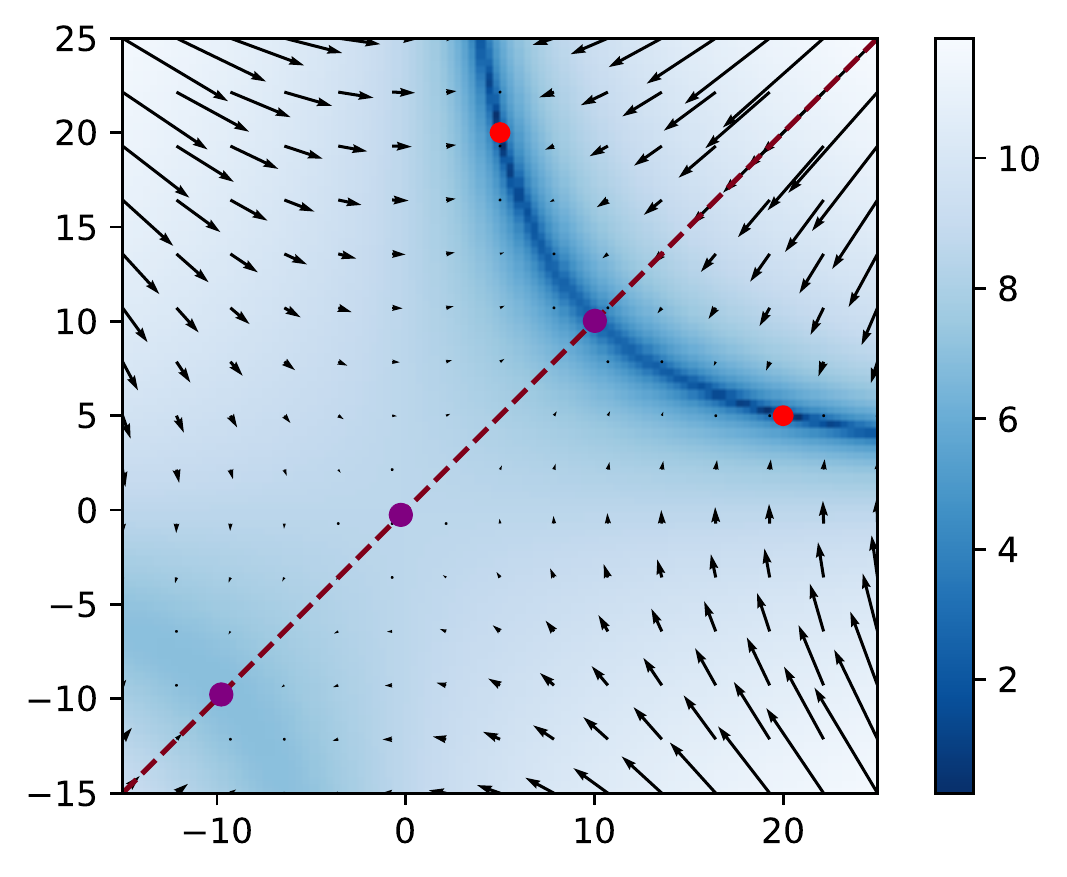}
        }
    \subfloat{
        \includegraphics[width=0.2\textwidth]{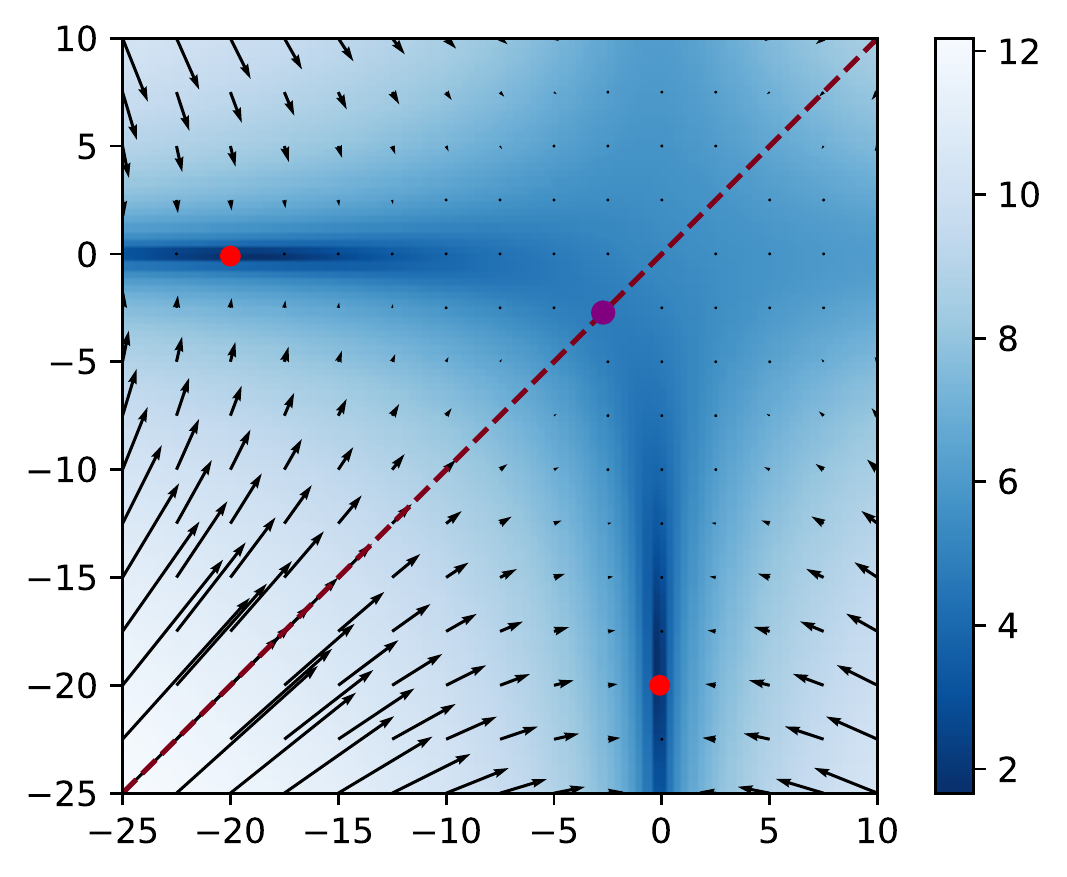}
        }
    \subfloat{
        \includegraphics[width=0.2\textwidth]{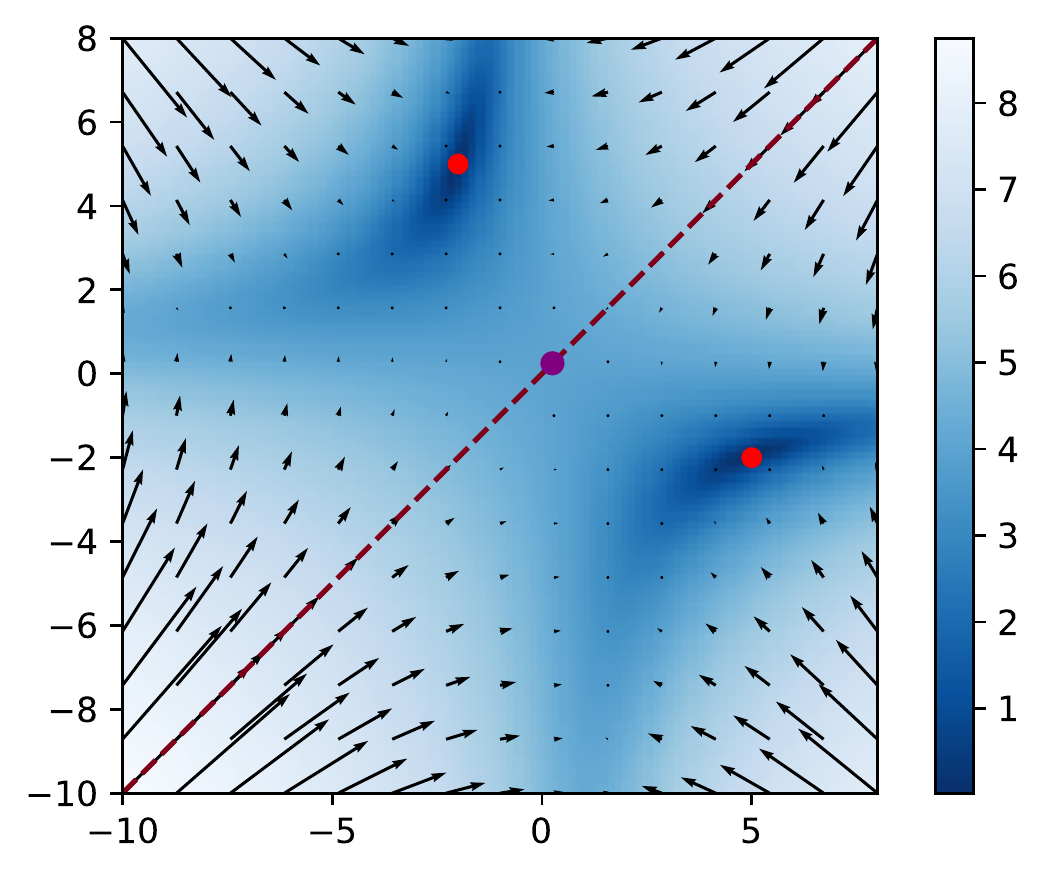}
        } \hspace{-1.5cm} \vfill
    \hspace{-1.5cm} \subfloat{
        \includegraphics[width=0.22\textwidth]{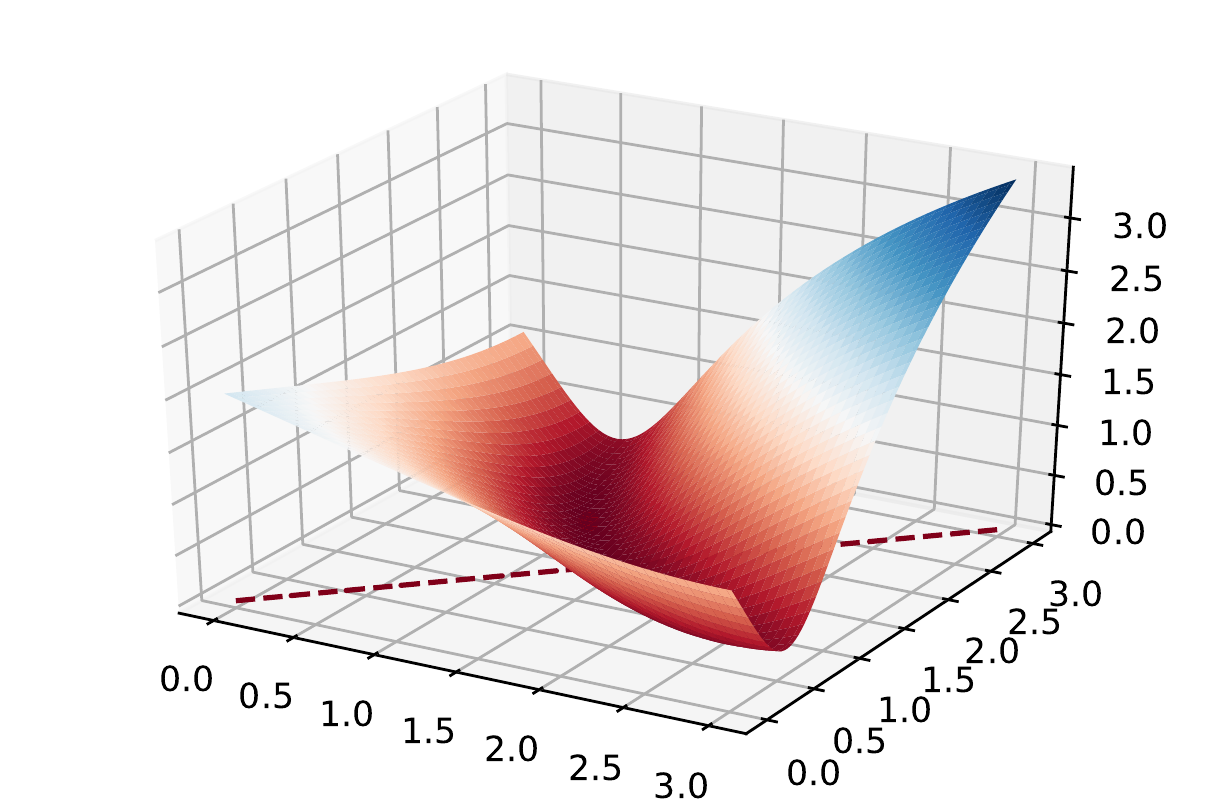}
      } \hspace{-0.5cm}
    \subfloat{
          \includegraphics[width=0.22\textwidth]{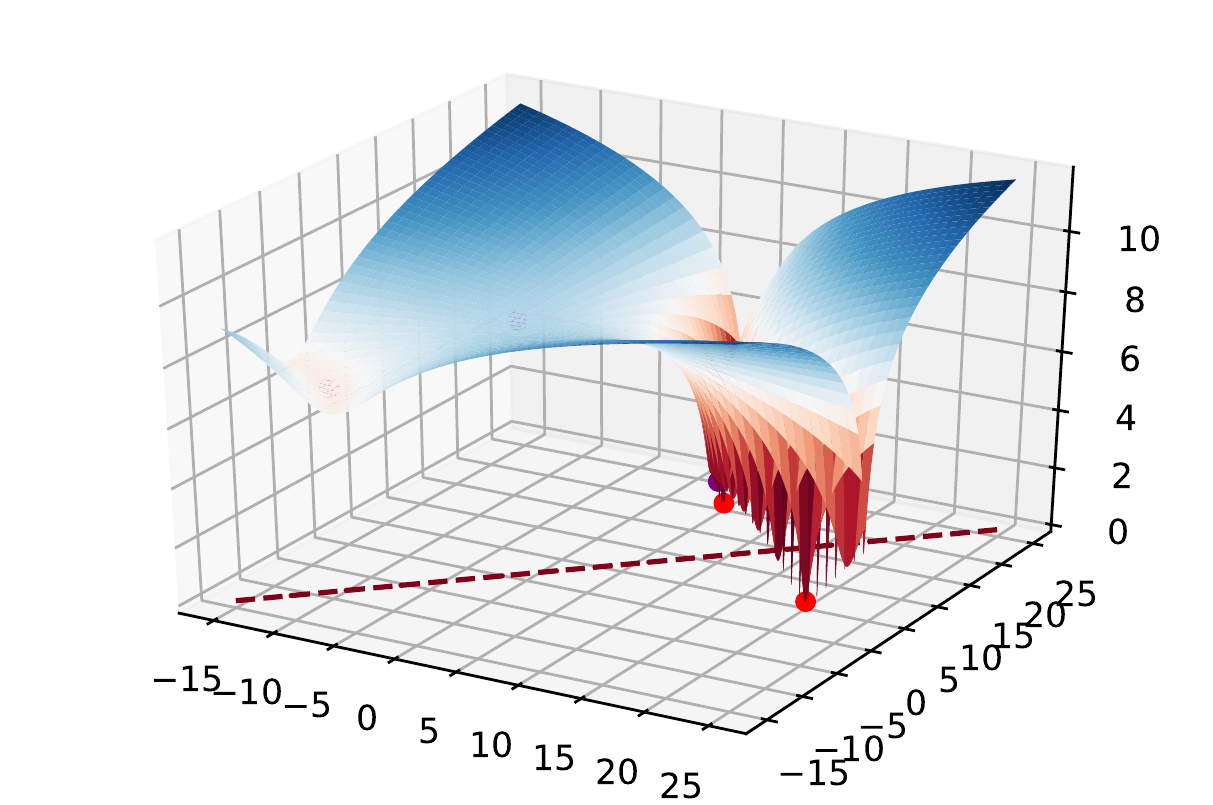}
        } \hspace{-0.5cm}
    \subfloat{
          \includegraphics[width=0.22\textwidth]{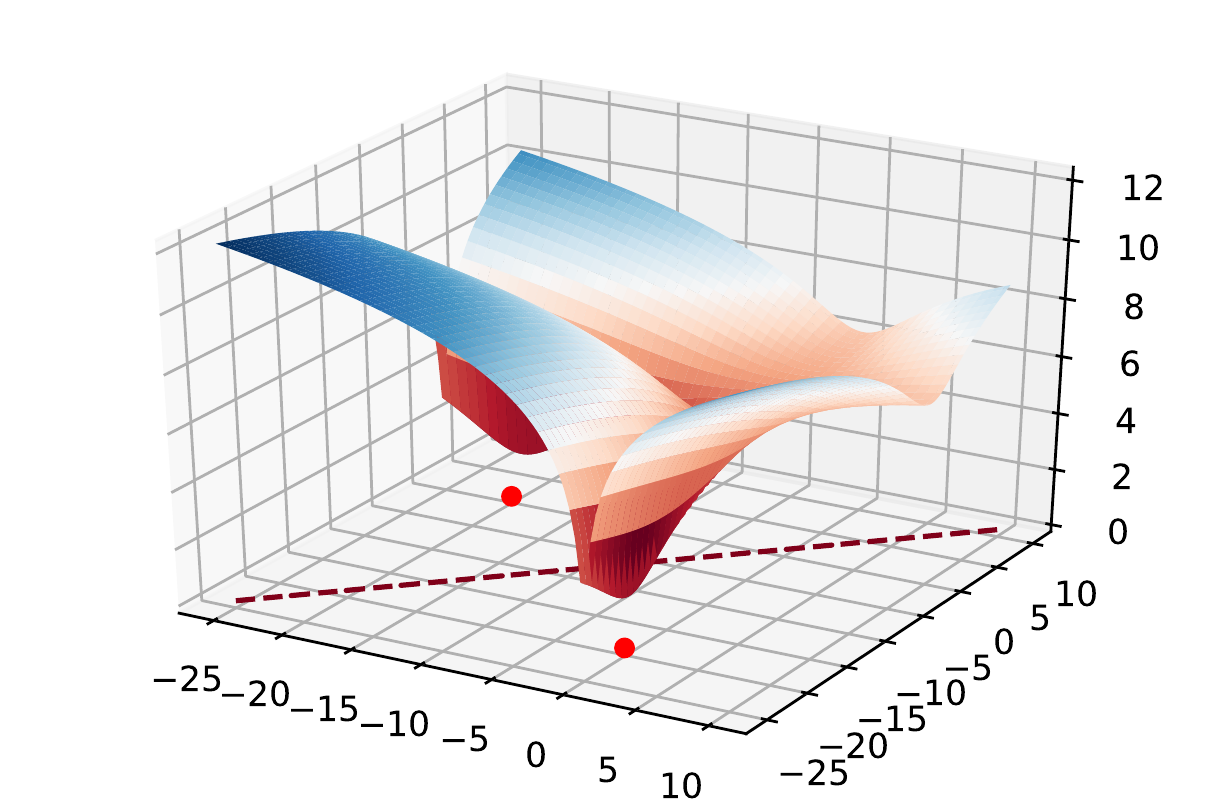}
        } \hspace{-0.5cm}
    \subfloat{
          \includegraphics[width=0.22\textwidth]{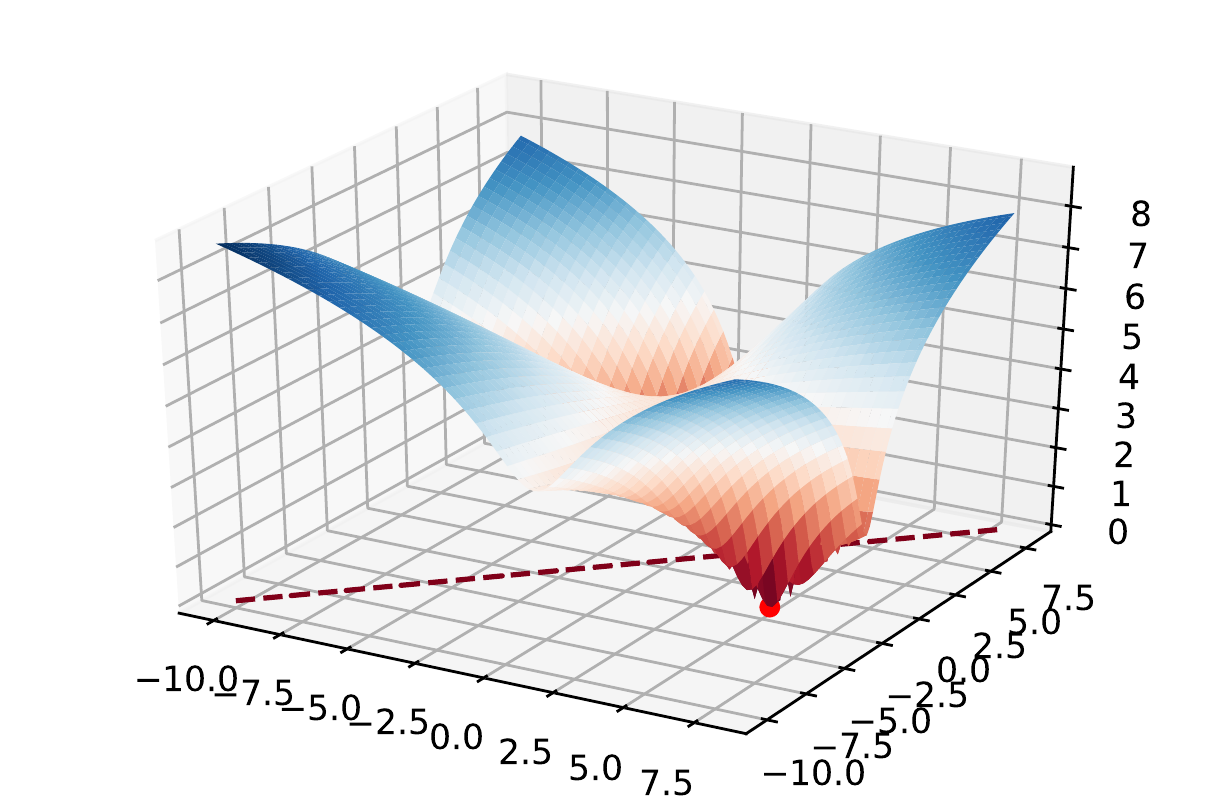}
        } \hspace{-1.5cm}
    \caption{\label{fig:example-losses} The gradient flow and the landscape and of a permutation-symmetric loss $L(w_1, w_2)= \log(\frac{1}{2} ((w_1 + w_2 - a)^2 + (w_1 w_2 - b)^2 ) + 1) $. Red dots: global minima, purple dots: non-global stationary points. Dashed lines represent the symmetry hyperplanes. \textbf{1st.} $a = 3$ and $b = 2$, the global minima at $(2, 1)$ and $(1, 2)$.
    \textbf{2nd.} $a = 25$ and $b = 100$, the global minima at $(20, 5)$ and $(5, 20)$.
     \textbf{3rd.} $a = -20.1$ and $b = 2$, the global minima at $(-20, -0.1)$ and $(-0.1, -20)$. \textbf{4th.} $a = 3$ and $b = -10$, the global minima at $(5, -2)$ and $(-2, 5)$.
    }
\end{figure}

\section{Proofs and Further Discussions}\label{sec:proof}

\subsection{Further Properties of Symmetric Losses}\label{sec:sym-losses2}

The most well known property of symmetric losses is the $ m! $ multiplicity of the critical points: for a critical point $ \ptheta^* = (\vartheta^*_1, \vartheta^*_2, \ldots, \vartheta^*_m) $ with distinct units $ \vartheta^*_i \neq \vartheta^*_j $ for all $ i \neq j $, there are $ m! $ equivalent critical points induced by permutations $ \pi \in S_m $.
Similarly, every point $ \ptheta $ with distinct units has $ m! - 1 $ partner points with equal loss.
For a symmetric loss function, a fundamental region
\begin{align*}
  \mathcal{R}_{0} := \{ (\vartheta_1, \ldots, \vartheta_m) \in \R^{ D m}
  : \vartheta_{1} \geq  \ldots \geq  \vartheta_{m} \}
\end{align*}
has $ m! - 1 $ partner regions where the landscape of the loss is the same up to permutations. Note that above and elsewhere we use the lexicographic order: for two units $ \vartheta, \vartheta' \in \R^D$, we write $ \vartheta > \vartheta' $ if there exists $j\in [D]$ such that $ \vartheta_i = \vartheta'_i \  \text{for all} \ i \in [j-1] \ \text{and} \ \vartheta_j > \vartheta'_j $; and $ \vartheta = \vartheta' $, if $ \vartheta_i = \vartheta'_i $ for all $ i \in [D] $.

\begin{definition} For a permutation $ \pi \in S_m $, a \textbf{\emph{replicant region}} $ \mathcal{R}_{\pi} $ is defined by
\begin{align}\label{eqn:def-rep-reg}
    \mathcal{R}_{\pi} := \{ (\vartheta_1, \ldots, \vartheta_m) \in \R^{ D m}
    : \vartheta_{\pi(1)} \geq  \ldots \geq  \vartheta_{\pi(m)} \}.
\end{align}
We denote by $ \mathring{\mathcal{R}}_{\pi} $  the interior of the replicant region.
\end{definition}

Any two partner points $ \ptheta_\pi \in \mathcal{R}_\pi $ and $ \ptheta_{\pi'} \in \mathcal{R}_{\pi'} $ have the same loss $ L^m( \ptheta_\pi) = L^m( \ptheta_{\pi'} ) $ and
they are linked with a permutation matrix $ \Perm_{\pi' \circ \pi^{-1}} $ : $ \Perm_{\pi' \circ \pi^{-1}} \ptheta_\pi = \ptheta_{\pi'} $.

Note that the lexicographic order is a total order thus it allows to compare any two $ D $-dimensional units. Therefore every point $ \ptheta \in \R^{Dm} $ falls in at least one replicant region, i.e.
\begin{align*}
    \R^{Dm} = \cup_{\pi \in S_m} \mathcal{R}_\pi.
\end{align*}

The intersection of all these regions $\mathcal R_\pi$ corresponds to the $ D $-dimensional linear subspace $ \vartheta_1 = \vartheta_2 = \cdots = \vartheta_m $; more generally intersections of replicant regions define symmetry subspaces.

As each constraint $ \vartheta_i = \vartheta_j $ suppresses $ D $ degrees of freedom, we have $ \dim(\H_{i_1, \ldots, i_k}) = D (m - k + 1) $.
Observe that the largest symmetry subspaces are $ \H_{i, j} $'s since any other symmetry subspace is included in one of these  $ \binom{m}{2}$ subspaces.

For $ D = 1 $, the largest symmetry subspaces have codimension $ 1 $. As a result, any path from $ \mathcal{R}_\pi $ to any another replicant region has to cross a symmetry subspace (see Figure~\ref{fig:3d-param-space}).
However, for $ D > 1 $, the symmetry subspaces have codimension at least $ D $; thus there exist paths connecting replicant regions without crossing symmetry subspaces.

\begin{wrapfigure}{r}{0.4\textwidth}
    \begin{center}
        \includegraphics[width=0.25\textwidth]{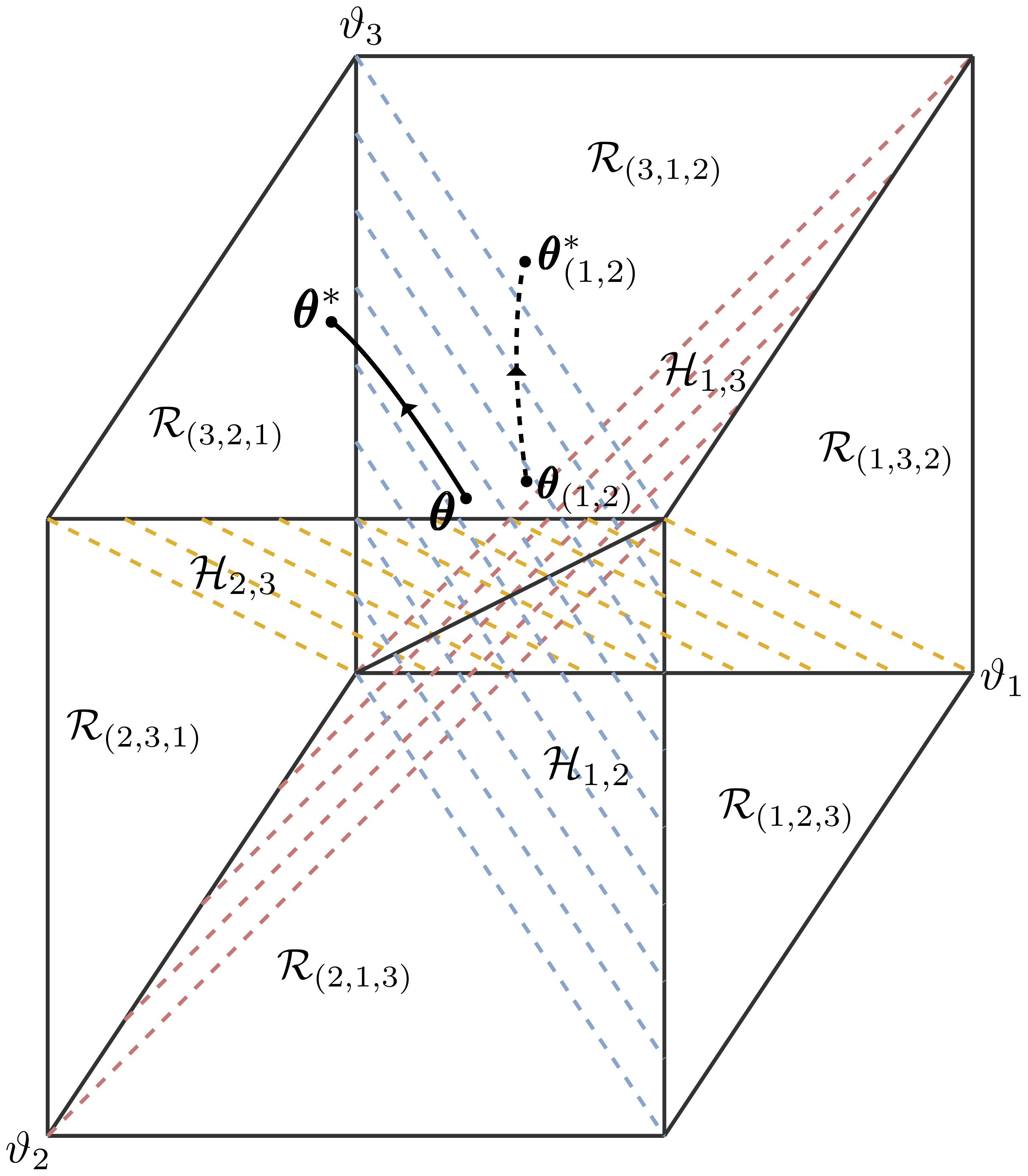}
      \end{center}
      \vspace{-0.3cm}
      \caption{{\small \textit{ Replicant regions $ \mathcal R_\pi $ and symmetry subspaces $ \mathcal H_{i,j} $ for the $ 3 $-dimensional parameter space $ \R^3 $.}} An example gradient flow trajectory starting at $ \ptheta \in \mathcal R_{(3,2,1)} $ and arriving at a minimum $ \ptheta^* $ (solid curve) and
      its partner trajectory starting at a partner point $ \ptheta_{(1,2)} \in R_{(3,1,2)} $ thus arriving at a partner minimum $ \ptheta_{(1,2)}^* $ (dashed curve) are shown. \label{fig:3d-param-space}}
\end{wrapfigure}

\begin{lemma}[Lemma 2.1 in the main]\label{invariant-manifolds-appendix} Let $ L^m : \R^{Dm} \to \R $ be a symmetric loss on $ m $ units thus a $C^1$ function and let $ \prho: \Rpz \to \R^{Dm} $ be its gradient flow.
If $ \prho (0) \in \H_{i_1, \ldots, i_k } $, the gradient flow stays inside the symmetry subspace, i.e. $ \prho (t) \in \H_{i_1, \ldots, i_k } $ for all $ t > 0 $.
If $ \prho (0) \notin \mathcal{H}_{i,j} $ for all $i \neq j \in [m]$, that is outside of all symmetry subspaces, the gradient flow does not visit any symmetry subspace in finite time.
\end{lemma}

\begin{proof} We will use the identity that comes from chain rule $ \nabla L^m ( \Perm_\pi \ptheta) = \Perm_\pi \nabla L^m ( \ptheta) $. We will show that if $ \ptheta=(\vartheta_1,\cdots, \vartheta_m)\in\mathcal{H}_{i_1,\ldots,i_k}$ where $\vartheta_{i_1}=\cdots=\vartheta_{i_k}$, its gradient satisfies $ \nabla_{i_1} L^m (\ptheta) = \ldots =  \nabla_{i_k} L^m (\ptheta) $ therefore the gradient flow remains on the symmetry subspace for all times.

We denote a transposition by $ (i, j) \in S_m $, which is a permutation that only swaps the units $ i $ and $ j $. Assume $ \ptheta \in \H_{i,j} $, that is
$\ptheta = \Perm_{(i,j)} \ptheta$, and thus
\begin{align*}
        \nabla L^m ( \ptheta ) = \nabla L^m ( \Perm_{(i,j)} \ptheta) = \Perm_{(i,j)} \nabla L^m ( \ptheta),
\end{align*}
and in particular $ \nabla_i L^m ( \ptheta ) = \nabla_j L^m ( \ptheta ) $.
This entails that for $ \ptheta \in \H_{i_1, \ldots, i_k} $, we have $ \nabla L^m ( \ptheta )  \in \H_{i_1, \ldots, i_k} $ as well, which completes the first part of the proof.

We now prove the second part of the claim by contradiction.
Suppose now that $\pgamma(0)\notin\mathcal{H}_{i,j}$ for any $i\neq j\in[k]$ and $t_0<\infty$ be the first time such that $\pgamma(t_0)\in\mathcal{H}_{i',j'}$ for some $i'\neq j'\in[k]$.
Let $\widetilde{\pgamma}(t)=\Perm_{(i',j')}\pgamma(t)$, that is the symmetric path with respect to $\mathcal{H}_{i',j'}$.
Then one sees that $\pgamma$ and $\widetilde{\pgamma}$ intersect for the first time at $t_0$ on $\mathcal{H}_{i',j'}$ and then $\pgamma(t)=\widetilde{\pgamma}(t)\in\mathcal{H}_{i',j'}$ for all $t>t_0$, as we showed in the first part of the proof.
Since $\nabla L^m$ is continuous, Picard-Lindelöf Theorem applies on a neighbourhood of $\pgamma(t_0)$, which ensures the unicity of the gradient flow on $[t_0-\epsilon,t_0]$ for some $\epsilon>0$. Thus, $\pgamma(t_0-\epsilon)=\widetilde{\pgamma}(t_0-\epsilon)$, which contradicts the fact that $t_0$ is the first time when $\pgamma$ intersects $\widetilde{\pgamma}$.

\end{proof}

We write the gradient of $ L^m $ in the block form
\begin{align*}
  \nabla L^m(\ptheta) = ( \nabla_1 L^m(\ptheta), \ldots, \nabla_m L^m(\ptheta) )
\end{align*}
where  for all $ j \in [m] $, $$ \nabla_j L^m ( \ptheta ) = (\partial_{D(j-1) + 1} L^m( \ptheta ), \ldots, \partial_{D(j -1) + D} L^m( \ptheta )) $$ is a $ D $-dimensional vector.

\begin{remark}\label{lem:flow-stays-in-replicant2} Let $ \prho(0) \in \mathcal{R}_\pi $ for some $ \pi \in S_m $. In the case of $ 1 $-dimensional units, $ D = 1 $, we have $ \prho(t) \in \mathcal{R}_\pi $ for all $ t \in \Rp $. Hence, in this case, the gradient flow can only be affected by the critical points of a single replicant region.
\end{remark}

\begin{proof}
Indeed, assume that $ \prho(0) = (\vartheta_1(0), \ldots, \vartheta_m(0)) \in \mathcal{R}_\pi $, i.e. $ \vartheta_{\pi_1}(0) \geq \cdots \geq  \vartheta_{\pi_m}(0) $ and
 $ {\prho(1) = (\vartheta_1(1), \ldots, \vartheta_m(1)) \in \mathcal{R}_{\pi'}} $ for another permutation $\pi'$, i.e. $ \vartheta_{\pi'_1}(0) \geq \cdots \geq  \vartheta_{\pi'_m}(0) $.
Since $ \pi \neq \pi' $, there exists a pair $ (i, j) $ such that $ \vartheta_i(0) \geq \vartheta_j(0) $ and $ \vartheta_j(1) \geq  \vartheta_i(1) $. Thus we have
\begin{align*}
     (\vartheta_i - \vartheta_j)(0) \geq 0 \geq
    (\vartheta_i - \vartheta_j)(1).
\end{align*}
Because the gradient flow $ \prho $ is continuous (since $ L^m $ is $ C^1 $) there exists a time $ t_0 $ such that $ (\vartheta_i - \vartheta_j)(t_0) = 0
$, i.e. $ \prho(t_0) \in \mathcal{H}_{i, j} $, which yields a contradiction.
\end{proof}

\begin{remark}\label{lem:flow-stays-in-replicant} In the case of $ 1 $-dimensional units, $ D = 1 $, if $\prho(0)\in\mathcal{R}_\pi$ for some $ \pi \in S_m $, we have $ \prho(t) \in \mathcal{R}_\pi $ for all $ t \in \Rp $. Hence, in this case, the gradient flow $ \prho $ can only be affected by the critical points of a single replicant region.
\end{remark}

\subsection{The Expansion Manifold in Two-Layer ANNs}\label{sec:exp-man2}

\begin{theorem}[Theorem 3.1 in the main]\label{thm:geometry-of-exp-manifold2} For $ m \geq r $, the expansion manifold $ \expman $ of an irreducible 
point $\ptheta^r $ consists of exactly\footnote{$ \binom{n_1 + ... + n_r}{n_1, ..., n_r} $ denotes the coefficient $\frac{ (n_1 + ... + n_r)! }{n_1! ... n_r!}$.}
 \begin{align*}
 T(r, m) := \sum_{j = 0}^{m-r}  \sum_{\substack{\summ(s) = m \\ k_i \geq 1, b_i \geq 1}} \binom{m}{k_1, ..., k_r, b_1, ..., b_j} \frac{1}{c_1! ... c_{m-r}!}
 \end{align*}
distinct affine subspaces (none is including another one) of dimension at least $\min(\din,\dout)(m-r)$, where $ c_i $ is the number of occurences of $ i $ among $ (b_1, ..., b_j) $.

For $ m > r $, $ \expman $ is connected: any pair of distinct points $ \ptheta, \ptheta' \in \expman $ is connected via a union of line segments $ \pgamma : [0,1] \to \expman $ such that $ \pgamma(0) = \ptheta $ and $ \pgamma(1) = \ptheta' $.
\end{theorem}

\begin{proof}
The proof of this theorem is divided in two parts. In Proposition~\ref{prop:invariances-all-in-expman}, we count the number of affine subspaces in $ \expman $ and in Theorem~\ref{thm:connectivity-man}, we prove the connectivity of the $ r \to m $ expansion manifold for $ m > r $.
\end{proof}

\begin{proposition}\label{prop:invariances-all-in-expman}
For $ m \geq r $, $ \expman $ has exactly
 \begin{align*}
 T(r, m) := \sum_{j = 0}^{m-r}  \sum_{\substack{\summ(s) = m \\ k_i \geq 1, b_i \geq 1}} \binom{m}{k_1, \ldots, k_r, b_1, \ldots, b_j} \frac{1}{c_b}
 \end{align*}
distinct affine subspaces (none is including another one) of dimension at least $ \min(\din, \dout) (m - r)  $.
Here
$ c_b := c_1! c_2! \cdots c_{m-r}! $ is a normalization factor where $ c_i $ is the number of occurence of $ i $ among $ (b_1, \ldots, b_j) $.
\end{proposition}

\begin{proof}
The dimension of the subspace $ \subs $ is
$$ \sum_{t=1}^r (k_t - 1) \dout + \sum_{t=1}^j (b_t - 1) \dout +  j \din = (m - r - j)\, \dout + j\, \din \geq \min(\din, \dout) (m - r). $$

It is enough to count the distinct configurations of the incoming weight vectors
\begin{align*}
    (\underbrace{w_1, \ldots, w_1}_{k_1}, \ldots, \underbrace{w_r, \ldots, w_r}_{k_r}, \underbrace{w_1', \ldots, w_1'}_{b_1}, \ldots, \underbrace{w_j', \ldots, w_j'}_{b_j})
\end{align*}
since the outgoing weight vectors configuration follows that of the incoming ones. For this particular tuple, the number of configurations is $ \frac{m!}{k_1! \cdots k_r! b_1! \cdots b_j!} \frac{1}{c_b} $ where the normalization factor $ c_b = c_1! \cdots c_{m-r}! $ comes from the following subconfigurations: if $b_1=b_2$, then we need to divide by $2$ since in that case one can swap $w_1'$ with $w_2'$.
More generally, if $b_{\ell_1}=b_{\ell_2}=\ldots=b_{\ell_{c_i}}=i$, we need to divide by the number of permutations between the groups of $i$ incoming weight vectors
\begin{align*}
    \underbrace{(\underbrace{w_{\ell_1}', \ldots, w_{\ell_1}'}_{i}, \ldots, \underbrace{w_{\ell_{c_i}}', \ldots, w_{\ell_{c_i}}'}_{i})}_{c_i}.
\end{align*}
There are $ c_i $ groups with the repetition of $ i $ zero-type incoming weight vectors (such that their summation is fixed at zero) thus we have to cancel out the recounting coming from these groups via a division by $ 1 / c_i! $.
Summing over all possible tuples $(k_1, \ldots, k_r, b_1, \ldots, b_j) $, we find the formula.
\end{proof}

\begin{theorem}\label{thm:connectivity-man}
For $ m > r $, $ \expman $ is connected: any pair of distinct points $ \ptheta, \ptheta' \in \expman $ is connected via a union of line segments $ \pgamma : [0,1] \to \expman $ such that $ \pgamma(0) = \ptheta $ and $ \pgamma(1) = \ptheta' $.
\end{theorem}

\begin{proof} We first prove the case $ m = r + 1$. Let $\ptheta^r=(w_1,\ldots,w_r,a_1,\ldots,a_r)$ and consider the following set of points
  \begin{align*}
          \widetilde{\Theta}_{r \to r + 1}(\ptheta^r) := \{ \Perm_\pi \ptheta^{r+1} : \ptheta^{r+1} = (w_0, w_{1}, w_{2}, \ldots, w_{r}, 0, a_{1}, a_{2}, \ldots, a_{r});
           \ \pi \in S_r, \ w_0 \in\R^{\din} \}
  \end{align*}
which is a subset of the expansion manifold $ \Theta_{r \to r + 1}(\ptheta^r) $.
We will show that by construction that a point $\ptheta_0\in \widetilde{\Theta}_{r \to r + 1}(\ptheta^r)$ such that $\ptheta_0 = (w_0, w_{1}, w_{2}, \ldots, w_{r}, 0, a_{1}, a_{2}, \ldots, a_{r})$ is connected to any other point $ \widetilde{\ptheta} = \Perm_\pi \ptheta_0 \in \widetilde{\Theta}_{r \to r + 1}(\ptheta^r) $

via a path in $ \Theta_{r \to r + 1}(\ptheta^r) $. To do so we first show that a neighbor where the neuron $ \vartheta_0 = (w_0, 0) $ is swapped with $ \vartheta_i = (w_i, a_i) $
 \begin{align*}
     \ptheta_1 = (w_{i}, w_{1}, \ldots, w_{i-1}, w_0, w_{i+1}, \ldots, w_{r},
     a_{i}, a_{1}, \ldots, a_{i-1}, 0, a_{i+1}, \ldots, a_{r})
 \end{align*}
can be reached in three steps using the following line segments $ \pupsilon^{(1)}_1, \pupsilon^{(1)}_2, \pupsilon^{(1)}_3: [0, 1] \to \Theta_{r \to r+1}(\ptheta^r) $
\begin{align*}
      \pupsilon^{(1)}_1(\alpha) &= (\alpha(w_{i} - w_0) + w_0, w_{1}, w_{2}, \ldots, w_{r}, 0, a_{1}, a_{2}, \ldots, a_{r}) \\
      \pupsilon^{(1)}_2(\alpha) &= (w_{i}, w_{1}, \ldots, w_{i-1}, w_{i}, w_{i+1}, \ldots, w_{r},
      \alpha a_{i}, a_{1}, \ldots, a_{i-1}, (1 - \alpha) a_{i}, a_{i+1}, \ldots, a_{r}) \\
      \pupsilon^{(1)}_3(\alpha) &= (w_{i}, w_{1}, \ldots, w_{i-1}, \alpha(w_0 - w_{i}) + w_{i}, w_{i+1}, \ldots, w_{r},
      a_{i}, a_{1}, \ldots, a_{i-1}, 0, a_{i+1}, \ldots, a_{r})
\end{align*}
where we have $ \pupsilon^{(1)}_1(0) = \ptheta_0 $, $ \pupsilon^{(1)}_1(1) = \pupsilon^{(1)}_2(0) $,
$ \pupsilon^{(1)}_2(1) = \pupsilon^{(1)}_3(0) $, and $ \pupsilon^{(1)}_3(1) = \ptheta_1 $.
In particular, we constructed a path $ \pgamma^{(1)} $ by glueing three line segments at their end points $$ \pgamma^{(1)}(t) = \pupsilon^{(1)}_1(3t)\mathbbm{1}_{t\in [0,1/3)} + \pupsilon^{(1)}_2(3(t-1/3))\mathbbm{1}_{t\in [1/3,2/3)} + \pupsilon^{(1)}_3(3(t-2/3))\mathbbm{1}_{t\in [2/3,1]}$$
where $ \pgamma^{(1)}(0) = \ptheta_0 $ and $ \pgamma^{(1)}(1) = \ptheta_1 $.
Note that going from $ \ptheta_0 \to \ptheta_1 $, we swapped the neurons $ \vartheta_0 $ and $ \vartheta_{i} $. Moreover, it is well known that any permutation can be written as a composition of transpositions (permutations leaving all elements unchanged but two) and that $(i\ j) = (0\ j) \circ (0\ i) \circ (0\ j)$.
In particular, we can reach $ \widetilde{\ptheta} $ only by swapping $ \vartheta_0 $ with other neurons, which corresponds to some other paths $ \pgamma^{(2)}, \ldots,  \pgamma^{(r)} $ made of three line segments.
Glueing these paths, we observe that $ \widetilde{\Theta}_{r \to r + 1}(\ptheta^r) $ is connected via paths in $ \Theta_{r \to r + 1}(\ptheta^r) $. To finish the case for $ m = r + 1 $, it is enough to show that any point $ \ptheta \in \Theta_{r \to r + 1}(\ptheta^r) \setminus \widetilde{\Theta}_{r \to r + 1}(\ptheta^r)$
\begin{align*}
    \ptheta = \Perm_{\pi}(w_{i}, w_{i}, w_{1}, \ldots, w_{r}, \alpha a_{i},(1 - \alpha)a_{i}, a_{1}, \ldots,  a_{r})
\end{align*}
is connected (via a line segment) to a point in $ \widetilde{\Theta}_{r \to r + 1}(\ptheta^r) $ which is simply
\begin{align*}
      \widetilde{\ptheta} = \Perm_{\pi}(w_{0}, w_{i}, w_{1}, \ldots, w_{r}, 0, a_i, a_{1}, \ldots,  a_{r}).
\end{align*}

Next we will prove for the general case $ m \geq r + 1 $ by induction. We assume that $ \Theta_{r \to m}(\ptheta^r) $ is connected and we will show that $ \Theta_{r \to m + 1}(\ptheta^r) $ is also connected. First we show the connectivity of the points in the following set
\begin{align*}
    \widetilde{\Theta}_{r \to m + 1}(\ptheta^r) := \{ \Perm_{\pi} \ptheta^{m+1}: \ptheta^{m+1} =
    (\underbrace{w_1, \ldots, w_1}_{k_1}, \ldots, \underbrace{w_r, \ldots, w_r}_{k_r}, \underbrace{w_1', \ldots, w_j'}_{j+1}, \underbrace{a_1^1, \ldots a_1^{k_1}}_{k_1}, \ldots,
    \underbrace{a_r^1, \ldots a_r^{k_r}}_{k_r}, \underbrace{0, \ldots, 0}_{j+1}) \notag \\
    \text{where} \ k_i \geq 1, j \geq 0, k_1 + \ldots + k_r + j = m, \sum_{i=1}^{k_j} a_j^i = a_j, \text{and} \ \pi\in S_{m+1} \}
\end{align*}
which is a subset of $ {\Theta}_{r \to m + 1}(\ptheta^r) $.
From the induction hypothesis, we have the connectivity of the manifold $\Theta_{r \to m}(\ptheta^r)$.

An element $ \widetilde{\ptheta} \in \widetilde{\Theta}_{r \to m + 1}(\ptheta^r) $ can be written as
\begin{align*}
    &\widetilde{\ptheta}  = \Perm_{\tilde{\pi}} (\underbrace{w_1, \ldots, w_1}_{k_1}, \ldots, \underbrace{w_r, \ldots, w_r}_{k_r}, \underbrace{w_1', \ldots, w_j'}_j, \underbrace{w_{0}}_1, \underbrace{a_1^1, \ldots a_1^{k_1}}_{k_1}, \ldots,
    \underbrace{a_r^1, \ldots a_r^{k_r}}_{k_r}, \underbrace{0, \ldots, 0}_{j+1}),
\end{align*}
for some $j\geq 0$ and $\tilde{\pi}\in S_{m+1}$.
For a fixed $w_{0}$ at a fixed position, there is a bijection $\widetilde{\Theta}_{r\to m+1}(\ptheta^r)\to\Theta_{r\to m}(\ptheta^r)$ that sends $\widetilde{\ptheta}$ to
\begin{align*}
  &\ptheta  = \Perm_{\pi} (\underbrace{w_1, \ldots, w_1}_{k_1}, \ldots, \underbrace{w_r, \ldots, w_r}_{k_r}, \underbrace{w_1', \ldots, w_j'}_j, \underbrace{a_1^1, \ldots a_1^{k_1}}_{k_1}, \ldots,
    \underbrace{a_r^1, \ldots a_r^{k_r}}_{k_r}, \underbrace{0, \ldots, 0}_j)
\end{align*}
for some $\pi\in S_m$, i.e. $\widetilde{\ptheta}$ where $w_{0}$ and its associated $0$ outgoing weight vector have been dropped.
In particular, any two points of $\widetilde{\Theta}_{r\to m+1}(\ptheta^r)$ with the same $w_{0}$ component at the same position are connected as a consequence of this correspondence and the connectivity of $\Theta_{r\to m}(\ptheta^r)$.
Moreover, we note that $\widetilde{\ptheta}\in\widetilde{\Theta}_{r\to m+1}(\ptheta^r)$ is connected via a line segment in $\widetilde{\Theta}_{r\to m+1}(\ptheta^r)$ to every other point in $\widetilde{\Theta}_{r\to m+1}(\ptheta^r)$ whose components are the same as $\widetilde{\ptheta}$ except for $w_{0}$. This straightforwardly generalizes for different positions of $ w_0 $ and this establishes the connectivity of $ \widetilde{\Theta}_{r\to m+1}(\ptheta^r) $.

Finally, we pick a point $ \ptheta \in \Theta_{r \to m + 1}(\ptheta^r) $ that is
$$
\ptheta = \Perm_\pi (\underbrace{w_1, \ldots, w_1}_{k_1}, \ldots, \underbrace{w_r, \ldots, w_r}_{k_r}, \underbrace{w_1', \ldots, w_1'}_{b_1}, \ldots, \underbrace{w_j', \ldots, w_j'}_{b_j}, \underbrace{a_1^1, \ldots a_1^{k_1}}_{k_1}, \ldots,
\underbrace{a_r^1, \ldots a_r^{k_r}}_{k_r}, \underbrace{\alpha_1^1, \ldots, \alpha_1^{b_1}}_{b_1}, \ldots, \underbrace{\alpha_j^1, \ldots, \alpha_j^{b_j}}_{b_j}).
$$
for some $ \pi \in S_{m+1} $. Note that $ \ptheta $ is connected to
$$
\widetilde{\ptheta} = \Perm_\pi (\underbrace{w_1, \ldots, w_1}_{k_1}, \ldots, \underbrace{w_r, \ldots, w_r}_{k_r}, \underbrace{w_1', \ldots, w_1'}_{b_1}, \ldots, \underbrace{w_j', \ldots, w_j'}_{b_j}, \underbrace{a_1^1, \ldots a_1^{k_1}}_{k_1}, \ldots,
\underbrace{a_r^1, \ldots a_r^{k_r}}_{k_r}, \underbrace{0, \ldots, 0}_{b_1}, \ldots, \underbrace{0, \ldots, 0}_{b_j}),
$$
which is in $ \widetilde{\Theta}_{r \to m + 1}(\ptheta^r) $. We have shown that all points in $ \Theta_{r \to m + 1}(\ptheta^r) $ are connected, which completes the induction step thus the proof.

\end{proof}

\subsection{No New Global Minimum}\label{sec:no-new-minima2}

The following assumption, only made in Theorem~\ref{thm:all-global-minima}, ensures that the activation function $\sigma$ has no specificity that yields other invariances than the symmetries between units, e.g. $\sigma$ cannot be even or odd.

\textbf{Assumption A.}
\textit{Let $ \sigma $ be a smooth activation function. We suppose that $ \sigma(0) \neq 0 $, that $ \sigma^{(n)}(0) \neq 0 $ for infinitely many even and odd values of $n\geq 0$, where $ \sigma^{(n)} $ denotes the $ n $-th derivative.}

The next lemma contains the main argument to prove that when considering an overparametrized 2-layers neural network, no new global minima are created besides those coming from invariances.

\begin{lemma}\label{lem:Vandermonde} Suppose that the activation function $ \sigma $ satisfies the Assumption A. If for some pairwise distinct nonzero $ \beta_1, \ldots, \beta_k \in\R $ and some constant $c\in\R$ we have $
        g(\alpha):=
        \sum_{\ell=1}^k a_\ell \sigma(\alpha \beta_\ell) = c $
for all $ \alpha \in \R $, then $ a_\ell = 0 $ for all $ \ell \in [k] $.
\end{lemma}

\begin{proof}
We reorder the indices such that for all $\ell\in [k-1]$, either $|\beta_\ell|>|\beta_{\ell+1}|$, or $\beta_{\ell}=-\beta_{\ell+1}$ such that $|a_{\ell}|\geq |a_{\ell+1}|$ (if the equality holds the labelling between the two is not important).
We distinguish the four following cases:
\begin{enumerate}
    \item $|\beta_1|>|\beta_2|$,
    \item $\beta_1=-\beta_2$ and $|a_1|>|a_2|$,
    \item $\beta_1=-\beta_2$ and $a_1=a_2$,
    \item $\beta_1=-\beta_2$ and $a_1=-a_2$.
\end{enumerate}
Note that there cannot be more that two indices $\ell$ with same $|\beta_\ell|$ and that 1. 2. 3. and 4. above are disjoint and cover all the possible cases.

Suppose that 1. holds.
Note that
\begin{align*}
    g^{(n)}(0)=\sum_{\ell = 1}^{k}a_\ell\beta_\ell^n\sigma^{(n)}(0)=0,
\end{align*}
for all $n\geq 1$, by assumption. On the other hand, the triangle inequality yields that
\begin{align*}
    |g^{(n)}(0)|
    &\geq \left(|a_1\beta_1^n|-\left|\sum_{\ell \neq 1}a_\ell\beta_\ell^n\right|\right)|\sigma^{(n)}(0)|
    \geq \left(|a_1\beta_1^n|-|\beta_2^n|\sum_{\ell \neq 1}|a_\ell|\right)|\sigma^{(n)}(0)|.
\end{align*}
One can always choose $n_0\geq 1$ large enough such that $\sigma^{(n_0)}(0)\neq 0$ and
\begin{align*}
    |\beta_1|>|a_1|^{-1/n_0}|\beta_2|\left(\sum_{\ell \neq \ell_1}|a_\ell|\right)^{1/n_0},
\end{align*}
so that $|g^{(n)}(0)|>0$, which is a contradiction with the fact that $g\equiv c$. Hence $a_1=0$. This shows the claim in the particular situation where all $|\beta_{\ell}|$'s are distinct.

One can deal with case 2. using that $|a_1|>|a_2|$, writing
\begin{align*}
    |g^{(n)}(0)|
    &\geq \left((|a_1|-|a_2|)|\beta_1^n|-|\beta_3|\sum_{\ell \neq 1,2}|a_\ell|\right)|\sigma^{(n)}(0)|.
\end{align*}
The reasoning is then identical to 1.

In the case 3., since $\sigma$ has infinitely many non-zero even derivatives at $0$, we use that $a_1\beta_1^{2n}+a_2\beta_2^{2n}=2a_1\beta_1^{2n}$ to write
\begin{align*}
    |g^{(2n)}(0)|
    &\geq \left((2|a_1|)|\beta_1^{2n}|-\sum_{\ell \neq 1,2}|a_\ell\beta_\ell^{2n}|\right)|\sigma^{(2n)}(0)|,
\end{align*}
then choose $n$ large enough to argue as above that $a_1=a_2=0$.
We can thus eliminates these terms from the definition of $g$ and go on with the argument.

In the case 4., if $\sigma$ has infinitely many non-zero odd derivatives at $0$, we apply the same reasoning as in 3. to show that $a_1=a_2=0$.

Since $\sigma$ has infinitely many even and infinitely many odd non-zero derivatives at $0$, we can iterate the argument and the proof is over since the four cases above cover all possible cases.
\end{proof}

When $\sigma$ does not satisfy Assumption A, the proof above allows us to derive the following results:

\begin{lemma}
If $\sigma$ is analytic such that $\sigma^{(n)}(0)\neq 0$ for infinitely many even $n\geq 0$ but only finitely many odd $n\geq 1$, then the function $g$ in Lemma~\ref{lem:Vandermonde} can be written as
\begin{align*}
    g(\alpha)
    =\sum_{\ell=1}^{\widetilde{k}}\widetilde{a}_\ell\widetilde{\sigma}(\alpha\widetilde{\beta}_\ell),
\end{align*}
where $\widetilde{\sigma}$ is an odd polynomial, the $\widetilde{a}_\ell$'s are nonzero and the $|\beta_\ell|$'s are pairwise distinct.

Similarly, if $\sigma^{(n)}(0)\neq 0$ for infinitely many odd $n\geq 1$ but only finitely many even $n\geq 0$, then the function $g$ in Lemma~\ref{lem:Vandermonde} can be written as
\begin{align*}
    g(\alpha)
    =\sum_{\ell=1}^{\widetilde{k}}\widetilde{a}_\ell\widetilde{\sigma}(\alpha\widetilde{\beta}_\ell),
\end{align*}
where $\widetilde{\sigma}$ is an even polynomial, the $\widetilde{a}_\ell$'s are nonzero and the $|\beta_\ell|$'s are pairwise distinct.
\end{lemma}

\begin{proof}
Suppose that $\sigma^{(2n+1)}(0)\neq 0$ for only finitely many $n\geq 0$. In the proof of Lemma \ref{lem:Vandermonde}, the only problematic situation is 4., that is $\beta_1=-\beta_2$ and $a_1=-a_2$.
In particular, they cancel out in the even derivatives of $g$, that is
\begin{align*}
    g^{(2n)}(0)=\sigma^{(2n)}(0)\sum_{\ell\neq 1,2}a_{\ell}\beta_\ell^{2n}.
\end{align*}
If $\beta_3,a_3,\beta_4,a_4$ do not fall into case 4. from the proof of Lemma \ref{lem:Vandermonde}, then one can show with the same argument therein that $a_3=a_4=0$.
Therefore, the problem reduces to the situation where $k$ is even, $\beta_{2\ell-1}=-\beta_{2\ell}$ and $a_{2\ell+1}=-a_{2\ell+2}$ for all $\ell\in [k/2]$.
We can then rewrite $g$ as
\begin{align*}
    g(\alpha)
    =\sum_{\ell=1}^{\widetilde{k}}\widetilde{a}_\ell\widetilde{\sigma}(\alpha\widetilde{\beta}_\ell),
\end{align*}
where $\widetilde{k}\leq k/2$, $\widetilde{a}_\ell:=a_{2\ell-1}$, $\widetilde{\beta}_\ell:=\beta_{2\ell-1}$ and $\widetilde{\sigma}(x):=\sigma(x)-\sigma(-x)$.
The function $\widetilde{\sigma}$ is analytic and locally polynomial around $0$, therefore is a polynomial on $\mathbb{R}$ and the $|\widetilde{\beta}_{\ell}|$'s are pairwise distinct.

When the even derivatives eventually vanish at $0$ instead, then the problematic situation is the 3. from Lemma \ref{lem:Vandermonde} and the function becomes
\begin{align*}
    g(\alpha)
    =\sum_{\ell=1}^{k/2}\widetilde{a}_\ell\widetilde{\sigma}(\alpha\widetilde{\beta}_\ell),
\end{align*}
where $\widetilde{a}_\ell:=a_{2\ell-1}$, $\widetilde{\beta}_\ell:=\beta_{2\ell-1}$ and $\widetilde{\sigma}(x):=\sigma(x)+\sigma(-x)$ with $\widetilde{\sigma}$ polynomial as above.
\end{proof}

\textbf{The case of the sigmoid activation $\sigma(x)=1/(1+e^{-x})$.} In this case, $\sigma(x)=1/2+\tanh(x)$ and $\tanh$ is an odd function, i.e. $\sigma^{(2n)}(0)=0$ for all $n\geq 1$. Hence, $\widetilde{\sigma}(x)=\sigma(x)+\sigma(-x)=1$ for all $x\in\R$ and one can construct the null function with already four $\beta$'s satisfying the constraints: $a_1\sigma(\beta_1 x)+a_1\sigma(-\beta_1 x)+a_3\sigma(\beta_3 x)+a_3\sigma(-\beta_3 x)=0$ as soon as $a_1=-a_3$, such that $|\beta_1|\neq |\beta_3|$. (One could then also achieve this for any even $p\geq 4$ such functions by tuning the $a_\ell$'s.)

\textbf{The case of the softplus activation $\sigma(x)=\ln(1+e^x)$.} The Softplus function is the primitive of the sigmoid such that $\sigma(x)=\int_{-\infty}^x \frac{1}{1+e^{-u}}\mathrm{d}u$. Therefore, $\sigma^{(2n+1)}(0)=0$ when $n\geq 1$. In particular, $\widetilde{\sigma}(x)=\sigma(x)-\sigma(-x)=x$ for all $x\in\R$.
One can thus obtain the null function with four (or a strictly greater even number) $\beta$'s satisfying the constraints: $a_1\sigma(\beta_1 x)-a_1\sigma(-\beta_1 x)+a_3\sigma(\beta_3 x)-a_3\sigma(-\beta_3 x)=0$, as soon as $a_1\beta_1+a_3\beta_3=0$, where $|\beta_1|\neq|\beta_3|$ are pairwise distinct.

\textbf{The case of the tanh activation function $\sigma(x)=(e^x-e^{-x})/(e^x+e^{-x})$.} Since $\sigma$ is an odd function, $\widetilde{\sigma}(x)=\sigma(x)+\sigma(-x)=0$ for all $x\in\R$ and therefore one can achieve the null function with two (or a strictly greater even number) $\beta$'s satisfying the constraints: $a_1\sigma(\beta_1 x)-a_1\sigma(-\beta_1 x)$.

We stress that for the three functions above, there is no other way to obtain the null function (i.e. the coefficients $\beta_{\ell}$'s and $a_{\ell}$'s have to be all in case 3. or case 4. depicted in the proof of Lemma \ref{lem:Vandermonde}, according to the derivatives of $\sigma$).

Recall that we consider the loss $ L_\mu^m $ where $ \mu $ is an input data distribution with support $ \R^{\din}$.

\begin{theorem}[Theorem 4.2 in the main]\label{thm:all-global-minima2} Suppose that the activation function $ \sigma $ satisfies the Assumption A. For $ m > r^* $,
let $ \ptheta $ be an $ m $-neuron point, and $ \ptheta_* $ be a unique $ r^* $-neuron global minimum up to permutation, i.e. $ L_\mu^{r^*}(\ptheta_*) = 0 $.
If $ L_\mu^m(\ptheta) = 0 $, then $ \ptheta \in \expmang $.
\end{theorem}

\begin{proof}
  For $ x\in\mathbb{R}^{\din}$, let $ h(x):=\sum_{j=1}^m a_j\sigma(w_j \cdot x)-\sum_{j=1}^{m^*} a_j^*\sigma(w_j^* \cdot x) $ and note that this function is zero on $ \mathbb{R}^{\din} $. Since $ \ptheta_* $ is irreducible, we know that the $ w_j^*$'s are pairwise distinct, and the $a_j^*$'s are nonzero.
  We can always group terms such that, wlog, the $w_j$'s are nonzero, pairwise distinct and the $ a_j $'s are nonzero, and we remain in the expansion manifold, as we now argue: we have that
  \begin{align*}
    h(x)=\sum_{j=1}^{m+m^*} a_j\sigma(w_j \cdot x),
  \end{align*}
 where we set $ a_j= -a_{j-m}^* $ and $ w_j=w_{j-m}^* $ for $ j\in\{m+1,\ldots,m+m^*\}$.
 If some of the $ w_j $'s appear several times, we group them together and if some are zero vectors, we summarize them in a constant $c\in\R$ and arrive at
 \begin{align*}
   h(x)=\sum_{j=1}^{M} A_j\sigma(W_j \cdot x)=c,
 \end{align*}
with $ M\leq m+m^* $, such that $ W_i \neq W_j $ for all $ i\neq j\in[M] $ with $W_j\neq (0,\ldots,0)^{\mathrm{T}}$.
Proving the claim, i.e. that $\ptheta\in\expmang$, is now equivalent to showing that $A_j=0$ for all $j\in M$.

If $ \din=1 $, we simply apply Lemma~\ref{lem:Vandermonde} which shows that $ A_j=0 $ for all $ j\in[M] $.

Suppose now that $ \din>1 $.
Let $\epsilon>0$ and let $t_\epsilon=(1,\epsilon,\epsilon^2,\ldots,\epsilon^M)^{\mathrm{T}}$. We define
\begin{align*}
    h_\epsilon(\alpha)
    :=\sum_{j=1}^MA_j\sigma(\alpha W_j\cdot t_\epsilon),\qquad\alpha\in\mathbb{R}.
\end{align*}
We claim that Lemma~\ref{lem:Vandermonde} applies to $h_\epsilon$, that is, the elements in $\{W_j\cdot t_\epsilon; j\in[M]\}$ are pairwise distinct for all $\epsilon>0$ small enough. Indeed, by contradiction, suppose that there exists a positive decreasing sequence $(\epsilon_n)_{n\geq 1}$ such that $\lim_{n\to\infty}\epsilon_n=0$ and $W_1\cdot t_{\epsilon_n} = W_2\cdot t_{\epsilon_n}$. Then $(W_1)_1 + \mathcal{O}(\epsilon_n) = (W_2)_1 + \mathcal{O}(\epsilon_n)$ where $(W_j)_k$ denotes the $k$-th component of $W_j$. Choosing $n$ large enough enforces $(W_1)_1=(W_2)_1$. It suffices then to explicit the terms of order $\epsilon_n$ in the identity and to reason identically since the rest is $\mathcal{O}(\epsilon_n^2)$. This implies that $W_1=W_2$, which is a contradiction with the assumption that the vectors $W_j$ are pairwise distinct.

Hence, by Lemma~\ref{lem:Vandermonde} applied on $h_\epsilon$, we have that $A_j=0$ for all $j\in[M]$, which concludes the proof.
\end{proof}

\begin{remark}
The theorem above does not apply to the sigmoid, the softplus and the $\tanh$ activation functions, since none of these satisfy Assumption A.
Nonetheless, we discussed above the theorem how to reconstruct a neural network function with these activations, with parameters that have to satisfy some explicit constraints depending on the activation (in particular, every $w'$ in the bigger network has to be either equal to $w$ or $-w$ of the smaller network).
By considering the extended expansion manifolds of these activation functions, comprised of the classical expansion manifold and these new points, Theorem \ref{thm:all-global-minima2} holds true, that is,
the extended expansion manifold is exactly the set of global minima.
\end{remark}

\subsection{Symmetry-Induced Critical Points}\label{sec:sym-induced-crit2}

We will prove the Propositions 4.3 and 4.4 in the main.

Recall that $ \ptheta^r_* = (w_1^*, \ldots, w_r^*, a_1^*, \ldots, a_r^*) $ denotes an irreducible critical point of $ L^r $.

\begin{proposition}[Proposition  4.3 in the main]\label{thm:crit-basins2} The expansion manifold $ \expmanc $ is a union of
  $$
   G(r, m) := \sum_{\substack{k_1 + \ldots + k_r = m \\ k_i \geq 1}} \binom{m}{k_1, \ldots, k_r}
  $$
distinct non-intersecting affine subspaces of dimension $ (m - r) $ and all points therein are critical points of $ L^m $.
\end{proposition}

\begin{proof} First, we show that $ \expmanc $ contains $ G(r, m) $ non-intersecting affine subspaces of dimension $ (m - r) $. Recall that by definition, we have
\begin{align*}
   \expmanc = \bigcup_{\substack{s = (k_1, \ldots, k_r) \\ \pi \in S_m}} \Perm_\pi \subsc
\end{align*}
where $ \subsc $ contains the points in the set
\begin{align*}
    \{ (\underbrace{w_1^*, \ldots, w_1^*}_{k_1}, \ldots, \underbrace{w_r^*, \ldots, w_r^*}_{k_r}, \underbrace{\beta^{1}_{1} a_1^*, \ldots, \beta^{k_1}_{1} a_1^*}_{k_1}, \ldots,
   \underbrace{\beta^{1}_{r} a_r^*, \ldots, \beta^{k_r}_{r} a_r^*}_{k_r}):
   \sum_{i=1}^{k_t} \beta_t^{i} = 1 \ \text{for} \ t \in [r] \}.
\end{align*}
Observe that this is an affine subspace. Its dimension is given by the number of free parameters, that is $ (k_1 - 1) + \cdots + (k_r - 1) = m - r $. All its permutations $ \Perm_\pi \subsc $ are also affine subspaces with the same dimension. For two of these subspaces to intersect, there should be a point contained in both subspaces. However, observe that the incoming weight vectors for two distinct subspaces are never the same thus an intersection point is not possible.

For the number of these subspaces, it is enough to count the distinct configurations of the incoming weight vectors, since the outgoing weight vectors follow the incoming ones. The formula for the number of distinct permutations of
\begin{align*}
    (\underbrace{w_1^*, \ldots, w_1^*}_{k_1}, \underbrace{w_2^*, \ldots, w_2^*}_{k_2}, \ldots, \underbrace{w_r^*, \ldots, w_r^*}_{k_r}).
\end{align*}
is $ \frac{m!}{k_1! \cdots k_r!} $ for a given tuple $ (k_1, \ldots, k_r) $ with $ k_i \geq 1 $ and $ k_1 + \cdots + k_r = m $. Summing over all such tuples, we find the formula for $ G(r, m) $.

Second, we will show that all points in $ \expmanc $ are critical.
 To do so we show that all points $ \ptheta^m_* \in \subsc $ are critical, then since $ \nabla L^m(\Perm_\pi \ptheta^m_*) = \Perm_\pi \nabla L^m(\ptheta^m_*) = 0$, we obtain the result for all points in $ \expmanc $.
For $ i \in [r] $, we denote the gradient components with respect to the $ i $-th incoming weight vector and the $i$-th outgoing weight vector as follows
\begin{align*}
      &\nabla_i^w L^r (\ptheta_*^r) = \frac{(a_i^*)^T}{N} \sum_{(x, y) \in \Trn} c'(f^{(2)}(x|\ptheta_*^r), y) \sigma'(w_i^* \cdot x) x, \\
      &\nabla_i^a L^r (\ptheta_*^r) =  \frac{1}{N} \sum_{(x, y) \in \Trn} c'(f^{(2)}(x|\ptheta_*^r), y) \sigma(w_i^* \cdot x).
\end{align*}
By introducing the $ \dout \times \din $ matrix U and the $ \dout $-dimensional vector V as
\begin{align*}
    U(w) &:= \frac{1}{N} \sum_{(x, y) \in \Trn} c'(f^{(2)}(x|\ptheta_*^r), y) \sigma'(w \cdot x) x, \\
    V(w) &:= \frac{1}{N} \sum_{(x, y) \in \Trn} c'(f^{(2)}(x|\ptheta_*^r), y) \sigma(w \cdot x).
\end{align*}
we have $ \nabla_i^w L^r (\ptheta_*^r) = (a_i^*)^T U(w_i^*) $ and $ \nabla_i^a L^r (\ptheta_*^r) = V(w_i^*) $. Since $ \ptheta^r_* $ is a critical point, we have $ (a_i^*)^T U(w_i^*) = 0 $ and $ V(w_i^*) = 0 $ for all $ i \in [r]$.
For $ \ptheta^m_* \in \subsc $, we have $ f^{(2)}(x|\ptheta_*^m) = f^{(2)}(x|\ptheta_*^r) $ and we write down the gradient components for $ L^m $ at $ \ptheta^m_* $
\begin{align*}
      &\nabla_{K_i + j}^w L^m (\ptheta_*^m) = \frac{\beta_i^j(a_i^*)^T}{n} \sum_{(x, y) \in \Trn} c'(f^{(2)}(x|\ptheta_*^m), y) \sigma'(w_i^* \cdot x) x = \beta_i^j (a_i^*)^T U(w_i^*) \\
      &\nabla_{K_i + j}^a L^m (\ptheta_*^m) =  \frac{1}{n} \sum_{(x, y) \in \Trn} c'(f^{(2)}(x|\ptheta_*^m), y) \sigma(w_i^* \cdot x) = V(w_i^*)
\end{align*}
where $ K_i = k_1 + \cdots + k_{i-1} $ and $ j \in [k_i] $ for all $ i \in [r]$. Since all gradient components are zero, thus $  \ptheta^m_* $ is a critical point of $ L^m $.

\end{proof}

\begin{proposition}[Proposition 4.4 in the main]\label{thm:spectra2} For twice-differentiable $ c $ and $ \sigma $, for all $ \ptheta_*^m \in \expmanc $, the spectrum of the Hessian $ \nabla^2 L^m (\ptheta^m_*) $ has $ (m - r) $ zero eigenvalues.
Moreover, if $ \ptheta_*^r $ is a strict saddle, then all points in $ \expmanc $ are also strict saddles.
\end{proposition}

\begin{proof} Because any $ \ptheta_*^m \in \expmanc $ lies in an equal-loss affine subspace of dimension $ (m-r) $, it has at least $ (m-r) $ zero eigenvalues in the Hessian.

For $ \ptheta_*^r $ that is a strict saddle of $ L^r $, we have an eigenvector $ \beta $ such that $ \beta^T \nabla^2 L^r(\ptheta_*^r) \beta < 0 $. Since $ \expmanc $ is an equal-loss manifold where all the points have the same loss as $ \ptheta_*^r $, we have $ L^r(\ptheta_*^r) = L^m (\ptheta_*^m) = L^m (U \ptheta_*^r)$ where $ U $ is a linear map.
Finally, we have $ (U\beta)^T \nabla^2 L^m(\ptheta_*^m) U \beta = \beta^T \nabla^2 L^r(\ptheta_*^r) \beta < 0  $ by the chain rule, and therefore $ \nabla^2 L^m(\ptheta_*^m) $ cannot be a positive semidefinite matrix, i.e. it has a negative eigenvalue, which completes the proof.
\end{proof}

\subsection{Combinatorial Analysis}\label{sec:comb-analysis2}

For proving the exact combinatorial results presented in the main (Proposition 4.5 and Lemma 4.6) it will be convenient to use Newton's series for finite differences \cite{milne2000calculus}:
\begin{definition}
Let $ p$ be a polynomial of degree $ d $, we define the $ k $-th {\it forward difference} of the polynomial $ p(x)$ at $ 0 $ as
\begin{align*}
   \Delta^k [p](0) &= \sum_{i=0}^k \binom{k}{i} (-1)^{k - i} p(i).
\end{align*}
\end{definition}
Hence, we can write $p(x)$ as
\begin{align}\label{eq:Newton-formula}
    p(x) = \sum_{k=0}^d \binom{x}{k} \Delta^k [p](0).
\end{align}
Rearranging the summands in  Equation~\ref{eq:Newton-formula}, one observes that Newton's series for finite differences is a discrete analog of Taylor's series
\begin{align*}
      p(x) = \sum_{k=0}^d \frac{\Delta^k [p](0)}{k!} [x]_k
\end{align*}
where $ (x)_k = x (x-1) \ldots (x - k + 1) $ is the falling factorial.

We now proceed with proving Proposition 4.5 in the main.
\begin{proposition}[Proposition 4.5 in the main]\label{prop:formulas-app} For $ r \leq m $, we have
  \begin{align}
        &G(r, m) = \sum_{i=1}^r \binom{r}{i}(-1)^{r - i}i^m, \label{eq:formula-g}\\
        &T(r, m) = G(r, m) + \sum_{u=1}^{m-r}  \binom{m}{u} G(r, m-u) g(u).\label{eq:formula-t}
  \end{align}
where $ g(u) = \sum_{j=1}^{u} \frac{1}{j!} G(j, u) $.
\end{proposition}
\proof The proof of the theorem is divided in the two next Propositon. In  Proposition~\ref{prop:formulas2} we prove Equation~\eqref{eq:formula-g}, while in Proposition~\ref{prop:formulas3}, using a counting argument (Lemma~\ref{lem:recounting}), we prove Equation \eqref{eq:formula-t}.
\endproof

\begin{proposition}\label{prop:formulas2} For $ r \leq m $, we have
  \begin{align*}
        G(r, m) = \sum_{i=1}^r \binom{r}{i}(-1)^{r - i}i^m.
  \end{align*}
\end{proposition}
\begin{proof}
First, recall that, by Proposition~\ref{thm:crit-basins2}, we have that
  $$
   G(r, m) := \sum_{\substack{k_1 + \ldots + k_r = m \\ k_i \geq 1}} \binom{m}{k_1, \ldots, k_r}.
  $$
The above can be restated by using the identity
\begin{align}\label{eq:expand-G}
    \sum_{\substack{k_1 + \cdots + k_r = m \\ k_i \geq 0}} \binom{m}{k_1, \ldots, k_r} =
    \sum_{\ell = 0}^r \binom{r}{\ell} \sum_{\substack{k_1 + \cdots + k_r = m \\ k_i \geq 0}}
    \binom{m}{k_1, \ldots, k_r}\mathbbm{1}_{I_\ell}(k_1,\ldots,k_r)
\end{align}
where  $I_\ell:=\{(0,\ldots, 0, k_{\ell+1}, \ldots, k_r): k_i\geq 1\text{ for } \ell+1 \leq i \leq  r \}$. Equation~\eqref{eq:expand-G} is equivalent to
\begin{align}\label{eq:recursion-G}
    r^m &=
    \sum_{\ell = 0}^r \binom{r}{\ell} G(r-\ell, m) \\
    &= \sum_{\ell = 0}^r \binom{r}{\ell} G(\ell, m),
\end{align}
with the convention that $ G(0, m) = 0 $.
Newton's series for finite differences (Equation~\eqref{eq:Newton-formula}), applied to the polynomial $ p(x) = x^m $ at $ x = r$, yields
\begin{align}\label{eq:recursion-N}
    r^m = \sum_{\ell =0}^r \binom{r}{\ell} \sum_{i=0}^\ell \binom{\ell}{i} (-1)^{\ell - i} i^m.
\end{align}
Note that the outer summation goes up to $ r $ instead of $ m $ since the terms with a factor $ \binom{r}{k} $ for $ k \geq r + 1 $ are zero. Hence we have
\begin{align}\label{eq:r-to-pow-m}
\sum_{\ell=0}^r \binom{r}{\ell} \left[ \sum_{i=0}^\ell \binom{\ell}{i} (-1)^{\ell - i} i^m -  G(\ell, m)\right]=0.
\end{align}
Indeed, with $m$ fixed, the solution
\begin{align}\label{eq:formula-G}
    G(\ell, m) = \sum_{i=0}^\ell \binom{\ell}{i} (-1)^{\ell - i} i^m
\end{align}
is the unique solution for the Equation~\eqref{eq:r-to-pow-m} with initial value given by the condition $1^m=1$. The uniqueness follows from an immediate induction argument: since
$$
G(1, m) = \sum_{k_1 = m} \binom{m}{k_1} = 1 =\sum_{i=0}^1 \binom{1}{i} (-1)^{1 - i} i^m,
$$
the initial step of induction is verified. Then, for the induction hypothesis, for $ k = 1, \ldots, r - 1 $, the first $r-1$ term in the summation in Equation~\eqref{eq:r-to-pow-m} are null, leaving us with the condition
\begin{align*}
    G(r, m) = \sum_{i=0}^r \binom{r}{i} (-1)^{r - i} i^m.
\end{align*}
\end{proof}

The Proposition above, which holds for $r<m$ shows that $ G(r, m)$ are the forward finite difference at $0$ for $ p(x) = x^m $, i.e. $G(r,m)=\Delta^r [p](0)$. We now comment on the meaning of the formula for $r\geq m$.
For a given polynomial $p(x)$ define the {\it rescaled} Newton's finite differences $\Delta^r_h [p](0)$ as Newton's finite differences (at $0$) for the polynomial $p(hx)$; hence, we can write the $r$-th derivative of the polynomial $p$ as the $h\to0$ limit of the $h$-the $r$-th Newton's finite difference:
\begin{align*}
    p^{(r)}(0) = \lim_{h \to 0^+ } \frac{\Delta_h^r [p](0)}{h^r}
    = \lim_{h \to 0^+ } \frac{1}{h^r}\sum_{i=0}^r \binom{r}{i} (-1)^{r - i} (hi)^m
    = \lim_{h \to 0^+ } \frac{1}{h^{r-m}} G(r,m).
\end{align*}
Hence for $ r = m $ we obtain $G(m,m)=m!$,
whereas for $ r > m $
we find $ G(r, m) = 0 $.

In order to prove Equation~\eqref{eq:formula-t}, we introduce the following Lemma~\ref{lem:recounting}, which is in fact a counting of the same number in two ways.

\begin{lemma}\label{lem:recounting} For $ j \leq n $, we have
  \begin{align*}
      \frac{1}{j!} G(j, n)  =
       \sum_{\substack{c_1 + 2 c_2 + \cdots + n c_n = n \\ c_1 + c_2 + \cdots + c_n = j \\ c_i \geq 0}} \frac{n!}{1!^{c_1} 2!^{c_2} \cdots n!^{c_n}} \frac{1}{c_1! \cdots c_n!}.
  \end{align*}
\end{lemma}

\begin{proof} By definition, we have
  \begin{align*}
      G(j, n) = \sum_{\substack{b_1 + \ldots + b_j = n \\ b_i \geq 1}} \binom{n}{b_1, \ldots, b_j}.
    \end{align*}
Starting from a tuple $ (b_1, \ldots, b_j) $, consider the tuple $ (c_1, \ldots, c_n) $ where $ c_i $ is the number of occurence of $ i $ in $ (b_1, \ldots, b_j) $. Therefore we have
\begin{align}\label{eqn:recounting1}
    \binom{n}{b_1, \ldots, b_j} = \binom{n}{\underbrace{1, \ldots, 1}_{c_1}, \underbrace{2, \ldots, 2}_{c_2}, \ldots, \underbrace{n}_{c_n}}
    = \frac{n!}{1!^{c_1} \cdots n!^{c_n}}.
\end{align}
Moreover, any $c$-tuple $ (c_1, \ldots, c_n) $ appears in
\begin{align}\label{eqn:recounting2}
      \binom{j}{c_1, \ldots, c_n} = \frac{j!}{c_1! \cdots c_n!}
\end{align}
$b$-tuples that are exactly $ (b_1, \ldots, b_j) $. From Equation~\eqref{eqn:recounting1} and Equation~\eqref{eqn:recounting2} and summing over all tuples $ (c_1, \ldots, c_n) $ we conclude.
\end{proof}

We are now in position to prove the closed-form formula for $ T $, Equation~\eqref{eq:formula-t}.

\begin{proposition}\label{prop:formulas3} For $ r \leq m $, we have
  \begin{align}\label{eq:formula-T}
        T(r, m) = G(r, m) + \sum_{u=1}^{m-r} \binom{m}{u} G(r, m-u) g(u)
  \end{align}
where $ g(u) = \sum_{j=1}^{u} \frac{1}{j!} G(j, u) $.
\end{proposition}

\begin{proof}
Let $ u = b_1 + \cdots + b_j $ and let $ c_i $ be, as in Lemma~\ref{lem:recounting}, the number of occurrences of $ i $ among $ (b_1, \ldots, b_j )$. Recall that for $T$ we have the identity
 \begin{align*}
 T(r, m) := \sum_{j = 0}^{m-r}  \sum_{\substack{\summ(s) = m \\ k_i \geq 1, b_i \geq 1}} \binom{m}{k_1, \ldots, k_r, b_1, \ldots, b_j} \frac{1}{c_b}.
 \end{align*}
We rewrite the outer summation in $ T $ from the number of $ b_i $'s to the summation of $ b_i $'s and we obtain
\begin{align*}
T(r, m) =  \sum_{u = 0}^{m-r} \sum_{j=0}^u \binom{m}{u} \sum_{\substack{k_1 + \cdots + k_r = m - u \\ b_1 + \cdots + b_j = u \\ k_i \geq 1, b_i \geq 1}} \binom{m - u}{k_1, \ldots, k_r}\binom{u}{b_1, \ldots, b_j} \frac{1}{c_1! c_2! \cdots c_{m-r}!}
\end{align*}
where we split the inner summation and the multinomial coefficient into two parts: one that comes from the incoming weight vectors and the others come from the zero-type neurons $ (w_1', \ldots, w_j') $. Using the formula for $ G $ on $ (k_1, \ldots, k_r )$, we simplify as follows
\begin{align*}
T(r, m) = \sum_{u = 0}^{m-r} \binom{m}{u} G(r, m-u) \sum_{j=0}^u \sum_{\substack{b_1 + \cdots + b_j = u \\ b_i \geq 1}} \binom{u}{b_1, \ldots, b_j} \frac{1}{c_1! c_2! \cdots c_{m-r}!}.
\end{align*}

Finally using Lemma~\ref{lem:recounting}, we find
\begin{align*}
T(r, m) = \sum_{u = 0}^{m-r} \binom{m}{u} G(r, m-u) \sum_{j=0}^u \frac{1}{j!} G(j, u)
\end{align*}
where $ G(0, 0) = 1 $. Splitting the case $ u = 0 $, we derive the closed form formula
\begin{align*}
T(r, m) = G(r, m) + \sum_{u = 1}^{m-r} \binom{m}{u} G(r, m-u) \sum_{j=1}^u \frac{1}{j!} G(j, u).
\end{align*}
\end{proof}

\begin{lemma}[Lemma 4.6 in the main]\label{lem:limiting-behavior} For any $ k \geq 0 $ fixed, we have,
  \begin{align*}
        G(m - k, m) \sim T(m - k, m) \sim \frac{m^k}{ 2^k k!} m!, \ \text{as} \ m \to \infty.
  \end{align*}
For any fixed $ r \geq 0 $, we have $ G(r, m) \sim r^m $ as $ m \to \infty $.
\end{lemma}

\begin{proof}
We begin to show that
\begin{align}\label{eq:limitG}
    \lim_{r\to\infty}\frac{1}{(r + k)! r^k} G(r, r + k)= \frac{1}{2^k k!} .
\end{align}
In particular, we observe that for $ k = 1 $ we have that
\begin{align*}
      G(r, r + 1) =  \sum_{\substack{k_1 + \ldots + k_r = r + 1 \\ k_i \geq 1}} \binom{r + 1}{k_1, \ldots, k_r} = \binom{r}{1} \binom{r + 1}{2, 1, \ldots, 1} = r \frac{(r + 1)!}{2!}.
\end{align*}
We find that the asymptotic in Equation~\eqref{eq:limitG} is in fact an exact  equality for any $ r>0 $.

For a generic $ k\geq 0 $, we divide the summation in $ G $ according to the number of $ 1 $'s in $ (k_1, \ldots, k_r) $
\begin{align}\label{eq:expand G}
      G(r, r + k)
      &=  \sum_{\substack{k_1 + \cdots + k_r = r + k \\ k_i \geq 1}} \binom{r + k}{k_1, \ldots, k_r} \nonumber\\
      &=  \binom{r}{k} \binom{r + k}{ \underbrace{2, \ldots, 2}_k, \underbrace{1, \ldots, 1}_{r-k}} +
      \sum_{n=1}^{k-1} \binom{r}{n}\sum_{\substack{k_1 + \cdots + k_n = n + k \\ k_i \geq 2}} \binom{r + k}{k_1, \ldots, k_n, \underbrace{1, \ldots, 1}_{r-n}}.
\end{align}
For a given tuple $(k_1, \ldots, k_n)$,
let $c = (c_2, \ldots, c_n) $, with $\sum_{i=2}^n c_i = n $ and $c_i$ is the number of occurrences of $ i $ among $ (k_1, \ldots, k_n)$, hence we have
\begin{align*}
  \binom{r + k}{k_1, \ldots, k_n, 1, \ldots, 1}
  = \frac{(r+k)!}{2!^{c_2}\cdots n!^{c_n}}.
\end{align*}
Since for a given $c = (c_2, \ldots, c_n) $ there are $\binom{n}{c_2,\ldots,c_n}$ $n$-tuples $(k_1,\ldots,k_n)$ with such occurrences, we rewrite Equation~\eqref{eq:expand G} as
\begin{align*}
      G(r, r + k)
      &= \binom{r}{k} \frac{(r + k)!}{2^k} + \sum_{n=1}^{k-1} \binom{r}{ n}\sum_{\substack{2c_2 + \cdots + nc_n = n + k\\ c_2 + \cdots + c_n = n}}\binom{n} {c_2, \ldots, c_n} \frac{(r+k)!}{2!^{c_2}\cdots n!^{c_n}}.
\end{align*}
Dividing both sides by $ (r + k)! r^k $, we find
\begin{align}\label{eq:Gexpansion}
      \frac{G(r, r + k)}{(r + k)! r^k} =  \frac{1}{2^k k! } \frac{r (r -1) \ldots (r - k + 1)}{ r^k }  +  \sum_{n=1}^{k-1}
      \sum_{\substack{2c_2 + \cdots + nc_n = n + k \\ c_2 + \cdots + c_n = n}} \frac{r (r-1) \ldots (r - n + 1) }{r^k} C_c,
\end{align}
where $ C_c : = 1 / (c_2! \cdots c_n! \cdot 2!^{c_2} \cdots n!^{c_n}) $. For $ n \leq k $, we have the following immediate double inequality:
\begin{align*}
      r^{n - k} \left(\frac{r - n + 1}{r} \right)^n \leq \frac{r (r-1) \ldots (r - n + 1)}{r^k} \leq r^{n - k}.
\end{align*}
Together with Equation~\eqref{eq:Gexpansion}, the above double inequality leads to
\begin{align*}
  &\frac{1}{2^k k! } \left(\frac{r - k + 1}{r} \right)^k +  \sum_{n=1}^{k-1}\sum_{\substack{2c_2 + \cdots + nc_n = n + k \\ c_2 + \cdots + c_n = n}}  r^{n - k} \left(\frac{r - n + 1}{r} \right)^n C_c   \\
  &\hspace{5cm}\leq \frac{1}{(r + k)! r^k} G(r, r + k)
  \leq  \frac{1}{2^k k! } +  \sum_{n=1}^{k-1}\sum_{\substack{2c_2 + \cdots + nc_n = n + k \\ c_2 + \ldots + c_n = n}} r^{n - k} C_c.
\end{align*}
In the limit $ r \to \infty $, both the lower and the upper bound converge to $ \frac{1}{2^k k!} $, hence giving
\begin{align*}
     G(r, r + k) \sim \frac{r^k (r + k)!}{2^k k!} \sim \frac{(r + k)^k (r + k)!}{2^k k!};
\end{align*}
finally, by choosing $ r = m - k $, we recover the first asymptotic of the Lemma.
In order to prove the asymptotic for $ T(m-k, m) $ we divide both sides in Equation~\eqref{eq:formula-T} (with $ r = m - k $) by $ G(m-k, m)$:
\begin{align*}
\frac{T(m-k,m)}{G(m-k,m)} = 1 + \sum_{u = 1}^{k} \binom{m}{u} \frac{G(m-k,m-u)}{G(m-k,m)} g(u).
\end{align*}
The limit of $T(m-k,m)$ as $ m \to \infty $, is then obtained from the asymptotic of $G(m-k,m)$ above:
\begin{align*}
1 + \sum_{u = 1}^{k} \binom{m}{u} \frac{G(m-k,m-u)}{G(m-k,m)} g(u) \sim
1 + \sum_{u = 1}^{k} \frac{m^u}{u!} c_u \frac{m^{k-u} (m-u)!}{m^k m!} g(u) \sim 1 + \sum_{u = 1}^{k} \frac{g(u)}{u!} \frac{c_u}{m^u} \sim 1
\end{align*}
hence, for large $ m $, $ T (m-k, m)  $ and $ G (m-k, m) $ grows at the same rate.

Finally, with an induction argument, we show that $ G(r, m) \sim r^m $ for fixed $r$ and $ m\gg r$.
For $r=1$, we have $G(1, m)=1$. For $ r = 2$, we have $ G(2, m) = 2^m - 2 \sim 2^m $.
We assume that for all $ \ell = 1, \ldots, r-1 $, we have $ G(\ell, m) \sim \ell^m $.
Normalizing Equation~\eqref{eq:recursion-G} by $ 1/r^m $, as $ m \to \infty $ we have
\begin{align*}
    1 = \frac{1}{r^m} G(r, m) +
    \frac{1}{r^m} \sum_{\ell = 1}^{r-1} \binom{r}{\ell} G(\ell, m) \sim \frac{1}{r^m} G(r, m) +
    \sum_{\ell = 1}^{r-1} \frac{r^\ell}{\ell!} \left(\frac{\ell}{r}\right)^m \sim \frac{1}{r^m} G(r, m).
\end{align*}
which completes the induction step, thus the Lemma.
\end{proof}

Thanks to the Propositions and Lemmas demonstrated in this section, we are now in position of proving the asymptotic behaviours presented in Equations (6) and (7) of the main, for \textit{mildly} and \textit{vastly} \textit{parameterized} regimes, respectively. We assume an overparameterized network of width $ m = r^* + n $ where the minimal width $ r^* $ is large.

\textbf{Mildly Overparameterized} (small $ h $). For fixed $k$ and $h$, in the limit $ r^* \to \infty $, Lemma~\ref{lem:limiting-behavior} (Lemma 4.6 in the main) gives the following asymptotic for $ G(r^* - k, m) $ and for $ T(r^*, m) $ :
\begin{align*}
    G(r^* - k, m) &\sim \frac{(m)^{k + h} }{2^{k+h} (k+h)! } m! \sim \frac{(r^*)^{k + h} }{2^{k+h} (k+h)! } m!\,,\\
    T(r^*, m) &\sim \frac{m^{h} }{2^{h} h! } m! \sim \frac{(r^*)^{h}}{2^{h} h! } m!\,.
\end{align*}
Taking the ratio of the two quantities above, we find
\begin{align*}
    \frac{G(r^* - k, m)}{T(r^*, m)}  \sim \frac{(r^*)^{k + h} }{2^{k+h} (k+h)! } \frac{2^{h} h!}{(r^*)^{h} } = \frac{(r^*)^k}{2^k (k+h) \cdots (h+1) }.
\end{align*}

\textbf{Vastly Overparameterized} ($ h \gg r^* $). We consider the case where $ h $ is much bigger than $ r^* $.

Using Equation~\eqref{eq:recursion-G} at $r= r^* - 1 $, we find
\begin{align*}
    \sum_{\ell  = 1}^{r^* - 1} \binom{r^* - 1} { \ell } G(\ell , m) = (r^* - 1)^m.
\end{align*}
We also have that $ T(r^*, m) \geq G(r^*, m) $. Thus if the numbers $ a_k $ of critical points in a network of width $k\in[r^*-1]$ are bounded by $\binom{r^* - 1}{r^*-k} $, we have
\begin{align*}
    \frac{\sum_{k = 1}^{r^* - 1} a_k G(r^*-k, m)}{T(r^*, m)}  \leq \frac{\sum_{r = 1}^{r^* - 1} \binom{r^* - 1}{r} G(r, m)}{G(r^*, m)} = \frac{(r^* - 1)^m}{G(r^*, m)}.
\end{align*}

On the other hand, since $ r^* \ll m $, we have that (Lemma~\ref{lem:limiting-behavior} , i.e. Lemma 4.6 in the main)
\begin{align*}
    \frac{(r^* - 1)^m}{G(r^*, m)} \sim \left(\frac{r^* - 1}{r^*}\right)^m
\end{align*}
as the limit $ m \to \infty $. Thus the inequality (7) in the main holds for large $ m $. Although beyond the scope of the paper, it is worth to point out that a more refined asymptotic analysis for $G$ can be carried on by means of the N\o rlund-Rice integral and saddle point techniques.

\subsection{Multi-Layer ANNs}\label{sec:multi-layer2}

In the case of multi-layers, the equivalence of two incoming weight vectors in the intermediate layers should be understood in the general sense, i.e. all incoming weight vectors of layer $ \ell $ are the outgoing weight vectors of layer $ \ell - 1 $ that can be written as
\begin{align*}
 \{ (\underbrace{(a_1^1)_d, \ldots, (a_{1}^{k_1})_d)}_{k_1}, \ldots, \underbrace{(a_r^1)_d, \ldots, (a_{r}^{k_r})_d}_{k_r},
  \underbrace{(\alpha_1^1)_d, \ldots, (\alpha_{1}^{b_1})_d}_{b_1}, \ldots, \underbrace{(\alpha_1^1)_d, \ldots, (\alpha_{r}^{b_j})_d}_{b_j}) :
  \sum_{i=1}^{k_t} (a_t^i)_d = (a_t)_d \ \text{and} \ \sum_{i=1}^{b_t} (\alpha_t^i)_d = 0 \}
\end{align*}
where $ d \in [r_\ell] $. All weight vectors in this set are equivalent in the sense that they produce the same neuron in layer $ \ell $.

For the general shape of the multi-layer expansion manifold, let us consider first a three-layer network.
If we add one neuron to the first hidden layer, we have that $ \Theta_{\br \to \bm}^{(1)}(\ptheta^{\br}) $ is connected. If we do not add a new neuron in the second hidden layer, the permutations of the neurons in the second hidden layer would bring $ r_2! $ disconnected components where each one of the disconnected components have $ T(r_1, r_1 + 1) $ affine subspaces that are connected to each other. Note that in this case the overall manifold $ \Theta_{\br \to \bm}(\ptheta^{\br}) $ is disconnected. However, adding one neuron to the second hidden layer, every $ r_2! $ disconnected components get connected through the parameters of the neurons in the second hidden layer, which yields a connected multi-layer expansion manifold $ \Theta_{\br \to \bm}(\ptheta^{\br}) $.

In general, adding $ n_1 $ neurons to the first hidden layer results in $ T(r_1, r_1 + n_1) $ connected affine subspaces instead of the usual $ r_1! $ discrete (i.e. disconnected) points. Adding $ n_2 $ neurons to the second hidden layer brings $ T(r_2, r_2 + n_2) $ affine subspaces instead of the usual $ r_2! $ points, for each one of the $ T(r_1, r_1 + n_1) $ affine subspaces. Note that this is multiplicative because every combination of the parameters in the first hidden layer can be paired with every combination of the parameters in the second hidden layer which results in a distinct affine subspace.
Similarly, via induction, if $ n_\ell \geq 1 $ for all $ \ell \in [L-1] $, adding $ (n_1, \ldots, n_{L-1}) $ neurons to each one of the hidden layers make a connected manifold of $ \prod_{\ell=1}^{L-1} T(r_1, r_1 + n_1) $ affine subspaces.

\end{document}